\documentclass{amsart}

\usepackage{xparse}

\usepackage{amssymb,amsfonts,amsmath}

\usepackage[mathscr]{eucal}

\usepackage{amsbsy}

\usepackage{stmaryrd}

\usepackage{bm}

\usepackage{braket}  

\usepackage{scalerel}

\usepackage{mathtools}

\NewDocumentCommand{\Tr}{s}{\IfBooleanTF{#1}{\vphantom{\intercal}}{\intercal}}

\NewDocumentCommand{\InvHelper}{}{\scalebox{0.5}[1.0]{\( - \)}1}
\NewDocumentCommand{\Inv}{s}{\IfBooleanTF{#1}{\vphantom{\InvHelper}}{\InvHelper}}
\NewDocumentCommand{\PinHelper}{}{\dagger}
\NewDocumentCommand{\Pin}{s}{\IfBooleanTF{#1}{\vphantom{\PinHelper}}{\PinHelper}} 

\makeatletter
\def\@eathelper#1#2\end@eath#3{%
	\if\relax\detokenize{#2}\relax%
	#3%
	\else%
	#1{\@eathelper#2\end@eath{#3}}%
	\fi}
\newcommand\ea[2]{%
	\@eathelper#1\relax\end@eath{#2}}
\makeatother

\NewDocumentCommand{\Vc}{O{} m !g t' t"}{%
  \bm{#1{\mathbf{\MakeLowercase{#2}}}}%
  \IfValueT{#3}{_{#3}}%
  \IfBooleanTF{#4}{^{\Tr}}{%
  	\IfBooleanT{#5}{^{\Tr*}}}%
}
\RenewDocumentCommand{\Vc}{O{} m !g t' t"}{%
  \bm{\ea{#1}{\mathbf{\MakeLowercase{#2}}}}%
  \IfValueT{#3}{_{#3}}%
  \IfBooleanTF{#4}{^{\Tr}}{%
    \IfBooleanT{#5}{^{\Tr*}}}%
}

\NewDocumentCommand{\Mx}{O{} m !g t' t"}{
  \bm{#1{\mathbf{\MakeUppercase{#2}}}}%
  \IfValueT{#3}{_{#3}}%
  \IfBooleanTF{#4}{^{\Tr}}{%
    \IfBooleanT{#5}{^{\Tr*}}}%
}
\RenewDocumentCommand{\Mx}{O{} m !g t' t"}{
  \bm{\ea{#1}{\mathbf{\MakeUppercase{#2}}}}%
  \IfValueT{#3}{_{#3}}%
  \IfBooleanTF{#4}{^{\Tr}}{%
    \IfBooleanT{#5}{^{\Tr*}}}%
}

\NewDocumentCommand{\Tn}{O{} m !g}{%
  \boldsymbol{#1{\mathscr{\MakeUppercase{#2}}}}%
  \IfValueT{#3}{_{#3}}%
}
\RenewDocumentCommand{\Tn}{O{} m !g}{%
  \boldsymbol{\ea{#1}{\mathscr{\MakeUppercase{#2}}}}%
  \IfValueT{#3}{_{#3}}%
}

\NewDocumentCommand{\Tm}{O{} m !g t' t"}{
  \bm{#1{\mathbf{\MakeUppercase{#2}}}_{(#3)}}%
  \IfBooleanTF{#4}{^{\Tr}}{%
    \IfBooleanT{#5}{^{\Tr*}}}%
}
\RenewDocumentCommand{\Tm}{O{} m !g t' t"}{
  \bm{\ea{#1}{\mathbf{\MakeUppercase{#2}}}_{(#3)}}%
  \IfBooleanTF{#4}{^{\Tr}}{%
	\IfBooleanT{#5}{^{\Tr*}}}%
}

\NewDocumentCommand{\tttLRexp}{mmmmmmm}{%
  \IfBooleanTF{#1}{%
    \IfBooleanTF{#2}{%
      \IfBooleanTF{#3}{%
        \IfBooleanTF{#4}%
        {\Biggl#5 #7 \Biggr#6}
        {\biggl#5 #7 \biggr#6}}%
      {\Bigl#5 #7 \Bigr#6}}%
    {\bigl#5 #7 \bigr#6}}%
  {#5 #7 #6}%
}
      
\NewDocumentCommand{\prn}{ssssm}{\tttLRexp{#1}{#2}{#3}{#4}{(}{)}{#5}}
\NewDocumentCommand{\Prn}{m}{\left( #1 \right)}

\NewDocumentCommand{\ang}{ssssm}{\tttLRexp{#1}{#2}{#3}{#4}{\langle}{\rangle}{#5}}
\NewDocumentCommand{\Ang}{m}{\left\langle #1 \right\rangle}

\NewDocumentCommand{\crly}{ssssm}{\tttLRexp{#1}{#2}{#3}{#4}{\{}{\}}{#5}}

\NewDocumentCommand{\sqr}{ssssm}{\tttLRexp{#1}{#2}{#3}{#4}{\lbrack}{\rbrack}{#5}}
\NewDocumentCommand{\Sqr}{m}{\left\lbrack #1 \right\rbrack}

\NewDocumentCommand{\ceil}{ssssm}{\tttLRexp{#1}{#2}{#3}{#4}{\lceil}{\rceil}{#5}}

\NewDocumentCommand{\floor}{ssssm}{\tttLRexp{#1}{#2}{#3}{#4}{\lfloor}{\rfloor}{#5}}

\NewDocumentCommand{\dsqr}{ssssm}{\tttLRexp{#1}{#2}{#3}{#4}{\llbracket}{\rrbracket}{#5}}

\NewDocumentCommand{\nrm}{ssssm}{\tttLRexp{#1}{#2}{#3}{#4}{\|}{\|}{#5}}
\NewDocumentCommand{\Nrm}{m}{\left\| #1 \right\|}

\NewDocumentCommand{\abs}{ssssm}{\tttLRexp{#1}{#2}{#3}{#4}{\vert}{\vert}{#5}}

\NewDocumentCommand{\had}{}{\ew}

\NewDocumentCommand{\ttm}{O{k}}{\mathop{\times_{\mspace{-2mu}#1}}}

\NewDocumentCommand{\ttv}{O{k}}{\mathbin{\bar{\smash{\times}\vphantom{\ast}}_{\mspace{-2mu}#1}}}

\DeclareMathOperator{\diag}{diag}

\DeclareMathOperator{\trace}{trace}

\DeclareMathOperator{\Prob}{Prob}

\DeclareDocumentCommand{\bigO}{}{\mathcal{O}}

\NewDocumentCommand{\qtext}{m}{\quad\text{#1}\quad}

\NewDocumentCommand{\Real}{}{\mathbb{R}}

\NewDocumentCommand{\Natural}{}{\mathbb{N}}

\NewDocumentCommand{\dsub}{m m t_ m}{#1_{#2_{#4}}}

\NewDocumentCommand{\tttMlist}{m m m m m m m}{%
  #1_{1} #2 %
  \IfBooleanF{#5}{#1_{2} #2} %
  \IfBooleanTF{#6}{ %
    #1_{3} %
    \IfBooleanT{#7}{#2 #1_{4} } %
  }{ %
    \IfBooleanTF{#3}{\cdots}{\dots} #2 #1_{#4} %
  }%
}

\NewDocumentCommand{\tttRMlist}{m m m m m m m}{%
  \IfBooleanTF{#6}%
  {\IfBooleanT{#7}{#1_{4} #2} #1_{3} #2 #1_{2} #2 #1_{1}}%
  {#1_{#4} \IfBooleanF{#5}{#2 #1_{#4-1}} #2 \cdots  #2 #1_{1}}%
}

\NewDocumentCommand{\tttSlist}{m m m m m m}{%
  #1_{1} %
  #2 \IfBooleanTF{#3}{\cdots}{\dots}
  #2 #1_{#5-1} %
  \IfValueT{#6}{#2 #6} %
  #2 #1_{#5+1} %
  #2 \IfBooleanTF{#3}{\cdots}{\dots}
  #2 #1_{#4} %
}

\NewDocumentCommand{\tttRSlist}{m m m m m m}{%
  #1_{#4} %
  #2 \IfBooleanTF{#3}{\cdots}{\dots}
  #2 #1_{#5+1} %
  \IfValueT{#6}{#2 #6} %
  #2 #1_{#5-1} %
  #2 \IfBooleanTF{#3}{\cdots}{\dots}
  #2 #1_{1} %
}

\NewDocumentCommand{\miwc}{s s t! O{i} O{d}}{%
  \tttMlist{#4}{,}{\BooleanFalse}{#5}{#3}{#1}{#2}%
}

\NewDocumentCommand{\siwc}{O{k} O{i} O{d} g}{%
  \tttSlist{#2}{,}{\BooleanFalse}{#3}{#1}{#4}%
}

\NewDocumentCommand{\minc}{s s t! O{i} O{d}}{%
  \tttMlist{#4}{}{\BooleanTrue}{#5}{#3}{#1}{#2}%
}

\NewDocumentCommand{\sinc}{O{k} O{i} O{d} g}{%
  \tttSlist{#2}{}{\BooleanTrue}{#3}{#1}{#4}%
}

\usepackage{xcolor}

\usepackage{tikz}
\usetikzlibrary{calc}
\usetikzlibrary{positioning}
\usepackage{pgfplots}
\usepackage{pgfplotstable}
\usetikzlibrary{pgfplots.colorbrewer}
\pgfplotsset{compat=1.17}

\usepackage{graphicx}

\usepackage[font=footnotesize,justification=Centering,singlelinecheck=false,margin=0pt]{subfig}

\usepackage{enumitem}

\usepackage{doi}
\usepackage[round,colon,authoryear%
]{natbib}

\usepackage{newfloat}

\usepackage{hyperref}
\usepackage[hyperpageref]{backref}
\definecolor{hrefcol}{HTML}{045a8d}
\hypersetup{
  colorlinks=true,
  allcolors=hrefcol,
}

\usepackage{algorithm, algpseudocode}

\usepackage[capitalize,nameinlink]{cleveref}

\crefformat{equation}{\textup{#2(#1)#3}}
\crefrangeformat{equation}{\textup{#3(#1)#4--#5(#2)#6}}
\crefmultiformat{equation}{\textup{#2(#1)#3}}{ and \textup{#2(#1)#3}}
{, \textup{#2(#1)#3}}{, and \textup{#2(#1)#3}}
\crefrangemultiformat{equation}{\textup{#3(#1)#4--#5(#2)#6}}%
{ and \textup{#3(#1)#4--#5(#2)#6}}{, \textup{#3(#1)#4--#5(#2)#6}}{, and \textup{#3(#1)#4--#5(#2)#6}}

\Crefformat{equation}{#2Equation~\textup{(#1)}#3}
\Crefrangeformat{equation}{Equations~\textup{#3(#1)#4--#5(#2)#6}}
\Crefmultiformat{equation}{Equations~\textup{#2(#1)#3}}{ and \textup{#2(#1)#3}}
{, \textup{#2(#1)#3}}{, and \textup{#2(#1)#3}}
\Crefrangemultiformat{equation}{Equations~\textup{#3(#1)#4--#5(#2)#6}}%
{ and \textup{#3(#1)#4--#5(#2)#6}}{, \textup{#3(#1)#4--#5(#2)#6}}{, and \textup{#3(#1)#4--#5(#2)#6}}

\crefdefaultlabelformat{#2\textup{#1}#3}

\makeatletter
\newcounter{algorithmicH}%
\let\oldalgorithmic\algorithmic
\renewcommand{\algorithmic}{%
  \stepcounter{algorithmicH}%
  \oldalgorithmic}%
\renewcommand{\theHALG@line}{ALG@line.\thealgorithmicH.\arabic{ALG@line}}
\makeatother

\newtheorem{theorem}{Theorem}[section]
\newtheorem*{nnthm}{Theorem} %
\newtheorem{lemma}[theorem]{Lemma}
\newtheorem{keyprop}[theorem]{Proposition}
\newtheorem{proposition}[theorem]{Proposition}
\newtheorem{corollary}[theorem]{Corollary}

\theoremstyle{definition}
\newtheorem{definition}[theorem]{Definition}

\theoremstyle{remark}
\newtheorem{example}[theorem]{Example}
\usepackage{tcolorbox}
\tcbuselibrary{breakable,skins,raster}
\definecolor{thmlight}{HTML}{b2df8a}
\definecolor{thmdark}{HTML}{33a02c}
\tcbset{
  colback=thmlight!50,
  colframe=thmdark,
  enhanced jigsaw,
}

\tcolorboxenvironment{theorem}{shrink tight,extrude by=2mm,before skip=4mm,after skip=4mm}
\tcolorboxenvironment{nnthm}{shrink tight,extrude by=2mm,before skip=4mm,after skip=4mm}
\tcolorboxenvironment{keyprop}{shrink tight,extrude by=2mm,before skip=4mm,after skip=4mm}

\newtcolorbox{algtcolorbox}[2][]{
	enhanced,
	colback=blue!10!white, %
	colframe=blue!50!black,
	shrink tight,
	extrude by=3mm,
	title    = {#2},
	colbacktitle=blue!10!white,
	before title={\vspace{1mm}}
	#1,
}

\DeclareFloatingEnvironment[
listname={List of Algorithms},
name=Algorithm,
placement=t,
within=none]{myalg}

\crefname{myalg}{Algorithm}{Algorithms}

\newenvironment{myalgo}[2]
{\def\algormiddle{
		\vspace{-8pt}
		\caption{#2}\label{#1}
		\vspace{-8pt}}%
	\begin{myalg}
		\begin{algorithmic}}
		{\end{algorithmic}
		\algormiddle	
\end{myalg}}

\renewenvironment{myalgo}[2]{%
  \begin{algorithm}
    \caption{#2}\label{#1}\footnotesize
    \begin{algorithmic}[1]
}{
    \end{algorithmic}
  \end{algorithm}
}

\renewenvironment{myalgo}[2]
{\begin{myalg}\footnotesize\begin{algtcolorbox}{\caption{#2}\label{#1}}
\begin{algorithmic}[1]}
{\end{algorithmic}\end{algtcolorbox}\end{myalg}}

\usepackage{fixme}
\fxsetup{
  status=draft,
  nomargin,
  inline,
  theme=color,
  silent,
}
\fxsetup{marginface=\linespread{1}\footnotesize}
\FXRegisterAuthor{tk}{tke}{{\textbf{TK}}}
\FXRegisterAuthor{jk}{jke}{{\textbf{JK}}}
\FXRegisterAuthor{jp}{jpe}{{\textbf{JP}}}%

\DeclareMathOperator{\opvec}{vec}

\DeclareMathOperator{\sym}{sym}

\DeclareMathOperator*{\avg}{avg}

\newcommand{\E}{{\mathbb{E}}}

\NewDocumentCommand{\sop}{O{d} m}{#2^{\otimes #1}}

\newcommand{\R}{{\mathbb{R}}}

\NewDocumentCommand{\W}{}{\Tn{W}}

\NewDocumentCommand{\U}{}{\Tn{U}}

\NewDocumentCommand{\V}{}{\Tn{V}}

\NewDocumentCommand{\M}{O{d}}{\Tn{M}{#1}}
\RenewDocumentCommand{\M}{O{d} O{}}{\Tn[#2]{M}^{(#1)}}

\NewDocumentCommand{\tM}{O{d}}{\Tn[\tilde]{M}{#1}}
\RenewDocumentCommand{\tM}{O{d}}{\Tn[\tilde]{M}^{(#1)}}

\NewDocumentCommand{\Mest}{O{d}}{\Tn[\widehat]{M}{#1}}
\RenewDocumentCommand{\Mest}{O{d}}{\Tn[\widehat]{M}\vphantom{\M}^{(#1)}}

\NewDocumentCommand{\T}{O{d}}{\Tn{T}{#1}}
\RenewDocumentCommand{\T}{O{d} o}%
{\IfValueTF{#2}{\Tn[#2]{T}^{(#1)}}{\Tn{T}^{(#1)}}}

\NewDocumentCommand{\Test}{O{d}}{\Tn[\widehat]{T}{#1}}
\RenewDocumentCommand{\Test}{O{d}}{{\Tn[\widehat]{T}}\vphantom{\T}^{(#1)}}

\NewDocumentCommand{\Testaug}{O{d}}%
{\Tn[\widehat\bar]{T}\vphantom{\T}^{(#1)}}

\DeclareDocumentCommand{\u}{o}{%
	\IfValueTF{#1}{\Vc[#1]{u}}{\Vc{u}}}

\DeclareDocumentCommand{\v}{o}{%
	\IfValueTF{#1}{\Vc[#1]{v}}{\Vc{v}}}

\NewDocumentCommand{\x}{o}{%
	\IfValueTF{#1}{\Vc[#1]{x}}{\Vc{x}}}

\DeclareDocumentCommand{\a}{}{\Vc{a}}%

\NewDocumentCommand{\gmdp}{m}{\alpha^{(#1)}}

\NewDocumentCommand{\mmdp}{O{j} m}{\beta_{#1}^{(#2)}}

\NewDocumentCommand{\umdp}{O{i} m}{\beta_{#1}^{(#2)}}

\NewDocumentCommand{\N}{}{\mathcal{N}}

\NewDocumentCommand{\m}{o}{%
\IfValueTF{#1}{\Vc[#1]{\mu}}{\Vc{\mu}}}
\NewDocumentCommand{\tm}{}{\Vc[\tilde]{\mu}}

\NewDocumentCommand{\C}{o}{%
\IfValueTF{#1}{\Mx[#1]{\Sigma}}{\Mx{\Sigma}}}

\NewDocumentCommand{\tC}{}{\Mx[\tilde]{\Sigma}}

\NewDocumentCommand{\Csqrt}{}{\Mx{H}}
\NewDocumentCommand{\tCsqrt}{}{\Mx[\tilde]{H}}

\NewDocumentCommand{\RV}{t' t" o}{Z\IfBooleanT{#1}{^{\Tr}}\IfBooleanT{#2}{^{\Tr*}}\IfValueT{#3}{_{#3}}}
\NewDocumentCommand{\tRV}{t' t" o}{\tilde{Z}\IfBooleanT{#1}{^{\Tr}}\IfBooleanT{#2}{^{\Tr*}}\IfValueT{#3}{_{#3}}}

\NewDocumentCommand{\maug}{}{\Vc[\bar]{\mu}}

\NewDocumentCommand{\Caug}{}{\Mx[\bar]{\Sigma}}

\NewDocumentCommand{\ARV}{s}{\IfBooleanT{#1}{\bar}Y}

\NewDocumentCommand{\Id}{}{\Mx{I}}

\NewDocumentCommand{\FD}{mm}{\nabla_{\!#2}\, #1}

\NewDocumentCommand{\ew}{}{\mathbin{\ast}} %
\NewDocumentCommand{\ewpow}{m}{\Sqr{#1}} %

\makeatletter
\let\ex\expandafter
\newcommand{\newcommandstring}[2]{%
	\expandafter\newcommand\csname #1\endcsname{#2}}
\newcounter{@lettercounter}
\def\letter{\alph{@lettercounter}}
\def\Letter{\Alph{@lettercounter}}
\newcommand{\lettercommand}[2]{%
\def\defhelper##1{%
\newcommandstring{#1}{%
\setcounter{@lettercounter}{##1} #2}}
\setcounter{@lettercounter}{0}
\loop
\stepcounter{@lettercounter}
\edef\@curletterval{\the@lettercounter}
\ex\defhelper\ex{\@curletterval}
\ifnum \value{@lettercounter}<26
\repeat}
\makeatother

\lettercommand{c\Letter}{\mathcal{\Letter}}

\lettercommand{bb\Letter}{\mathbb{\Letter}}

\lettercommand{bf\letter}{\mathbf{\letter}}

\lettercommand{bf\Letter}{\mathbf{\Letter}}

\lettercommand{sf\letter}{\mathsf{\letter}}
\lettercommand{sf\Letter}{\mathsf{\Letter}}

\newcommand{\alignindent}[1]{\hspace{#1}&\hspace{-#1}}

\NewDocumentCommand{\dblfact}{ m !o !O{#2-1}}
{\IfNoValueTF{#2}%
	{\frac{(2#1)!}{#1!2^{#1}}}%
	{\frac{(#2)!}{(#1)!2^{#1}}}}

\newcommand\numberthis{\stepcounter{equation}\tag{\theequation}}

\NewDocumentCommand{\bigger}{sss}{%
\IfBooleanTF{#1}{%
	\IfBooleanTF{#2}{%
		\IfBooleanTF{#3}%
		{\Biggl}
		{\biggl}}%
	{\Bigl}}%
{\bigl}}

\NewDocumentCommand{\PsiFn}{O{d}mmmm}%
{\Psi^{(#1)}\prn*{{#2}, \!{#3} \,, {#4}, \!{#5}}}

\NewDocumentCommand{\fGMM}{O{d}}
{F^{(#1)}}
\NewDocumentCommand{\fdeb}{O{d}}
{\varGamma^{(#1)}}

\algrenewcommand\algorithmicrequire{\textbf{Input:}}     
\algrenewcommand\algorithmicensure{\textbf{Output:}}

\NewDocumentCommand{\tB}{e{^}}
{\IfValueTF{#1}{\Mx[\tilde]{B}\vphantom{\Mx{B}}^{#1}}{\Mx[\tilde]{B}}}

\NewDocumentCommand{\sqrenum}{ssss t! m m}
{\edef\helpersqrenum{\IfBooleanT{#1}{*}\IfBooleanT{#2}{*}\IfBooleanT{#3}{*}\IfBooleanT{#4}{*}}
\expandafter\sqr\helpersqrenum{\tttMlist{#6}{,}{\BooleanFalse}{#7}{#5}{\BooleanFalse}{\BooleanFalse}}}
	
\NewDocumentCommand{\Sqrenum}{t! m m}
{\Sqr{\tttMlist{#2}{,\,}{\BooleanFalse}{#3}{#1}{\BooleanFalse}{\BooleanFalse}}}

\usepackage{silence}
\makeatletter
\def\@textbottom{\vskip \z@ \@plus 100pt}
\let\@texttop\relax
\makeatother

\begin{document}

\title[Tensor Moments of Gaussian Mixture Models]%
{Tensor Moments of Gaussian Mixture Models:\\
  Theory and Applications}

\author[J.~M.~Pereira]{Jo\~{a}o M. Pereira}
 \address{Oden Institute, University of Texas at Austin}
\email{joao.pereira@utexas.edu}

\author[J.~Kileel]{Joe Kileel}
 \address{Department of Mathematics and Oden Institute, University of Texas at Austin}
 \email{jkileel@math.utexas.edu}

\author[T.~G.~Kolda]{Tamara G. Kolda}
\address{MathSci.ai, Dublin, California}
\email{tammy.kolda@mathsci.ai}

\date{\today}

\begin{abstract}
  Gaussian mixture models (GMMs) are fundamental
  tools in statistical and data sciences.  
  We study the moments of multivariate Gaussians and GMMs.
  The $d$-th moment of an $n$-dimensional random variable
  is a symmetric $d$-way tensor of size $n^d$,
  so working with moments naively is assumed to be prohibitively expensive
  for $d>2$ and larger values of $n$.  
  In this work, we develop theory and numerical methods for
  \emph{implicit computations} with moment tensors of GMMs,
  reducing the computational and storage costs to $\bigO(n^2)$ and $\bigO(n^3)$,
  respectively, for general covariance matrices, and to $\cO(n)$ and $\cO(n)$, respectively, for diagonal ones.
  We derive concise analytic expressions for the moments in terms of symmetrized tensor products, relying on the correspondence between symmetric tensors and homogeneous polynomials, and combinatorial identities involving Bell polynomials.
  The primary application of this theory is to estimating GMM parameters (means and covariances)
  from a set of observations,
 when formulated as a moment-matching optimization problem.
  If there is a known and common covariance matrix, we also show
  it is possible
  to debias the data observations, in which case the problem of estimating the unknown
  means reduces to symmetric CP tensor decomposition.
  Numerical results validate and illustrate the numerical efficiency of our approaches.
  This work potentially opens the door to the competitiveness of the method of moments as compared to expectation maximization methods for parameter estimation of GMMs.
\end{abstract}

\maketitle

\section{Introduction}
\label{sec:introduction}

The Gaussian mixture model (GMM)
is a fundamental
tool in statistical and data sciences.
The Gaussian distribution, also known as the normal distribution,
is the limiting distribution of the average of any sequence of  independent random variables (under mild conditions).
A finite convex combination (i.e., a mixture) of Gaussian distributions is a GMM.
Utilization of GMMs is ubiquitous in density approximation,
clustering, and anomaly detection, finding application
in domains such as image processing,
biomedicine, financial forecasting, text analytics,  process monitoring,
and much more.

In this work, we consider the characterization of the
moments of multivariate GMMs, with the primary aim
of determining the parameters of a GMM by
matching sample and model moments.
In contrast to the often-used expectation maximization (EM)
method which does maximimum likelihood estimation \citep{HaTiFr09,Mu12,XuJo96},
the method of moments may have superior theoretical properties \citep{LiBa93,HsKa13,GeHuKa15,BaDiJiKa20,KhMaMo21,Kane_2021}.

The main difficulty with moments is that 
a $d$th-order moment can be prohibitive to compute and store
since it involves the expectations of many products of the coordinates of a random variable:
the $d$th moment of an $n$-dimensional random variable
is a symmetric $d$-way tensor of size $n^d$.
In this work, we provide a novel explicit formulation for the GMM
moment in terms of symmetrized tensor products.
Using the relationship between symmetric tensors and homogeneous polynomials,
we show that this formulation enables a reduction in the computational
and storage complexity of the method of moments.
We can compute the norm of the difference of moments and gradients with respect
to the parameters without ever forming the moments explicitly.
This makes the method of moments competitive with EM
because the computational and storage complexities
are now of the same order.

There are other implications of these results as well. We can now easily compare
GMM moments to each other or to observed data, opening the door to alternative
solution quality metrics, regardless of the method of solution.
Additionally, we show that it is possible to debias the moments
for any data contaminated with known white noise.
Finally, the formulas for moments of (single) Gaussians
and the techniques in this paper might enable the use of the method of moments for fitting Gaussian Processes.
This might allow
 applications to uncertainty quantification and generative models.

\subsection{Gaussian and Gaussian Mixture Models}
\label{sec:gauss-gauss-mixt}
We say that a random variable $X \in \R^n$ is from a multivariate Gaussian distribution if its probability density function (pdf) is
\begin{displaymath}
  f(X) = \frac{1}{(2\pi)^{n/2} |\C|^{1/2}}
  \exp \Prn{ -\frac{1}{2} (X-\m)^{\Tr} \C^{-1}(X-\m) }
\end{displaymath}
where $\m \in \R^n$ is the mean and
$\C \in \R^{n \times n}$ is the symmetric positive definite covariance matrix.
We denote this as
\begin{equation*}
    X \sim \N(\m,\C)
\end{equation*}
If $\C$ is a diagonal matrix, i.e., $\C = \diag(\sigma_1^2,\dots,\sigma_n^2)$, then we say that the Gaussian is \emph{diagonal} or \emph{axis-aligned}.
If, additionally, $\C = \sigma^2 \Id$, where $\Id$ denotes then $n \times n$ identity matrix, we say that the Gaussian is \emph{spherical} or \emph{isotropic}.

A finite convex combination of multiple Gaussians forms a GMM.
We denote
a random variable $X \in \R^n$ from a mixture of $m$ Gaussian \emph{components} as
\begin{equation}%
  \notag
  X \sim \sum_{j=1}^m \lambda_j \; \N( \m{j}, \C{j} ),
\end{equation}
where 
$\lambda_j \in [0,1]$ is the probability of
drawing from the $j$-th component $\N(\m{j},\C{j})$,
and $\sum_{j=1}^m \lambda_j = 1$.

\subsection{Moments of Gaussians and GMMs}
\label{sec:moments-gauss-gmms}
To explain higher-order moments, we first recall that the tensor product, denoted by $\otimes$, is the
higher-order analogue of the vector outer product.
We use the shorthand
$\sop{\v} = \v \otimes \v \otimes \cdots \otimes \v$ ($d$ times) to denote
the tensor product of an object with itself $d$ times.
For example, if $\Tn{V} = \sop[3]{\v}$, then $\Tn{V}(i,j,k) = v_i v_j v_k$; see \cref{fig:sop}.
\begin{figure}
  \centering
  \includegraphics{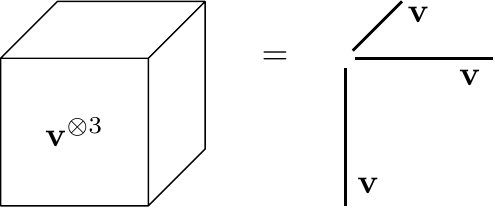}
  \caption{Three-way symmetric outer product.}
  \label{fig:sop}
\end{figure}
The $d$th moment of a random variable $X$, which we denote by $\M$, is a $d$-way symmetric tensor corresponding to the expected value of the $d$-way tensor product of the random variable with itself, i.e., 
\begin{displaymath}
  \M = \E(\sop{X}).
\end{displaymath}
The $d$th moment is useful for understanding higher-order dependencies of the random coordinates; 
for example, $\M[3](i,j,k) = \E(X_i X_j X_k)$.
The first moment is the mean.
If the data is centered, the second moment is the covariance.

A first result, previewed below, is an explicit formulation
for the $d$th-order moment of a GMM in terms of the model parameters.
The (single) Gaussian result corresponds to $m=1$.  

\vspace{0.5em}

\begin{nnthm}[{Preview of \cref{thm:gaussian_moments,prop:GMMbiasedmoments}}]
  If $X \sim \sum_{j=1}^m \lambda_j \; \N( \m{j}, \C{j} )$, then
  its $d$th moment $\M = \E(\sop{X})$ for $d\geq2$ is given by
  \begin{equation}
    \label{eq:Md_definition}
    \M = \sum_{j=1}^m \sum_{k=0}^{\floor{d/2}} 
    \lambda_j \, C_{d,k} \;
    \sym\prn**{ \sop[(d-2k)]{\m{j}} \otimes \sop[k]{\C{j}}}.
  \end{equation}
  with $C_{d,k}= \binom{d}{2k} \dblfact{k}$.
\end{nnthm} 

\vspace{0.5em}

\noindent
Here,  $\sym(\cdot)$ creates a symmetric version of the given tensor; e.g.,
for a matrix $\Mx{A} \in \R^{n \times n}$,
$\sym(\Mx{A}) = \frac{1}{2}(\Mx{A}' + \Mx{A})$.
To the best of our knowledge, this is the most general and succinct expression
of the moment tensor in terms of the parameters of a Gaussian or GMM.
For example, the third moment is
\begin{displaymath}
\M[3] = \sum_{j=1}^m \lambda_j \prn***{
  \sop[3]{\m_j} + 3\sym\prn{\m_j \otimes \C_j}},  
\end{displaymath}
as illustrated in \cref{fig:moment3}.

\begin{figure}
  \centering
  \includegraphics[scale=0.8]{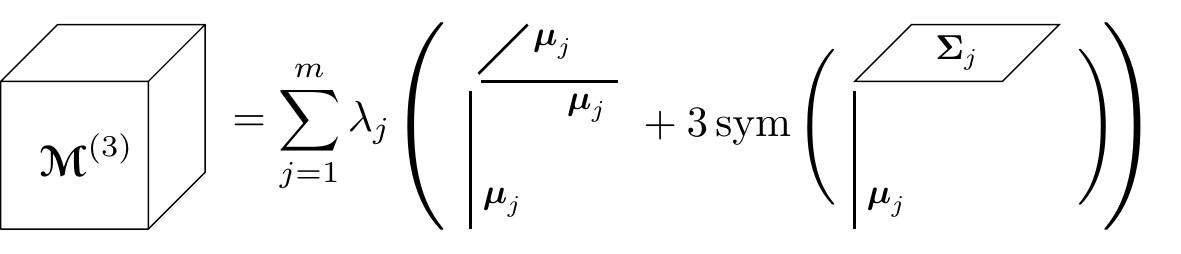}  
  \caption{Third moment of Gaussian mixture model is a three-way symmetric tensor.}
  \label{fig:moment3}
\end{figure}

A major challenge in working with higher-order moments ($d\geq3$) is
that the storage and computational costs are exponential in $d$, i.e., $\cO(n^d)$.
Working with third-order or fourth-order moments can quickly exhaust computational
resources for even moderately-sized variables.
One of our main contributions is extending the above result to calculate the quantities
\begin{equation*}
    \ang{ \M, \sop{\a} }, 
    \quad
    \nabla_{\a} \ang{ \M, \sop{\a} },
    \quad
    \nabla_{\m_j} \ang{ \M, \sop{\a} },
    \qtext{and}
    \nabla_{\C_j} \ang{ \M, \sop{\a} },
\end{equation*}
for an arbitrary vector $\a \in \R^n$,
without forming $\M$,
in time and storage that is quadratic in $n$ and linear in $m$ and $d$.
Additionally, $\nrm{\M}^2$ (which has cross-products)
can be calculated in the same storage and 
in time that is cubic in $n$, quadratic in $m$, and quadratic in $d$.
If the covariance matrices are diagonal, then the dependencies on $n$ are all linear.  
These formulas enable us to compare empirical moment tensors and model moment tensors without forming either moment explicitly.

\subsection{Applications to Parameter Estimation}
\label{sec:appl-param-estim}
The main application of our results are to efficiently recover the parameters of a GMM, i.e., 
$\set{(\lambda_j,\m{j},\C{j})}_{j=1}^m$,
from $p$ independent realizations $\set{\miwc[\x][p]}$
of the random variable
$X\sim \sum_{j=1}^m \lambda_j \; \N( \m{j}, \C{j} )$.
\Cref{fig:gmm_example} illustrates an example three-component  GMM in two dimensions,
showing its probability distribution function (pdf) and 1000 sample realizations ($n=2,m=3,p=1000$).

\begin{figure}
  \centering
  \subfloat[Two-dimensional pdf]{%
    \includegraphics[width=0.45\textwidth]{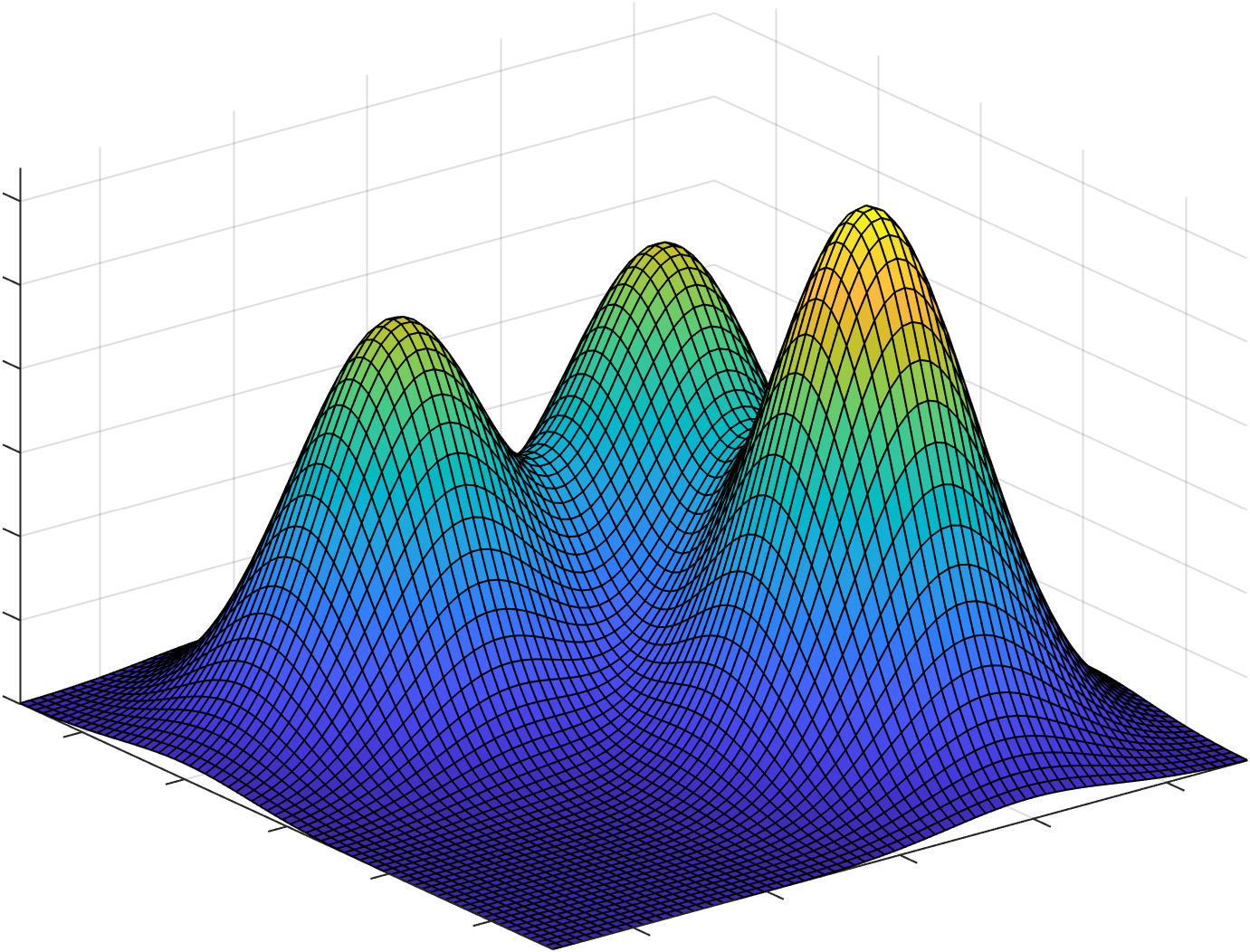}}
  \;\;
  \subfloat[Contour lines of pdf with 1000 sample realizations as black dots]{%
    \includegraphics[width=0.45\textwidth]{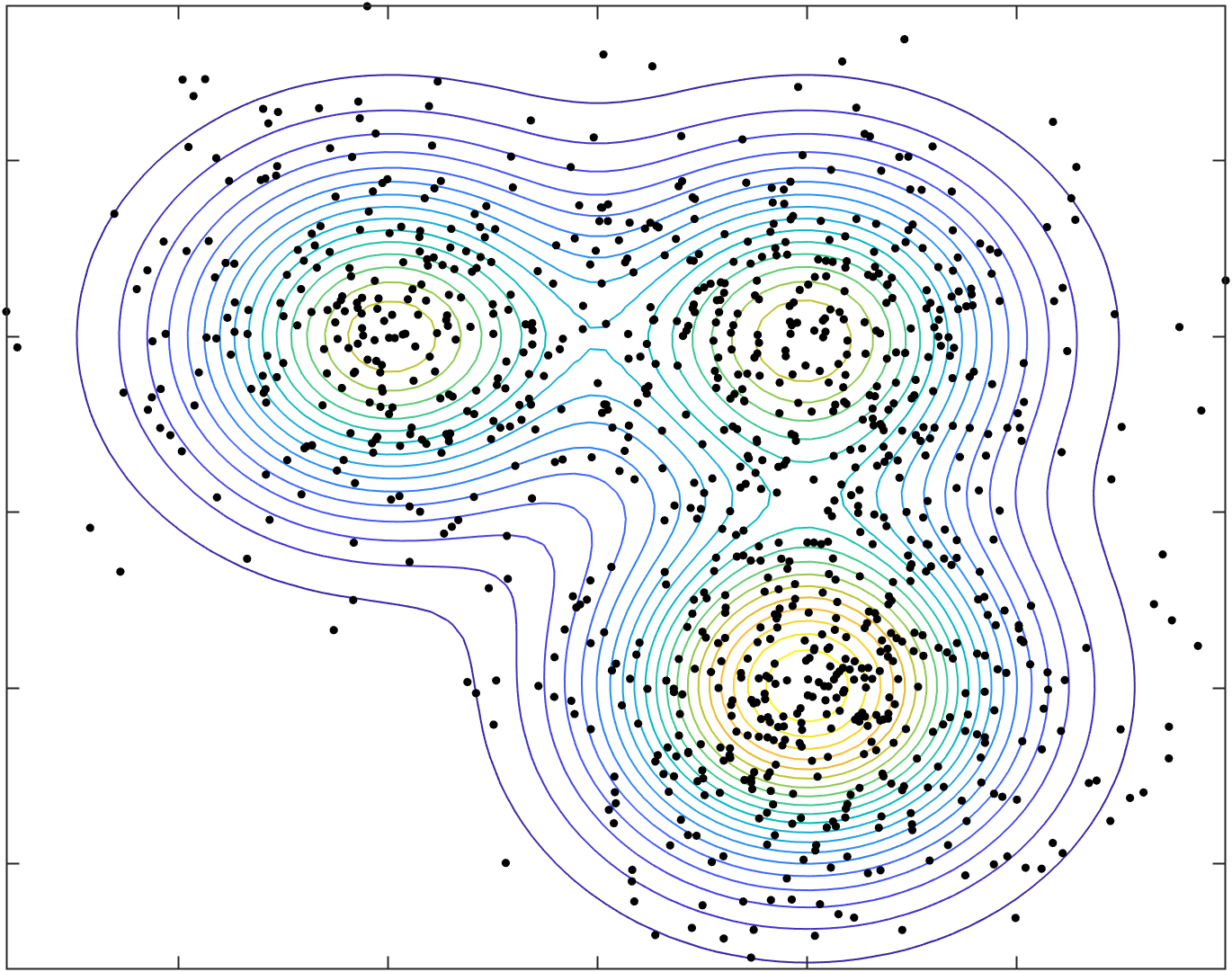}}
  \caption{Example GMM in two dimensions ($n=2$)
    and three components ($m=3$).
    The means are
    $\m_1=\begin{bsmallmatrix} 1 & -1 \end{bsmallmatrix}^{\Tr}$,
    $\m_2=\begin{bsmallmatrix} 1 & 1 \end{bsmallmatrix}^{\Tr}$,
    $\m_3=\begin{bsmallmatrix} -1 & 1 \end{bsmallmatrix}^{\Tr}$,
    and the common covariance is
    $\C=0.4\,\Mx{I}$.
    The convex combination is defined by the weights
    $\Vc{\lambda} =
    \begin{bsmallmatrix}
      0.4 & 0.3 & 0.3
    \end{bsmallmatrix}^{\Tr}$.
  }
  \label{fig:gmm_example}
\end{figure}

\subsubsection{Method of Moments}
\label{sec:moment-matching-via}
Expectation maximization is a standard tool for fitting GMMs but has some limitations.
An alternative is 
the method of moments, which optimizes the model parameters to match the empirical moments.
Consider the problem, for now, in terms of a single moment.
For any value of $d$, 
we can estimate $\M=\E(\sop{X})$ from the realizations, $\set{\miwc![\x][p]}$, i.e.,
\begin{equation}\label{eq:empircal_moment}
  \Mest = \frac{1}{p} \sum_{i=1}^p \sop{\x_i}.
\end{equation}
For example, the third empirical moment is illustrated in \cref{fig:empmoment3}.
\begin{figure}
  \centering
  \includegraphics[scale=0.8]{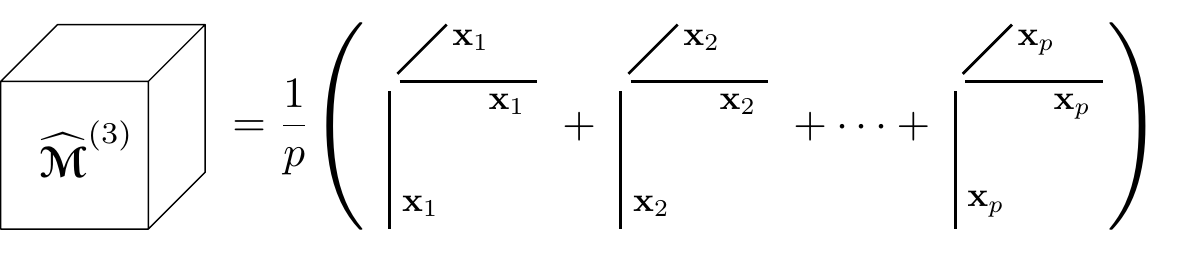}  
  \caption{Third empirical moment of a random variable is a three-way symmetric tensor.}
  \label{fig:empmoment3}
\end{figure}
Then we can cast the GMM parameter identification problem as
an optimization problem of the following form:
\begin{equation}\label{eq:opt-problem}
  \min F^{(d)}(\theta) \equiv \Nrm{\; \M - \Mest \;}^2
  \qtext{with respect to} \theta = \set{ (\lambda_j, \m{j}, \C{j}) }_{j=1}^m.
\end{equation}
where $\M$ has the form defined in \cref{eq:Md_definition}.
We can rewrite the objective function as
\begin{equation}
  \label{eq:opt-problem-expanded}
  F^{(d)}(\theta)
  =   \Nrm{\M}^2
  - \frac2{p} \sum_{i=1}^p \Ang{\M, \sop{\x{i}}}
  + \frac1{p^2}\sum_{i=1}^p \sum_{j=1}^p \Ang{\x{i},\x{j}}^d.
\end{equation}
The first two terms in \eqref{eq:opt-problem-expanded} reduce to  quantities that we can compute efficiently,
including the gradients, using our results discussed above.
The last term does not involve the model parameters and can be ignored for the purposes of optimization.

\subsubsection{Debiasing with Known Common Covariance Matrix}
\label{sec:debiasing}

If a data set of observations is contaminated with Gaussian noise,
i.e., $\N(\bf{0},\C)$ where the covariance $\C$ is known, then we can \emph{debias}
the data observations via the following result.

\vspace{0.5em}

\begin{nnthm}[Preview of \cref{thm:GMMdebiasedmoments}]
  Let $X = Y + Z$ where $Y$ and $Z$ are independent random variables and
  $Z \sim \N( \bf{0}, \C)$. Let
  \begin{equation}
    \T = \E\prn{\sop{Y}}%
  \end{equation}
  be the $d$th moment of $Y$.
  Then, given sample realizations $\set{\miwc![\x][p]}$ of $X$, 
  \begin{equation}\label{eq:Test_intro}
    \Test = \frac{1}{N}
    \sum_{i=1}^p \sum_{k=0}^{\floor{d/2}} C_{d,k} (-1)^k
    \sym \Prn{  \sop[d-2k]{\x{i}} \otimes \sop[k]{\C} },
  \end{equation}
  where $C_{d,k}= \binom{d}{2k} \dblfact{k}$,
  is an unbiased estimator for $\T$ $(\E\prn{\Test} = \T)$.  
\end{nnthm}

\vspace{0.5em}

In the context of GMM with known common covariance matrix $\C = \C{1} = \cdots = \C{m}$,
we can recast the estimate of the remaining parameters $\set{(\lambda_j, \m{j})}_{j=1}^m$
as a \textit{symmetric tensor decomposition} problem:
minimize $\nrm{\T - \Test}^2$ with $\T = \sum_{j=1}^m \lambda_j \sop{\m{j}}$.
As before with the standard moment tensor,
we extend this result to
\emph{implicitly} calculate quantities needed in the gradient-based optimization, such as
\begin{equation*}
  \ang{ \Test, \T },
  \quad
  \nabla_{\lambda_{j}} \ang{ \Test, \T },
  \qtext{and}
  \nabla_{\m{j}} \ang{ \Test, \T }.
\end{equation*}
This can be done without forming $\Test$ or $\T$, and
in time and storage that is quadratic in $n$ (the number of variables) and linear in $m$, $p$, and $d$.

\subsubsection{Handling All Moments Simultaneously}
\label{sec:augm-fix-scal}

For a fixed order $d$,
there is a scaling ambiguity in \eqref{eq:opt-problem} causing the optimization problem to have multiple continuously varying solutions.
This is sometimes fixed by working with multiple moments simultaneously.
We propose to do something similar but implicitly,
solving the optimization problem for multiple moments simultaneously
via a mathematical conversion where a constant is appended to
each observation.
If $X \sim \sum_{j=1}^m \lambda_j \; \N( \m{j}, \C{j} )$
and $\set{\x{1},\dots,\x{p}}$ is a set of realizations, then
the augmented observations
\begin{displaymath}
  \Vc[\bar]{x}{i} =
  \begin{bmatrix}
    \x{i} \\ \omega
  \end{bmatrix} \in \R^{n+1},
\end{displaymath}
where $\omega$ is some constant,
can be seen as coming from the random variable
\begin{equation*}
  \bar X \sim \sum_{j=1}^m \lambda_j \; \N( \Vc[\bar]{\mu}{j}, \Mx[\bar]{\Sigma}{j} )
  \qtext{where}
  \Vc[\bar]{\mu}{j} = \begin{bmatrix} \m{j} \\ \omega \end{bmatrix}
  \qtext{and}
  \Mx[\bar]{\Sigma}{j} = \begin{bmatrix} \C{j} & \bf 0 \\ \bf 0 & 0 \end{bmatrix}.
\end{equation*}
\Cref{sec:augmented-system-gmm} shows that
minimizing the augmented $d$-th moment matching problem defined by
$F^{(d)}(\bar{\theta})$
with
respect to the augmented parameters
$\bar{\theta}=\set{ (\lambda_j, \Vc[\bar]{\mu}{j}, \Mx[\bar]{\Sigma}{j})}_{j=1}^m$,
is equivalent to a weighted
sum of all moments up to order $d$, i.e., 
\begin{displaymath}
  F^{(d)}(\bar{\theta}) = \sum_{k=1}^d \binom{d}{k} \omega^{2d-2k} F^{(d)}(\theta) + C,
\end{displaymath}
where $C$ is a constant that does not depend on $\theta$.
In this way, the augmented problem is simultaneously matching all moments from orders 1 to $d$.

\subsection{Approach and Tools}
\label{sec:approach-tools}

Our derivations in part rely upon the equivalence of symmetric tensors and homogeneous
polynomials, which is a  well-known correspondence in the computational algebraic geometry community.
This equivalence is used, for instance, to derive
the concise formulation of the moment tensor in terms of the GMM parameters.
Along the way, we prove a binomial theorem for tensors (\cref{thm:binom-tensors}).

We also employ tools from combinatorics for, e.g., calculating the
inner products of moment tensors. For this, we use
Bell polynomials, which are intimately related to cumulants.
Using recurrences in terms of Bell polynomials, we can efficiently compute key quantities
such as, for instance, $\ang{\M,\sop{\a}}$ where $\M$ is a moment
tensor for a GMM and $\a \in \Real^n$ (\cref{thm:gmmdotp}).

We demonstrate the utility of these approaches for estimating
the parameters of a GMM, especially in contexts where EM approaches
are not as successful.

\subsection{Related Work}
\label{sec:related-work}

There are two basic approaches for fitting statistical models: 
expectation maximization (EM) and the method of moments. 
For multivariate GMMs, EM has to this point been considered the
only practical model, but we did find one interesting
early application of the method of moments to GMMs by \citet{LiBa93}.

To date, most interest in the
method of moments has  come from a theoretical point of view.  There are two theoretical advantages.  
Firstly,
the method of moments can be used to remove the usual requirement that means be well separated.  The works  \citet{Da99}, \citet{BeSi09,BeSi10} and \citet{moitra2010settling} and others  show that the restriction
can be relaxed via the method of moments, at least in theory.  This stands in contrast to the situation for EM, where there is a need for well-separated means \citep{XuJo96}. 
As a second advantage, the method of moments can lead to structured polynomial systems or tensor decomposition problems for parameter estimation of GMMs.
This has allowed some authors to develop provable polynomial-time algorithms, with bounds on the number of samples required, although it has not been clear how practical these methods are.
Under the condition that the means are linearly independent
and the covariances are spherical, \citet{HsKa13}  develop
a method based on first-, second-, and third-order moments
in which they can recover the means and covariances
using a combination of eigen- and symmetric tensor decomposition.
\citet{KhMaMo21} develop a similar method and recommend it
for intializing EM.
\citet{GeHuKa15} extend this to arbitrary covariance matrices,
but they work with vectorized covariance matrices and as a result 
forfeit some of the symmetries. 
\citet{agostini2021moment} use algebraic geometry  to prove identifiability results: namely when each mixture component has the same unknown covariance, the GMM's parameters are uniquely determined by a few of the model's moments.
These past works are primarily concerned with the
question of sample efficiency, i.e., how many samples
are required to recover the model, as well as identifiability questions.
We do not consider those issues in our work (even though
they are certainly topics for future work), but we do 
build upon these in other ways: we handle
moments of arbitrary order, develop practical
moment-based algorithms using numerical optimization, and provide computational evidence of both their effectiveness
and efficiency.

As mentioned, symmetric tensor decomposition has played an important role in 
the bit of computational work done thus far for GMMs and moment
methods.  
This is because the moment formed by the means is an approximation
to the empirical moment tensor:
\begin{displaymath}
    \sum_{j=1}^m \lambda_j \sop{\m_j} \approx \frac{1}{p} \sum_{i=1}^p \sop{\Vc{x}_i}.
\end{displaymath}
The connection has been considered in several works, e.g., \cite{AnGeHsKa14a,AnGeHsKa14}.
The works that use symmetric tensor decomposition tend to use
simultaneous diagonalization, which is not very robust.  Other recent works employ algebraic methods from polynomial solving for the cases of spherical and diagonal covariances, as in  \cite{guo2021learning}, \citet{lindberg2021estimating} and \citet{KhMaMo21}.
By contrast, \citet{sk-2019} developed a different computational approach,
based on numerical optimization, 
that avoids forming the empirical or approximate model moment
tensors.  
We build upon their approach in this \nolinebreak work.

\subsubsection{Parallels with a Scientific Domain}
\label{sec:cryo-EM}
In cryo-electron microscopy (cryo-EM), the goal is to estimate a three-dimensional model for a molecule, given many noisy two-dimensional images \citep{bendory2020single}. 
This can be viewed  as a  parameter estimation problem, where the data are the images and the unknown parameters correspond to the molecule (and other features). 
So far in cryo-EM,  expectation maximization methods have been dominant. 
See \cite{sigworth2010introduction} for a description of the EM approach, and \cite{scheres2012relion} for a software implementation.  
However, recent works have  considered using the method of moments, e.g. 
  \citet{bandeira2017estimation} and \citet{sharon2020method}.  
We believe our paper might find applications to cryo-EM, because the noise on cryo-EM images is typically assumed to be  Gaussian.  Also, GMMs have been used a modeling tool \citep{chen2021deep}.

\section{Preliminaries}
\label{sec:preliminaries}

\subsection{Tensors and tensor products}
\label{sec:tens-tens-prod}
We begin by establishing notation and basic terminology for tensors.
We let $\cT_n^d = \sop{(\R^{n})} = \R^{n} \otimes \R^n \otimes \cdots \otimes \R^n$
($d$ times) denote the vector space of real 
tensors of \textit{order} $d$ and \textit{dimension} $n$.
Tensors with dimensions $0$, $1$ and $2$ are scalars, vectors and matrices, respectively.
If $\W \in \cT_n^d$, then $w_{\minc!} \equiv \W(\miwc!)$ is the entry indexed
by $(\miwc!) \in [n]^{d}$, where $[n] = \set{1, \ldots, n}$.

The tensor product is a generalization of the outer product,
the tensor power is an outer product of a tensor (possibly a vector or matrix)
with itself, and the tensor inner product is the dot product of the
vectorized representations. We formalize these ideas below.

\begin{definition}[Tensor product]
  \label{def:tensor-product}
  For tensors $\W \in  \cT_n^{d}$ and $\U \in \cT_n^{d'}$,
  their \textit{tensor product}  in $\cT_n^{d+d'}$ is defined by
  \begin{equation*}
    (\W \otimes \U)(\miwc!,\miwc![j][d'])
    = w_{\minc!}u_{\minc![j][d']}
    \; \forall \, (\miwc!,\miwc![j][d'])\in [n]^{d+d'}.
  \end{equation*}
\end{definition}

\begin{definition}[Tensor power]
  The \textit{tensor power} $\sop[m]{\W}\in \cT_n^{m d}$ is the
  tensor product of $\W$ with itself $m$ times. 
\end{definition}

Consder the case of $d=1$, so we have just a vector. If $\v \in \R^n = \cT_n^1$,
then its tensor power $\sop{\v}$ is a tensor
with $\prn*{\sop{\v}}(\miwc!) = \prod_{k=1}^d v_{i_k}$.

\begin{definition}[Tensor inner product]\label{def:tensor-inner-product}
  The \textit{tensor inner product} of
  $\W,\U \in \cT_n^d$ is 
  \begin{equation}
    \Ang{\W, \U}
    =\sum_{i_1=1}^n \cdots \sum_{i_d=1}^n w_{\minc!}u_{\minc!}.
  \end{equation}
\end{definition}

The norm of a tensor is $\nrm{\W} = \sqrt{\ang{\W, \W}}$.
For $d=0,1,2$, the dot product of these as tensors coincides with the usual dot product definition:
if $\Vc{u}, \v \in  \cT_n^1$, then $\ang{ \Vc{u}, \v } = \Vc{u}' \v$,
and if $\Mx{A}, \Mx{B} \in  \cT_n^2$, then $\ang{ \Mx{A}, \Mx{B} }
= \operatorname{trace}(\Mx{A}' \Mx{B}) = \opvec(\Mx{A})^{\Tr} \opvec(\Mx{B})$. 

Properties of tensor inner and outer products can be combined in useful ways, as follows.
\begin{lemma}[Inner product of tensor products, \citealp{hackbusch2019tensor}]
  \label{lemma:inner_tensor_powers}
  For tensors $\W_1,\W_2 \in \cT_n^d$, $\U_1,\U_2 \in \cT_n^{d'}$,
  we have 
  \begin{displaymath}
    \Ang{\W_1 \otimes \U_1, \W_2 \otimes \U_2}
    =  \Ang{\W_1 , \W_2}  \Ang{\U_1, \U_2}.
  \end{displaymath}
  In particular, for all vectors $\Vc{u},\v\in \R^n$,
  we have $\Ang{\sop{\v}, \sop{\Vc{u}}} = \Ang{\v, \Vc{u}}^d$.
\end{lemma}

\subsection{Symmetric tensors}
\label{sec:symmetric-tensors}

A symmetric tensor is a tensor whose entries are invariant under any
permutation of the indices.
Symmetric tensors have important properties, and
these are relevant to us because moment tensors are symmetric.

\begin{definition}[Symmetric tensor]\label{def:sym-tensor}
  A tensor $\W \in \cT_n^d$ is \emph{symmetric} if it is
  unchanged by any permutation of indices, that is, 
  \begin{equation}
    w_{\minc![j]} = w_{\minc[\dsub{j}{\sigma}]}
    \!\! \quad \forall \, (\miwc![j])\in [n]^{d}
    \textup{ and } \sigma\in \mathfrak{S}^d,
  \end{equation}
  where $\mathfrak{S}^d$ is the permutation group on $[d]$.
  We denote by $\cS_n^d \subset \cT_n^d$
  the vector space of real symmetric tensors of order $d$ and length $n$. 
\end{definition}

For example, for $\v \in \R^n$, the tensor power $\sop{\v}$ is symmetric.
Hence, a moment tensor $\M = \sop{X}$, which is the expectation of a tensor power,
is also symmetric.

If a tensor is not already symmetric, then it can be made symmetrized
via the $\sym(\W)$ operation; see \cref{fig:sym}.
Moreover, a tensor is symmetric if and only if $\sym(\W) = \W$;
see \citet[Prop.~3.1]{CoGoLiMo08}.

\begin{definition}[Symmetrization]\label{def:sym}
  A tensor $\W \in \cT^d_n$ may be \textit{symmetrized} via 
  \begin{equation}\label{eq:symmetrizing_operator}
    \sym (\W)(\miwc!) =
    \frac{1}{d!}\sum_{\sigma\in \mathfrak{S}^d}
    w_{\minc![\dsub{i}{\sigma}]}
    \quad \forall \, (\miwc!)\in [n]^{d}.
  \end{equation}  
\end{definition}

\begin{figure}
  \centering
  \includegraphics{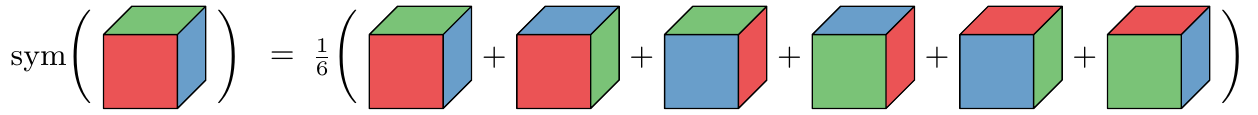}
  \caption{Symmetrization of a 3-way tensor.}
  \label{fig:sym}
\end{figure}

Symmetrization is potentially costly but can essentially be skipped
for certain computations involving symmetrized tensors, as
elucidated in the next lemma.
This will be useful later for computational efficiency in working with moments.
 
\begin{lemma}[\citealp{hackbusch2019tensor}]
\label{lemma:sym_orthogonal_projection}
The $\sym$ operation defined in \cref{def:sym} is an orthogonal projection
and therefore self-adjoint.
In particular, for any vector $\v \in \R^n$ and tensor $\U \in \cT^{d}_n$,
  we have
  \begin{displaymath}
    \Ang{\sym(\U),\sop{\v}}=\Ang{\U,\sop{\v}}.
  \end{displaymath}
\end{lemma}

\subsection{Tensors and homogeneous polynomials}
\label{sec:polynomials}

Key to several main results is the basic link between tensors and homogeneous polynomials;
 see, e.g., \citet[Sec.~2.6.4]{landsberg2012tensors} or \citet{CaSt13}.
We denote by $\R[\miwc![z][n]]$ the ring of real polynomials in $n$ variables
$\Vc{z}=(\miwc![z][n])$, and 
{$\R[\miwc![z][n]]_{d}$}
denotes the subspace of homogeneous degree $d$ forms.

\vspace{0.5em}

\begin{keyprop}\label{lemma:Phi}
  Define $\Phi:\bigcup_{d=0}^\infty \cT^d_n\to \bigcup_{d=0}^{\infty} \R[\miwc![z][n]]_{d}$
  (the set of homogeneous polynomials in $n$ variables), such that for $\V\in \cT^d_n$:
  \begin{equation*}%
    \Phi\Sqr{\V}(\Vc{z}) = \ang{ \V, \sop{\Vc{z}} }.
  \end{equation*}
  Then, we have the following.
  \begin{enumerate}[label={\bf(\Alph*)},ref={(\Alph*)}]
  \item \label{enum:Phi_1to1}
    For every integer $d$, $\Phi$ is a linear map from
    $\cT^d_n$ to $\R[\miwc![z][n]]_{d}$,
    it is bijective when restricted to $\cS^d_n$,
    and for all $\V\in \cT^d_n$:\vspace{-1mm}\newline
    \begin{minipage}{2\linewidth-\textwidth}
    \begin{equation}\label{eq:Phi_sym}
	\Phi\Sqr{\V} = \Phi\Sqr{\sym(\V)}.
	\end{equation}    	
    \end{minipage}\medskip\newline
    In particular, if $\V \in \cS^d_n$ and
    $\Phi\Sqr{\V} = \Phi\Sqr{\W}$,
    then $\V = \sym (\W)$.
  \item \label{enum:Phi_matvec}
    For any vector $\v \in \R^n$ and matrix $\Mx{M}\in  \R^{n\times n}$,
    the homogeneous polynomials $\Phi\Sqr{\v}$, $\Phi\Sqr{\Mx{M}}$ are such that
    \begin{equation*}%
      \Phi\Sqr{\v}(\Vc{z}) = \v'  \Vc{z}
      \quad \text{and} \quad
      \Phi\Sqr{\Mx{M}}(\Vc{z}) = \Vc{z}'  \Mx{M} \Vc{z}.
    \end{equation*}
  \item \label{enum:Phi_prod}
    For all $\V\in \cT^d_n$ and $\W\in \cT^{d'}_n$, we have
    \begin{equation*}%
      \Phi\sqr{\V \otimes \W} = \Phi\sqr{\V}\Phi\sqr{\W}.
    \end{equation*}
  \end{enumerate}
\end{keyprop}
\begin{proof}
  We first show \ref{enum:Phi_1to1}.
  The linearity of $\Phi$ follows from the bilinearity of the inner product.
  Since all the entries of the tensor $\sop{\Vc{z}}$ are in $\R[z_1, \ldots, z_n]_d$,
  $\Phi\sqr{\V}$ is also in $\R[z_1, \ldots, z_n]_d$.
  Moreover, the tensor $\sop{\Vc{z}}$ contains all the monomials of degree $d$ in $n$ variables,
  thus $\Phi$ is surjective over all tensors in $\cT_n^d$.
  Additionally, \Cref{lemma:sym_orthogonal_projection} implies \eqref{eq:Phi_sym},
  which in turn implies that $\Phi$ is surjective over all symmetric tensors in $\cS_n^d$. 
  Since the vector space dimensions of $\cS_n^d$ and $\R[z_1, \ldots, z_n]_d$
  coincide \citep{hackbusch2019tensor},
  $\Phi$ is a one-to-one map between these spaces.
  \ref{enum:Phi_matvec} follows from the dot product definition for vectors and matrices;
  for matrices we have
  $\ang{ \Mx{M}, \sop[2]{\Vc{z}} }
  = \operatorname{trace}(\Mx{M} \Vc{z} \Vc{z}')
  = \Vc{z}'  \Mx{M} \Vc{z}$.
  Finally, \ref{enum:Phi_prod} follows from \Cref{lemma:inner_tensor_powers}:
  \begin{equation*}
    \Phi\sqr{\V \otimes  \W}(\Vc{z})
    = \ang*{\V\otimes  \W, \sop[(d+d')]{\Vc{z}}}
    =  \ang*{\V, \sop{\Vc{z}}} \ang*{\W, \sop[d']{\Vc{z}}}
    = \Phi\sqr{\V}(\Vc{z})  \Phi\sqr{\W}(\Vc{z}).
    \qedhere
  \end{equation*}
\end{proof}

\newcommand{\specialref}[1]{\cref{lemma:Phi}\ref{#1}}

\Cref{lemma:Phi} is useful mainly for two reasons.
First, \specialref{enum:Phi_1to1} implies that we can  determine a symmetric tensor by describing its inner product with rank-1 tensors.
Secondly, although the operation $\otimes$ is not commutative, for symmetric tensors we can write it in terms of products of homogeneous polynomials using \specialref{enum:Phi_prod}, and this product is commutative.
A quick application of this idea gives the following identity, which we call the binomial theorem for tensors.  

\vspace{0.5em}

\begin{corollary}[Binomial theorem for tensors]\label{thm:binom-tensors}
  For all $\v, \Vc{u}\in \R^n$,
  \begin{displaymath}
    \sop{(\v + \Vc{u})}
    = \sum_{k=0}^d \binom{d}{k} \sym\Prn{ \sop[k]{\v} \otimes \sop[d-k]{\Vc{u}} }.    
  \end{displaymath}
\end{corollary}

\begin{proof}
  Using \cref{lemma:Phi}, we can rewrite this result in terms of polynomials and apply the binomial theorem for powers of polynomials:
  \begin{align*}
    \Phi\Sqr{\sop{(\v + \Vc{u})}}
    = \Phi\Sqr{\v + \Vc{u}}^d%
    = \sum_{k=0}^d \binom{d}{k} \Phi\Sqr{\v}^k\Phi\sqr{\Vc{u}}^{d-k}%
    = \sum_{k=0}^d \binom{d}{k} \Phi\Sqr{\sop[k]{\v} \otimes \sop[d-k]{\Vc{u}}}.
  \end{align*}
  Then, since $\sop{(\v + \Vc{u})}$ is a symmetric tensor and
  $\Phi$ is bijective when restricted to symmetric tensors, we obtain the result.
\end{proof}

\subsection{Bell polynomials, cumulants and higher-order moments}
\label{sec:bell-polyn-cumul}
An important tool from combinatorial enumeration is Bell polynomials, and
these play an intimate role in computation of higher-order moments.
First, we present the definition and basic facts of Bell polynomials,
which can be found in \citet{bell1927partition,comtet2012advanced}.
Then we explain the connection to moments.

\begin{definition}[Bell polynomials]\label{def:BellPolynomials}
  The (complete) Bell polynomials are defined by
  \begin{equation}\label{eq:belldef}
    B_k\prn{x_1, \dots, x_k}
    = \sum_{\Vc{j}\in \cP_k}
    \frac{k!}{j_{1}!j_{2}!\cdots j_{k}!}
    \Prn{\frac{x_{1}}{1!}}^{j_{1}}
    \Prn{\frac{x_{2}}{2!}}^{j_{2}}
    \cdots
    \Prn{\frac{x_{k}}{k!}}^{j_{k}},
  \end{equation}
  where $\cP_k = \Set{\Vc{j} \in \bbZ^k_{\scalebox{.65}{${\ge }0$}}:\, j_1 + 2 j_2 + \cdots + k j_k = k}$.

  Although the cardinality of $\cP_k$ grows exponentially with $k$, the Bell polynomials can be calculated instead through the recursion 
  \begin{equation}\label{eq:bell-recursion}
    B_{k}\prn{x_1, \ldots, x_{k}}
    = \sum_{i=0}^{k-1} \binom{k-1}{i} B_{i}(x_1, \ldots, x_{i}) x_{k-i},
  \end{equation}
  with the convention $B_0 = 1$. 
  The first few Bell polynomials are
  \begin{align*}
      B_1(x_1) &= x_1, \\
      B_2(x_1,x_2) &=x_1^2 + x_2, \\
      B_3(x_1,x_2,x_3) &=x_1^3 + 3x_1x_2 + x_3, \qtext{and}\\
      B_4(x_1,x_2,x_3,x_4) &= x_1^4 + 6 x_1^2 x_2 + 4 x_1 x_3 + 3 x_2^2 + x_4.
  \end{align*}
  The partial derivatives are given by
  \begin{equation}\label{eq:bell-derivative}
    \frac{\partial B_{k}}{\partial x_i} \prn{x_1, \ldots, x_{k}}
    = \binom{k}{i}B_{k-i}(x_1, \ldots, x_{k-i}).
  \end{equation}
\end{definition}

Bell polynomials are useful for expressing higher-order moments of random variables in terms of cumulants per the following lemma.
\begin{lemma}[{Moments, Cumulants and Bell polynomials; \citealp[p.~160]{comtet2012advanced}}]
\label{lem:cums_and_bell}
  Let $Z$ be a real random variable and suppose %
  \begin{equation*}
	\gamma(t) = \log\prn*{\E\prn{e^{tZ}}}
	\qtext{and}
	\kappa_k = \frac{\mathrm{d}^{k}\gamma}{\mathrm{d} t^k} (0),\; k\in \Natural
  \end{equation*}
  exist.  Then
  \begin{displaymath}
      \E \prn*{Z^d} = B_d(\kappa_1, \dots, \kappa_d).
  \end{displaymath}
  The function $\gamma(t)$ is called the cumulant generating function, and its $k$-th derivatives evaluated at zero, $\kappa_k$, are called cumulants.
\end{lemma}

In working with moments, it is useful to know
the convention that $0!=1$, so that
$\binom{d}{0} = 1$ for any $d \geq 0$.
Further, any matrix $\Mx{C}$ to the power zero is the identity, i.e., $\Mx{C}^0 = \Id$.

\section{Characterizing Moments of Multivariate Gaussians}

We are interested in the moments of Gaussians primarily as a prelude to understanding
moments of GMMs; nevertheless, some of these results may also have other applications.
We write $X \sim \cN(\m, \C)$ to denote the Gaussian random variable $X \in \R^n$
with mean $\m \in \R^n$ and symmetric positive definite covariance $\C \in \R^{n \times n}$.

In this paper, it is convenient to allow for $\C$ to be only positive \emph{semi}definite.
In the case that $\C$ is rank-deficient, we say $X$ has a degenerate distribution. 
Then $X$ is supported on the affine subspace $\m + \operatorname{colspan}(\C) \subseteq \R^n$, restricted to which it has a probability density function given analogously as in \cref{sec:gauss-gauss-mixt}. 
An important case is when $\C=\Mx{0}$, where $X=\m$
with probability 1 (a discrete distribution with only one option).
The proofs of our results only require $\C$ to be positive semidefinite.

\subsection{Symmetric tensor formulation of Gaussian moments }
\label{sec:symm-tens-form}

Below we present the formulas for the higher-order moments of a Gaussian vector in $\R^n$.
These formulas are obtained almost directly from \cref{lemma:Phi} combined with
the formula for the moments of a one-dimensional random Gaussian variable.
Although we provide a formulation for the moment, we do not recommend computing
these directly. Rather, this is a stepping stone to efficient methods for computing
quantities involving the moments.

\begin{theorem}\label{thm:gaussian_moments}
  If $X \sim \N( \m, \C )$, then
  its $d$th moment $\M = \E(\sop{X})$ for any $d \in \mathbb{N}$ is given by
  \begin{equation} \label{eq:gaussian-moment}
    \M
    = \sum_{k=0}^{\floor{ d/2}}
    C_{d,k} \, \sym\Prn{ \sop[d-2k]{\m} \otimes \sop[k]{\C} },
  \end{equation}
  with $C_{d,k}= \binom{d}{2k} \dblfact{k}$.
\end{theorem}
\begin{proof}
  Using \specialref{enum:Phi_matvec}, we have
  \begin{equation}
    \Phi\Sqr{\E\prn{\sop{X}}}(\Vc{z})
    =  \Ang{ \sop{\Vc{z}}, \E\prn{\sop{X}} }
    = \E\prn*{\ang{ \sop{\Vc{z}}, \sop{X}}}
    = \E\prn*{(\Vc{z}'  X)^d}.
  \end{equation}
  Fixing $\Vc{z}\in \R^n$, we have
  $\Vc{z}' X \sim \mathcal{N}(\Vc{z}' \m, \Vc{z}' \C \Vc{z})$
  per the properties of marginals of multivariate Gaussians.
  The higher-order moment of a univariate Gaussian $Y \sim \N(\mu,\sigma)$ is
  \begin{displaymath}
    \E(Y^d) = \sum_{k=1}^{\floor{d/2}} \binom{d}{2k} \dblfact{k} \mu^{d-2k} \sigma^k
  \end{displaymath}
  per, e.g., \citet[Theorem 3.1]{gut2009probability}.
  Hence, the higher order moments of $\Vc{z}'X$ are given by 
  \begin{align*}
    \Phi\sqr*{\E\prn{\sop{X}}}(\Vc{z}) = \E\prn{(\Vc{z}'X)^d}
    &= \sum_{k=0}^{\floor{d/2}} \binom{d}{2k} \dblfact{k} (\Vc{z}' \m)^{d-2k} (\Vc{z}' \C \Vc{z})^k, \\
    &= \sum_{k=0}^{\floor{d/2}} \binom{d}{2k} \dblfact{k} \Phi\Sqr{\sop[d-2k]{\m} \otimes \sop[k]{\C}}(\Vc{z}).
  \end{align*}
  Here the last line follows from \specialref{enum:Phi_matvec} and \specialref{enum:Phi_prod}.
  Finally, since $\E\prn{\sop{X}}$ is a symmetric tensor, the result follows from \specialref{enum:Phi_1to1}.
\end{proof}
\begin{example}
The first few moments for $X \sim \N(\m,\C)$ are
\begin{align*}
  \M[1] &= \m
  && \in \Real^n, \\
  \M[2] &= \sop[2]{\m} + \C
  && \in \Real^{n \times n}, \\
  \M[3] &= \sop[3]{\m} + 3\sym\prn{ \m \otimes \C}
  && \in \Real^{n \times n \times n},\\
  \M[4] &= \sop[4]{\m} + 6\sym\prn{ \sop[2]{\m} \otimes \C} + 3 \sym\Prn{\sop[2]{\C}}
  && \in \Real^{n \times n \times n \times n}.
\end{align*}
\end{example}
\cref{thm:gaussian_moments} can be alternatively formulated in terms of homogeneous polynomials, using $\Phi$, as
\begin{equation}\label{eq:gaussian-moment-Phi}
\Phi\sqr*{\M} = \E\prn*{\Phi\sqr{X}^d} = \sum_{k=0}^{\floor{ d/2}} C_{d,k} \Phi\sqr{\m}^{d-2k} \Phi\sqr{\C}^k.
\end{equation}
\subsection{Inner product of Gaussian moments}
\label{sec:inner-product-gaussian}

A key proposition characterizes the inner product of two Gaussian moments.
We show that these can be computed efficiently using the Bell polynomials discussed in \cref{sec:bell-polyn-cumul}.
As a standalone result, this proposition has several applications, including kernel learning.
On this topic, similar specific formulas for $d=1,2,3$ are provided by \citet[Table 1]{MuFuDiSc12}.
Nevertheless, to the best of our knowledge, the formula for general $d$ has not been previously discovered.
As for our paper, this proposition is crucial not only for the implicit calculation of the inner product of two Gaussian moments (see \cref{thm:gmm_tensor_norm}) but also to obtain useful recursion formulas for other implicit calculations (see \cref{prop:biaseddotp-single,thm:gmmdotp,prop:debiaseddotp}).
This result is applicable to the case of degenerate covariance matrices
as well, but we defer its proof to \cref{sec:proof-prop:debiaseddotp}.

\vspace{0.5em}

\begin{keyprop}\label{prop:Psi}
	Suppose $X \sim \N( \m, \C )$ and $\tilde X \sim \N( \tm, \tC)$ are independent random variables, let $d\in\mathbb{N}$ and define the dot-product between the dth moments $\M = \E(\sop{X})$ and $\tM = \E(\sop{\tilde X})$ as
	\begin{equation}\label{eq:Psidef}
		\PsiFn{\m}{\C}{\tm}{\tC} := \Ang{\M,\tM}
	\end{equation}	
	Then 
	\begin{equation}\label{eq:Psi_formula}
		\PsiFn{\m}{\C}{\tm}{\tC} =  B_d\prn{\kappa_1, \dots, \kappa_d},
	\end{equation}
	where $B_d$ is the $d$th Bell polynomial and for each $k\in [d]$
	\begin{equation}\label{eq:kappa_def}
		\kappa_k =
		\begin{cases}
			(k-1)!\trace\prn*{\Mx{Z}^{\frac{k}2}} +
			\frac{k!}{2} \prn*{\m'\tC \Mx{Z}^{\frac{k-2}2} \m + \tm' \Mx{Z}^{\frac{k-2}2} \C \tm }
			&\text{if } k \text{ is even},\\
			k! \; \tm'\Mx{Z}^{\frac{k-1}2}\m
			&\text{if } k \text{ is odd},
		\end{cases}
	\end{equation}
	with $\Mx{Z} = \C \tC$.
\end{keyprop}

\begin{proof}
We have
\begin{equation}
	\PsiFn{\m}{\C}{\tm}{\tC}
	=  \ang**{\E\prn*{\sop{X}}, \E\prn*{\sop{\tilde X}}}
	=  \E\prn**{\ang{X, \tilde X}^{d}}.
\end{equation}
Defining the random variable $Y=\ang{X, \tilde{X}}$,
\cref{lem:cums_and_bell} implies that this $d$th moment is given in terms of the cumulants
\begin{displaymath}
	\E\prn{Y^{d}} = B_d(\kappa_1, \dots, \kappa_d)
\end{displaymath}
where
\begin{displaymath}
  \kappa_k=\frac{d^r}{d t^r} \prn**{\log\prn*{\E\prn{e^{t Y}}}}_{t=0}
  \qtext{for all} k \in [d].
\end{displaymath}
We defer the technical details of the calculation of the cumulants, $\kappa_k$, to \cref{prop:cum_formula}.
\end{proof}

These formulas may seem intimidating, but these are scalars
that can be computed directly from the inputs via relatively simple equations.
\begin{example}
The first few $\kappa_k$'s are:
\begin{align*}
  \kappa_1 &= \m'\tm
  ,&
  \kappa_2 &= \trace(\Mx{Z}) + \tm'\C\tm + \m'\tC\m
  ,\\
  \kappa_3 &= 6 \m' \Mx{Z} \tm
  ,&
  \kappa_4 &= 6 \trace(\Mx{Z}^2)
  + 12 \prn*{\tm'\C\Mx{Z}\tm + \m'\Mx{Z}\tC\m}
  .
\end{align*}
To calculate the dot-product we use the scalars $\kappa_k$'s and the recurrence formula for Bell polynomials  \cref{eq:bell-recursion}. For instance, for $d=1,2,3$, we have
\begin{align*}
	\PsiFn[1]{\m}{\C}{\tm}{\tC} &= \kappa_1 = \m'\tm\\
	\PsiFn[2]{\m}{\C}{\tm}{\tC} &= \kappa_1^2 + \kappa_2 = \prn{\m'\tm}^2 + \trace(\Mx{Z}) + \tm'\tC\tm + \m'\C\m,\\
	\PsiFn[3]{\m}{\C}{\tm}{\tC} &= \kappa_1(\kappa_1^2 + \kappa_2) + 2 \kappa_2 \kappa_1 + \kappa_3 = \kappa_1^3 + 3 \kappa_1 \kappa_2 + \kappa_3\\
	&=\prn{\m'\tm}^3 +3\m'\tm\prn*{\trace(\Mx{Z}) + \tm'\C\tm + \m'\tC\m}+ 6 \m'\Mx{Z} \tm.
\end{align*}
\end{example}
It is useful to calculate the derivatives of $\Psi^{(d)}$ with respect to $\m$ and $\C$ for the optimization problem \cref{eq:opt-problem}. However, since these formulas for general GMM models are even lengthier, we include these in \cref{sec:inner_product_derivatives}. Particular cases of these formulas, for instance when the $\C{j}$ are diagonal, are included and explained in \cref{sec:implementation}.

\subsection{Gradients of Gaussian moments}
\label{sec:grad-gauss-moments}

We generally do not need to use the moment formulated in \cref{eq:gaussian-moment} explicitly; instead,
its inner product with a vector outer product can be computed efficiently,
including gradients.
The import is that working with the Gaussian moment tensor \emph{implicitly}
can be extremely efficient in terms of both storage and computations. 
Using \cref{prop:Psi}, we can efficiently compute dot products and corresponding
gradients with
respect to a Gaussian moment, as follows.

\vspace{0.5em}

\begin{theorem}\label{prop:biaseddotp-single}
  Let $X \sim \N( \m, \C )$ and
  its $d$th moment be $\M = \E(\sop{X})$.
  Then for any $\a \in \R^n$, we can compute
  \begin{displaymath}
      \gmdp{d} \equiv \ang{ \M, \sop{\a}}
  \end{displaymath}
  where $\gmdp{d}$ can be calculated implicitly using the recursion formula
  \begin{equation}\label{eq:andotprod-single}    
    \gmdp{d} = \gmdp{d-1}\a' \m +(d-1) \gmdp{d-2} \a' \C \a
  \end{equation}
  with $\gmdp{0} = 1$ and $\gmdp{1} = \a' \m$.
  Furthermore, the gradients are given by
  \begin{align}\label{eq:andotprodder-a}
    \nabla_{\a} \ang{ \M, \sop{\a} }
    &= d \prn**{\gmdp{d-1} \m + (d-1) \gmdp{d-2} \C \a}, \\
    \label{eq:andotprodder-mu}
    \nabla_{\m} \ang{ \M, \sop{\a} }
    & = d \gmdp{d-1} \a,
    \qtext{and} \\
    \label{eq:andotprodder-sigma}
    \nabla_{\C} \ang{ \M, \sop{\a} }
    &= \binom{d}{2} \gmdp{d-2} \a" \a'.
  \end{align}
\end{theorem}

\begin{proof}
Let $\tilde X\sim \N( \a, \bf{0})$, that is, $\tilde X = \a$ with probability $1$, and notice that for any $d \in \Natural$, $\E\prn{\sop{\tilde X}} = \sop{\a}$. Applying \cref{cor:Psi} (the extension of \cref{prop:Psi} to symmetric matrices, not necessarily positive definite),
\begin{displaymath}
  \gmdp{d} =   \ang{ \M, \sop{\a}} = \PsiFn{\m}{\C}{\a}{{\bf 0}}
  = B_d\prn{\a'\m, \a' \C\a, 0, 0, \ldots, 0}.
\end{displaymath}
Then the recursion formula \cref{eq:andotprod-single}
follows from the recursion formula for Bell polynomials, \cref{eq:bell-recursion}.
The derivatives, \cref{eq:andotprodder-a,eq:andotprodder-mu,eq:andotprodder-sigma},
follow from the derivatives of the Bell polynomial, \cref{eq:bell-derivative},
and the chain rule.
\end{proof}

\section{Gaussian mixture models}

Now that we have characterized the moment of a single Gaussian, we can
employ those results for mixtures of Gaussians. %
We write $X \sim \sum_{j=1}^m \lambda_j \; \N( \m{j}, \C{j} )$ to denote the GMM
with $m$ components.
The value $\lambda_j \in [0,1]$ is the probability of component $j$,
and $\sum_{j=1}^m \lambda_j = 1$.
The $j$th Gaussian has mean $\m{j}$ and covariance $\C{j}$.

\subsection{Symmetric tensor formulation of GMM moments}
\label{sec:symm-tens-form-1}

An immediate corollary of the characterization of moments for a Gaussian, \cref{thm:gaussian_moments}, is a similar result for GMMs, as follows.

\begin{theorem}\label{prop:GMMbiasedmoments}
  If $X \sim \sum_{j=1}^m \lambda_j \; \N( \m{j}, \C{j} )$, then
  its $d$th moment $\M = \E\prn{\sop{X}}$ for any $d\geq2$ is given by
  \begin{equation}
    \M = \sum_{j=1}^m \sum_{k=0}^{\floor{d/2}} \lambda_j \, C_{d,k} \;
    \sym\prn**{ \sop[d-2k]{\m{j}} \otimes \sop[k]{\C{j}} }.
  \end{equation}
  with $C_{d,k}= \binom{d}{2k} \dblfact{k}$.
\end{theorem}
\begin{proof}
Let $X_j \sim \N( \m{j}, \C{j} )$, and $\tau$ a discrete integer random variable such that $P(\tau = j) = \lambda_j$. We then have that $X$ and $X_\tau$ have the same distribution, hence
$$\E\prn{\sop{X}} = \E_\tau\Prn{\E\prn{\sop{X_\tau}|\tau}} = \sum_{j=1}^m \lambda_j \E\prn{\sop{X_j}}$$
and the result follows by applying \cref{thm:gaussian_moments}.
\end{proof}

\subsection{Moment matching for GMMs with efficient computations}
\label{sec:moment-matching-gmms}

Suppose we have realizations $\set{\miwc[\x][p]}$ of a GMM and wish to determine
its parameters.
We can formulate a \emph{moment-matching} optimization problem of the form in
\cref{eq:opt-problem}.
Discarding the constant term, we are left with the following optimization problem:
\begin{equation}\label{eq:opt-problem-detailed}
\begin{aligned}
  \min_{\theta} \hspace{1em} & \fGMM(\theta) \equiv
  \underbrace{\nrm*{\M}^2}_{\fGMM_1(\theta)} -
  2 \underbrace{\frac1{p}\sum_{i=1}^p \ang**{\M, \sop{\x{i}}}}_{\fGMM_2(\theta)} \\
  \text{subject to}\hspace{1em} 
  & \M = \sum_{j=1}^m \sum_{k=0}^{\floor{d/2}} \lambda_j \, C_{d,k} \;%
    \sym\prn**{ \sop[d-2k]{\m{j}} \otimes \sop[k]{\C{j}} }, \\
    & \sum_{j=1}^m \lambda_j = 1, \; \lambda_j\ge 0, \; \m{j} \in \Real^n, \; \C{j} \in \Real^{n \times n}, \C{j} \succ 0, \\
    & \theta = \Set{ \prn{ \lambda_j, \m{j}, \C{j} }}_{j=1}^m .
\end{aligned}
\end{equation}
To solve this optimization problem using first-order methods, we need to be able
to efficiently compute
\begin{displaymath}
  \fGMM, \quad \FD{\fGMM}{\lambda_j}, \quad \FD{\fGMM}{\m{j}}, \qtext{and} \FD{\fGMM}{\C{j}}
  \qtext{for all} j \in [m].
\end{displaymath}
Let us consider the two terms in the objective function \cref{eq:opt-problem-detailed}, $\fGMM_1(\theta)$ and $\fGMM_2(\theta)$, independently.
The first term is calculated efficiently via a corollary of \cref{prop:Psi}. 

\begin{theorem}\label{thm:gmm_tensor_norm}
  Let $X \sim \sum_{j=1}^m \lambda_j \; \N( \m{j}, \C{j} )$
  and its $d$th moment be $\M = \E(\sop{X})$.
  Then
  \begin{displaymath}
    \fGMM_1(\theta) \equiv \nrm*{\M}^2 = \sum_{i=1}^m\sum_{j=1}^m \lambda_i\lambda_j
    B_d\prn{\kappa^{(ij)}_1, \dots, \kappa^{(ij)}_d},
  \end{displaymath}
  where $B_d$ is the $d$-th complete Bell polynomial and for each $k\in [d]$
  \begin{equation}\label{eq:kappa_def_ij}
    \kappa_k^{(ij)} =
    \begin{cases}
      (k{-}1)!\trace\prn{\Mx{Z}{ij}^{\frac{k}2}} + 
      \frac{k!}{2} \Prn{ v_k^{(ij)} + v_k^{(ji)} }
      & \text{if $k$ even}, \\
      k! \m{j}' \Mx{Z}{ij}^{\frac{k-1}{2}} \m{i}
      & \text{if $k$ odd},
    \end{cases}
  \end{equation}
  with $\Mx{Z}_{ij} = \C{i}\C{j}$ and $v_k^{(ij)} = \m{i}' \C{j}  \Mx{Z}{ij}^{\frac{k-2}2}  \m{i}$.
\end{theorem}

Gradients of $\fGMM_1$ in terms of $\C{j}, \m{j}$ and $\lambda_j$ follow from \cref{prop:Psi_derivatives}. The term $\fGMM_2(\theta)$ in \cref{eq:opt-problem-detailed} is computed using an immediate corollary of \cref{prop:biaseddotp-single},
as follows.

\begin{theorem}\label{thm:gmmdotp}
	Let $X \sim \sum_{j=1}^m \lambda_j \; \N( \m{j}, \C{j} )$,
	its $d$th moment be $\M = \E(\sop{X})$ and $\set{\miwc[\x][p]}$ realizations of a GMM. Then
	\begin{equation*}%
		\fGMM_2(\theta) \equiv \frac1{p}\sum_{i=1}^p\Ang{\M, \sop{\x{i}}} = \frac1{p}\sum_{i=1}^p\sum_{j=1}^m \lambda_j \; \alpha_{ij}^{(d)},
	\end{equation*}
	where
	$\alpha_{ij}^{(d)}$ can be calculated implicitly using the recursion formula
	\begin{equation}\label{eq:alpha_recursion_formula}
		\alpha_{ij}^{(d)} =
		\alpha_{ij}^{(d-1)} \x{i}' \m{j}" +(d-1) \alpha_{ij}^{(d-2)} \x{i}' \C{j}" \x{i}"
	\end{equation}
	with the convention that $\alpha_{ij}^{(0)} = 1$ and $\alpha_{ij}^{(1)} = \x{i}' \m{j}"$.
	Furthermore, the gradients are given by
	\begin{align*}
		\nabla_{\m{j}} \fGMM_2(\theta)
		&= \frac{d}{p} \lambda_j \sum_{i=1}^p \alpha_{ij}^{(d-1)} \x{i}, \\
		\nabla_{\C{j}} \fGMM_2(\theta)
		&= \frac1{p}\binom{d}{2} \lambda_j \sum_{i=1}^p \alpha_{ij}^{(d-2)} \x{i}" \x{i}',
		\qtext{and} \\
		\nabla_{\lambda{j}} \fGMM_2(\theta)  &= \frac1{p}\sum_{i=1}^p \; \alpha_{ij}^{(d)}.
	\end{align*}
\end{theorem}
\begin{proof}[Proof of \cref{thm:gmm_tensor_norm,thm:gmmdotp}]
	Let
	$\M_j \equiv C_{d,k} \;\sym \prn**{ \sop[d-2k]{\m{j}} \otimes \sop[k]{\C{j}} }$.
	Then \cref{prop:GMMbiasedmoments} implies that $\M = \sum_{j=1}^m \lambda_j \M_j$, and the results follow from replacing this into \eqref{eq:opt-problem-detailed}, and 
	applying \cref{prop:Psi} and \cref{prop:biaseddotp-single}.
\end{proof}

\subsection{Augmented system for GMM}
\label{sec:augmented-system-gmm}

The optimization problem of \cref{eq:opt-problem-detailed} is not enough to fully determine the underlying model $X\sim \sum_{j=1}^m \lambda_j \; \N( \m{j}, \C{j} )$, since it has multiple possible solutions.
Let $\ARV\sim \sum_{j=1}^m \lambda_j^{\Tr*} \gamma_j^{-d} \; \N( \gamma_j^{\Tr*}\m{j}", \gamma_j^{2}\C{j}")$, where $\Vc{\gamma}\in \Real^m$ is such that $\sum_{j=1}^m  \lambda_j^{\Tr*} \gamma_j^{-d} = 1$. Then \cref{prop:GMMbiasedmoments} implies that
\begin{align*}
\E\prn{\sop{X}} &=\sum_{j=1}^m \sum_{k=0}^{\floor{d/2}} \lambda_j \, C_{d,k} \;
\sym\prn**{ \sop[d-2k]{\m{j}} \otimes \sop[k]{\C{j}} }\\
&= \sum_{j=1}^m \sum_{k=0}^{\floor{d/2}} \lambda_j^{\Tr*} \gamma_j^{-d} \, C_{d,k}  \;
\sym\prn**{ \sop[d-2k]{(\gamma_j \m{j})} \otimes \sop[k]{(\gamma_j^{2} \C{j})} } = \E\prn{\sop{\ARV}}
\end{align*}
This implies that \cref{eq:opt-problem-detailed} is not able to distinguish two models $X$ and $\ARV$ related by this scale ambiguity. To address this we propose an augmented model that implicitly considers several moment orders at the same time. Let $\set{\x{1},\dots,\x{p}}$ be a set of realizations of $X$. Then the augmented model with constant $\omega\in \R$ is defined by
\begin{equation*}
	\bar X \sim \sum_{j=1}^m \lambda_j \; \N( \Vc[\bar]{\mu}{j}, \Mx[\bar]{\Sigma}{j} )
	\qtext{where}
	\Vc[\bar]{\mu}{j} = \begin{bmatrix} \m{j} \\ \omega \end{bmatrix}
	\qtext{and}
	\Mx[\bar]{\Sigma}{j} = \begin{bmatrix} \C{j} & \bf 0 \\ \bf 0 & 0 \end{bmatrix}.
\end{equation*}
Note that $\bar X(1{:}n)$ coincides in distribution with $X$, and $\bar X(n+1) = \omega$ with probability $1$, therefore the augmented observations
\begin{equation}\label{eq:augmented_samples}
	\Vc[\bar]{x}{i} =
	\begin{bmatrix}
		\x{i} \\ \omega
	\end{bmatrix} \in \R^{n+1},
\end{equation}
are samples of $\bar X$. We then propose to solve the augmented optimization problem,
\begin{equation}\label{eq:opt-problem-aug}
\begin{aligned}
  \min_{\bar \theta} \hspace{1em} & \fGMM(\bar \theta) \equiv
  \nrm*{\M[d][\bar]}^2 - \frac2{p} \sum_{i=1}^p \Ang{ \M[d][\bar], \sop{\Vc[\bar]{x}{i}}} \\
	\text{subject to}\hspace{1em} 
	& \M[d][\bar] = \sum_{j=1}^m \sum_{k=0}^{\floor{d/2}} \lambda_j \, C_{d,k}%
	\sym\prn**{ \sop[d-2k]{\Vc[\bar]{\mu}{j}} \otimes \sop[k]{\Mx[\bar]{\Sigma}{j}} } \\
	& \sum_{j=1}^m \lambda_j = 1, \; \lambda_j\ge 0,\; \Vc[\bar]{\mu}{j} \in \Real^{n{+}1}, \; \Mx[\bar]{\Sigma}{j} \in \Real^{(n{+}1) \times (n{+}1)}, \Mx[\bar]{\Sigma}{j} \succeq 0, \\
	& \Vc[\bar]{\mu}{j}(n{+}1) = \omega,\; \Mx[\bar]{\Sigma}{j}({:},n{+}1) = \Mx[\bar]{\Sigma}{j}(n{+}1,{:}) = \Vc{0}\\
	& \bar\theta = \Set{ \prn{ \lambda_j, \Vc[\bar]{\mu}{j}, \Mx[\bar]{\Sigma}{j} } }_{j=1}^m.
\end{aligned} 
\end{equation}
Suppose that $\ARV*$ is independent and identically distributed as $\bar{X}$, we then have
\begin{align*}
  \fGMM_1(\bar \theta)
  = \nrm*{\M[d][\bar]}^2 = \nrm*{\E\prn{\sop{\bar{X}}}}^2 = \Ang{\E\prn{\sop{\ARV*}},\E\prn{\sop{\bar{X}}}} = \E\prn*{\ang*{\ARV*,\bar{X}}^d}
\end{align*}
Since $\bar{X}(1{:}n)$ has the same distribution of $X$, and $\bar{X}(n{+}1) = \omega$ with probability 1, we have,
$\E\prn*{\ang*{\ARV*,\bar{X}}^d} = \E\prn*{(\ang{\ARV,X}+\omega^2)^d}$ and
\begin{align*}
  \fGMM_1(\bar \theta)
  &= \E\prn**{\prn*{\ang{\ARV,X}+\omega^2}^d}
  = \sum_{k=0}^d \binom{d}{k} \omega^{2(d-k)} \E\prn*{\ang{\ARV,X}^k}\\
  &= \sum_{k=0}^d \binom{d}{k} \omega^{2(d-k)} \nrm*{\M[k]}^2
  =  \sum_{k=0}^d \binom{d}{k} \omega^{2(d-k)} \fGMM[k]_1(\theta),
\end{align*}
with the convention $\fGMM[0]_1(\theta) = \E\prn*{\ang{\ARV,X}^0} = 1$. Similarly, with the conventions $\fGMM[0]_2(\theta) = 1$ and $\fGMM[0](\theta) = -1$, we have $\fGMM_2(\bar \theta) = \sum_{k=0}^d \binom{d}{k} \omega^{2(d-k)} \fGMM[k]_2(\theta)$, and 
\begin{displaymath}
\fGMM(\bar \theta) = \sum_{k=0}^d \binom{d}{k} \omega^{2(d-k)} \fGMM[k](\theta).
\end{displaymath}
Therefore the augmented model puts a weight of $\binom{d}{k} \omega^{2(d-k)}$ in the moment of order $k$. We conclude that the augmented optimization problem \cref{eq:opt-problem-aug} is equivalent to
\begin{equation*}%
	\begin{aligned}
		\min_{\theta} \hspace{1em} & \sum_{k=0}^d \binom{d}{k} \omega^{2(d-k)}\fGMM[k](\theta) \\
		\text{subject to}\hspace{1em} 
		& \sum_{j=1}^m \lambda_j = 1, \; \lambda_j\ge 0,\; \m{j} \in \Real^n, \; \C{j} \in \Real^{n \times n}, \C{j} \succ 0, \\
		&\theta = \Set{ \prn{ \lambda_j, \m{j}, \C{j} }}_{j=1}^m .
	\end{aligned}
\end{equation*}

\section{Debiasing and tensor decomposition for the case of common covariance}
\label{sec:debi-tens-decomp}

Beyond the formula for the moments of GMMs \cref{eq:Md_definition}, our techniques can also be used to implicitly debias moments of general GMM models, with common covariance.
The setting of this section is as follows. Suppose that we have N i.i.d.\@
samples $\set{\x{1},\dots \x{p}} \subset \R^n$ from a distribution $X$,
which decomposes as $X = Y + Z$.
Here $Z \sim \N( \Vc{0}, \C )$ and the moments of
$Y$ have a low-rank structure, which we want to exploit.
More specifically, we want to obtain the moments of $Y$ in terms of the samples $\set{\x{1},\dots \x{p}}$.

\begin{theorem}\label{thm:GMMdebiasedmoments}
  Let $X = Y + Z$ where $Y$ and $Z$ are independent and
  $Z \sim \N( \bf{0}, \C)$. Let
  \begin{equation}
    \T = \E\prn{\sop{Y}}%
  \end{equation}
  be the $d$th moment of $Y$.
  Then
  \begin{equation}\label{eq:Test}
    \Test = \frac{1}{p}
    \sum_{i=1}^p \sum_{k=0}^{\floor{d/2}} C_{d,k} (-1)^k
    \sym \Prn{  \sop[d-2k]{\x{i}} \otimes \sop[k]{\C} },
  \end{equation}
  with $C_{d,k}= \binom{d}{2k} \dblfact{k}$,
  is an unbiased estimator of $\T$ $(\E\prn{\Test} = \T)$.
\end{theorem}

\begin{proof}[Proof of \cref{thm:GMMdebiasedmoments}]
Conditioned on $Y$, we have $X\sim \N(Y, \C)$. Then, using \cref{eq:gaussian-moment-Phi}, we have for all $k\in [d]$
\begin{align*}
\Phi\Sqr{\E\prn{X^{\otimes k}}} &=\E_Y\Sqr{\E\prn{\Phi\sqr{X}^k|Y}}=\sum_{r=0}^{\floor{ k/2}}
C_{k,r}\,\E_Y\Sqr{\Phi\sqr{Y}^{k-2p}}\Phi\sqr{\C}^r
\\
&=\sum_{r=0}^{\floor{ k/2}}
C_{k,r}\,\Phi\sqr*{\T[k-2p]}\Phi\sqr{\C}^r ,
\end{align*}
where we used that $\E_Y\Sqr{\Phi\sqr{Y}^{k-2p}} = \Phi\sqr*{\E_Y\sqr*{\sop[k-2p]{Y}}} = \Phi\sqr*{\T[d-2k]}$. Therefore we have
\begin{align*}
\Phi\sqr*{\E\prn*{\Test}} &= \sum_{k=0}^{\floor{d/2}} C_{d,k} (-1)^k
\E\prn*{\Phi\sqr*{X}^{d-2k}} \Phi\Sqr{\C}^k,\\
&= \sum_{k=0}^{\floor{d/2}} C_{d,k} (-1)^k
\sum_{r=0}^{\floor{ d/2}-k}
	C_{d-2k,r}\,  \Phi\sqr*{\T[d-2k-2p]}\Phi\sqr{\C}^r \Phi\Sqr{\C}^k,\\
&= \sum_{k=0}^{\floor{d/2}} \sum_{r=0}^{\floor{d/2}-k} C_{d,k} (-1)^k  C_{d-2k,r} 
\Phi\sqr*{\T[d-2k-2p]}\Phi\sqr{\C}^{r+k},\\
&= \sum_{q=0}^{\floor{d/2}} \Phi\sqr*{\T[d-2q]}\Phi\sqr{\C}^{q} \sum_{k=0}^{q} C_{d,k} (-1)^k C_{d-2k,q-k}, \numberthis \label{eq:hatTd_expectation}\\
&= \Phi\sqr*{\T[d]}
\end{align*}	
where we set $q = r + k$, and use \cref{lem:Cdelta0qsum}. The result now follows from \specialref{enum:Phi_1to1}.
\end{proof}

\subsection{Application of debiasing to GMMs}
\label{sec:appl-debi-gmm}

An application of \cref{thm:GMMdebiasedmoments} is a GMM with a known common covariance. Suppose $X\sim \sum_{j=1}^m \lambda_j \; \N( \m{j}, \C )$, then $X$ can be decomposed as in \cref{thm:GMMdebiasedmoments}, with $Y\sim \sum_{j=1}^m \lambda_j \; \N( \m{j}, \bf{0} )$. Thus we have
\begin{displaymath}
 \T  = \sum_{j=1}^m \lambda_j \sop{\m{j}},  
\end{displaymath}
and the mixture components may be obtained from the tensor decomposition of $\Test$.
We propose to obtain the decomposition by solving
the optimization problem
\begin{equation}\label{eq:opt-problem-debias}
  \begin{aligned}
  	\min_{\theta} \hspace{1em} & \nrm***{ \sum_{j=1}^m \lambda_j \sop{\m{j}} - \Test }^2\\
  	\text{subject to}\hspace{1em} 
  	& \sum_{j=1}^m \lambda_j = 1, \; \lambda_j\ge 0,\; \m{j} \in \Real^n, \; \theta = \Set{ \prn{ \lambda_j, \m{j}}}_{j=1}^m .
  \end{aligned}
\end{equation}
where $\Test$ is as defined in \cref{thm:GMMdebiasedmoments}.
Discarding the constant, $\nrm{\Test}^2$, this can be rewritten as
\begin{displaymath}
	\begin{aligned}
		\min_{\theta} \hspace{1em} & \fdeb(\theta) \equiv
		\underbrace{\sum_{i=1}^m \sum_{j=1}^m \lambda_i \lambda_j \ang{\m{i},\m{j}}^d }_{\fdeb_1(\theta)}
		-2 \underbrace{\sum_{j=1}^m \lambda_j \Ang{ \sop{\m{j}}, \Test}}_{\fdeb_2(\theta)}\\
		\text{subject to}\hspace{1em} 
		& \sum_{j=1}^m \lambda_j = 1, \; \lambda_j\ge 0,\; \m{j} \in \Real^n, \; \theta = \Set{ \prn{ \lambda_j, \m{j}}}_{j=1}^m .
	\end{aligned}
\end{displaymath}
which we break into two terms for convenience of the discussion. The term $\fdeb_1(\theta)$ is already expressed in a way that allows for efficient implicit calculation, and its gradients were analyzed in \citet{sk-2019}.
$$\FD{\fdeb_1(\theta)}{\m{j}} = 2d\lambda_j \sum_{i=1}^m \lambda_i \ang{\m{i},\m{j}}^{d-1} \m{i} \qtext{and} \FD{\fdeb_1(\theta)}{\lambda_{j}} = 2\sum_{j=1}^m \lambda_i \ang{\m{i},\m{j}}^d. $$
Regarding the term $\fdeb_2(\theta)$, we calculate it using a recursion similar to that of \cref{prop:biaseddotp-single}.

\vspace{0.5em}

\begin{theorem}\label{prop:debiaseddotp}
	Define $\Test$ as in \cref{thm:GMMdebiasedmoments}.
	Then
	\begin{equation}\label{eq:andotprod_debias}
		\fdeb_2(\theta) \equiv \sum_{j=1}^m \lambda_j \Ang{\sop{\m{j}},  \Test}  = \frac{1}{p}  \sum_{j=1}^m\sum_{i=1}^p \lambda_j \umdp[ij]{d}.
	\end{equation}
	where $ \umdp[ij]{d}$ can be calculated implicitly using the
	recursion formula
	\begin{equation}\label{eq:beta_recursion_formula}
		 \umdp[ij]{d} =   \umdp[ij]{d} \m{j}' \x{i} -(d-1) \umdp[ij]{d-2} \m{j}'\C \m{j},
	\end{equation}
	with the convention that $\umdp[ij]{0}=1$ and $\umdp[ij]{1}= \m{j}' \x{i} $.
	Further, the gradient is given by
	\begin{align}
		\label{eq:andotprodder_debias}
		\FD{\fdeb_2(\theta)}{\m{j}} 
		&= \frac{d}{p}\lambda_j \sum_{i=1}^p \umdp{d-1} \x{i} - (d-1) \umdp{d-2} \C \m{j}\\\
		\label{eq:andotprodder_debias_lambda}
		\FD{\fdeb_2(\theta)}{\lambda_{j}}
		&= \frac{1}{p} \sum_{i=1}^p  \umdp[ij]{d}.
	\end{align}
\end{theorem}

\begin{proof}
	Let $\Test_i = \sum_{k=0}^{\floor{d/2}} C_{d,k} (-1)^k
	\sym \Prn{  \sop[d-2k]{\x{i}} \otimes \sop[k]{\C} }$ and $\umdp[ij]{d} = \Ang{\sop{\m{j}},  \Test_i}$, then \cref{eq:andotprod_debias} follows from \cref{thm:GMMdebiasedmoments}. Furthermore, 
	\begin{align*}
		\umdp[ij]{d} &:= \left\langle 
		\sum_{k=0}^{\floor{ d/2 }} C_{d,k} (-1)^k
		\sym\Prn{  \x{i}^{\otimes (d-2k)} \otimes \C^{\otimes k} },
		\m{j}^{\otimes d}
		\right\rangle,\\
		&= \left\langle 
		\sum_{k=0}^{\floor{ d/2 }} C_{d,k}
		\sym\Prn{  \x{i}^{\otimes (d-2k)} \otimes (-\C)^{\otimes k} },
		\m{j}^{\otimes d}
		\right\rangle,\\
		&=\PsiFn{\x{i}}{-\C}{\m{j}}{\bf{0}} = B_d\prn{\m{j}'\x{i}, -\m{j}' \C\m{j}, 0, 0, \ldots}.
	\end{align*}
	Here we use an extension of \cref{prop:Psi} to symmetric matrices that are not necessarily positive definite, which we state and prove in \cref{cor:Psi}. From this, the rest of the properties follow from properties of Bell polynomials, \cref{def:BellPolynomials}; the recursion formula follows from \cref{eq:bell-recursion}, and the derivative formula \eqref{eq:andotprodder_debias} follows from \cref{eq:bell-derivative} and the chain rule.
\end{proof}

\subsection{Augmented system for debiasing}
\label{sec:augmented-system-debiased}

Similarly to \cref{sec:augmented-system-gmm}, the optimization problem \cref{eq:opt-problem-debias} has a scaling ambiguity. If $X\sim \sum_{j=1}^m \lambda_j \; \N( \m{j}, \C )$ and $\tilde X\sim \sum_{j=1}^m \gamma_j^{-d}\lambda_j \; \N( \gamma_j \m{j}, \C )$, where $\Vc{\gamma}\in \Real^m$ is such that $\sum_{j=1}^m  \lambda_j^{\Tr*} \gamma_j^{-d} = 1$, then
\begin{displaymath}
	\T  = \sum_{j=1}^m \lambda_j \sop{\m{j}} = \sum_{j=1}^m \lambda_j \gamma_j^{-d}\sop{(\gamma_j \m{j})} = \T[d][\tilde].
\end{displaymath}
Here we propose the augmented model
\begin{equation*}
	\bar X \sim \sum_{j=1}^m \lambda_j \; \N( \maug{j}, \Caug )
	\qtext{where}
	\Vc[\bar]{\mu}{j} = \begin{bmatrix} \m{j} \\ \omega \end{bmatrix}
	\qtext{and}
	\Mx[\bar]{\Sigma} = \begin{bmatrix} \C & \bf 0 \\ \bf 0 & 0 \end{bmatrix}.
\end{equation*}
With the augmented observations defined as in \cref{eq:augmented_samples}, we propose to solve the augmented optimization problem,

\begin{equation}
	\begin{aligned}
		\min_{\bar \theta} \hspace{1em} & \fdeb(\bar \theta) \equiv \sum_{i=1}^m \sum_{j=1}^m \lambda_i \lambda_j \ang{\m[\bar]{i},\m[\bar]{j}}^d
		-2 \sum_{j=1}^m \lambda_j \Ang{ \sop{\maug{j}}, \Testaug}\\
		\text{subject to}\hspace{1em} 
		& \Testaug = \frac{1}{p}
		\sum_{i=1}^p \sum_{k=0}^{\floor{d/2}} C_{d,k} (-1)^k
		\sym \Prn{  \sop[d-2k]{\x[\bar]{i}} \otimes \sop[k]{\Caug} }\\
		& \sum_{j=1}^m \lambda_j = 1,  \; \lambda_j\ge 0, \; \maug{j} \in \Real^{n{+}1}, \;\maug{j}(n{+}1) = \omega, \\
		& \bar\theta = \Set{ \prn{ \lambda_j, \maug{j}}_{j=1}^m}. 
	\end{aligned} \label{eq:opt-problem-aug-debias}
\end{equation}
which, similarly to \cref{sec:augmented-system-gmm}, is equivalent to
\begin{equation*}%
	\begin{aligned}
		\min_{\theta} \hspace{1em} & \sum_{k=0}^d \binom{d}{k} \omega^{2(d-k)}\fdeb[k](\theta) \\
		\text{subject to}\hspace{1em} 
		& \sum_{j=1}^m \lambda_j = 1, \; \lambda_j\ge 0, \; \m{j} \in \Real^n,
		\theta = \Set{ \prn{ \lambda_j, \m{j}}_{j=1}^m }.
	\end{aligned}
\end{equation*}

\section{Implementation and Computational aspects}
\label{sec:implementation}

This section is focused on the practical details of the implementation
of these methods, grouping calculations to be matrix- rather than
vector-based for efficiency. Some readers may opt to skip
this section entirely or refer only to the algorithms which are self-contained.

\subsection{Moment matching for GMMs}
\label{sec:implementation-gmm}

Here we describe our implementation of the calculation of $\fGMM(\theta)$ and its derivatives, which can be used by any first-order optimization method to solve \eqref{eq:opt-problem-detailed}. Computing $\fGMM(\theta)$ for general covariance matrices can be costly, with a overall computational complexity of $O(n^3 m^2)$. Henceforth, we focus on the diagonal case, which is a common assumption to reduce the complexity.

Recalling that $\fGMM(\theta) = \fGMM_1(\theta) - 2 \fGMM_2(\theta)$ from \cref{eq:opt-problem-detailed}, we show how each term is calculated in the following sections
under the assumption that $\C_{j} = \diag\prn{\Vc{d}{j}}^2$.
It is convenient to use entry-wise products for calculations involving diagonal covariances. For two vectors $\Vc{a}, \Vc{b} \in \R^n$, we denote their entry-wise product by $\Vc{a}\ew \Vc{b}$ and the entry-wise power by $\Vc{a}^{\ewpow{d}}$. We use the following identity involving entry-wise product, that is valid for any vectors $\Vc{a}, \Vc{b}, \Vc{c} \in \R^n$,
\begin{displaymath}
  \Vc{a}' \diag(\Vc{b}) \Vc{c} = \Vc{a}' (\Vc{b}\ew \Vc{c}) =\Vc{b}' (\Vc{a}\ew \Vc{c}) = \Vc{c}' (\Vc{b}\ew \Vc{a}).
\end{displaymath}
Similarly, for any two matrices $\Mx{A},\Mx{B}$ of the same size,
we let $\Mx{A}\ew\Mx{B}$ denote their entry-wise product.

\subsubsection[Calculating moment norm]{Calculating $\fGMM_1(\theta)$}
First, we consider some of the constituent elements.
For \eqref{eq:kappa_def_ij}, we have
$\Mx{Z}{ij} = \diag(\Vc{d}_i)^2\diag(\Vc{d}_j)^2
= \diag\prn*{\Vc{d}_i^{\ewpow{2}} \ew \Vc{d}_j^{\ewpow{2}}}  $;
moreover, since diagonal matrices commute, we have
$\Mx{Z}{ij}^{k} = \diag(\Vc{d}_i)^{2k}\diag(\Vc{d}_j)^{2k}$
for any integer $k$.
Further,
\begin{align*}
  v_k^{(ij)}
  &= \m{i}' \C{j}  \Mx{Z}{ij}^{\frac{k-2}2}  \m{i}
  =  \m{i}'\diag(\Vc{d}_i)^{k-2} \diag(\Vc{d}_j)^{k}  \m{i}\\
  &=\m{i}'\prn*{\Vc{d}{i}^{\ewpow{k-2}} \ew \Vc{d}{j}^{\ewpow{k}}\ew \m{i}}
  = \prn*{\Vc{d}{i}^{\ewpow{k-2}} \ew \m{i}^{\ewpow{2} }}^{\Tr} \Vc{d}{j}^{\ewpow{k}}.
\end{align*}
Using analogous manipulations, we obtain
\begin{displaymath}
	\kappa_k^{(ij)} =
	\begin{cases}
		(k{-}1)!\prn*{\Vc{d}{i}^{\ewpow{k}}}^{\Tr} \Vc{d}{j}^{\ewpow{k}} + 
		\displaystyle\frac{k!}{2} \prn*{ v_k^{(ij)} + v_k^{(ji)} }
		& \text{if $k$ even}, \bigskip \\
		k! \prn*{\Vc{d}{i}^{\ewpow{k-1}} \ew \m{i}}^{\Tr} \prn*{\Vc{d}{j}^{\ewpow{k-1}} \ew \m{j}}
		& \text{if $k$ odd}.
	\end{cases}
\end{displaymath}
Using these formulas and $\nabla_{\Vc{y}}\Prn{\Vc{z}'\Vc{y}^{\ewpow{k}}} = k \Vc{z}\ew \Vc{y}^{\ewpow{k-1}}$, we calculate the gradients
{\setlength{\jot}{\bigskipamount}
\begin{align*}
  \nabla_{\m{j}}\kappa_k^{(ij)} &=
  \begin{cases}
    k! \; \Vc{d}{i}^{\ewpow{k}} \ew \Vc{d}{j}^{\ewpow{k-2}} \ew \m{j}  
    & \text{if $k$ even}, \\
    k! \; \Vc{d}{i}^{\ewpow{k-1}} \ew \Vc{d}{j}^{\ewpow{k-1}} \ew \m{i}
    & \text{if $k$ odd},
  \end{cases}\bigskip
  \\
  \nabla_{\Vc{d}{j}}\kappa_k^{(ij)} &=
  \begin{cases}
    \displaystyle k! \; \Vc{d}{i}^{\ewpow{k}} \ew \Vc{d}{j}^{\ewpow{k-1}} + 
    \frac{k!}{2} \Prn{ \nabla_{\Vc{d}{j}}v_k^{(ij)} + \nabla_{\Vc{d}{j}}v_k^{(ji)} }
    & \text{if $k$ even}, \medskip\\
    k!\Prn{k-1} \Vc{d}{i}^{\ewpow{ k-1}} \ew \Vc{d}{j}^{\ewpow{k-2}} \ew  \m{i} \ew \m{j}
    & \text{if $k$ odd},
  \end{cases}
  \\
  \nabla_{\Vc{d}{j}}v_k^{(ij)} &= k \; \Vc{d}{i}^{\ewpow{ k-2}} \ew \Vc{d}{j}^{\ewpow{k-1}} \ew \m{i}^{\ewpow{ 2}},
  \\
  \nabla_{\Vc{d}{j}}v_k^{(ji)} &= (k-2) \; \Vc{d}{i}^{\ewpow{ k}} \ew \Vc{d}{j}^{\ewpow{k-3}} \ew \m{j}^{\ewpow{ 2}} .
\end{align*}}%
It can be checked that these formulas coincide with the formulas in \cref{sec:inner_product_derivatives}, when the covariances are assumed to be diagonal.

Now we use these to compute  $\fGMM_1(\theta) = \nrm{\M}^2$.
In practice, we use matrix operations to calculate $\fGMM_1(\theta)$. Let $\Mx{B}{k}$ and $\Mx{K}{k}$ such that $\Mx{B}{k}[i, j] = B_k\prn{\kappa^{(ij)}_1, \dots, \kappa^{(ij)}_k}$ and $\Mx{K}{k}[i, j]=\kappa^{(ij)}_k$, and define the vector and matrix quantities,
\begin{align*}
	\Vc{\lambda}
	&= \sqrenum!{\lambda}{m} \in \R^{1 \times m},
        \\ 
	\Mx{A}
	&= \sqrenum!{\m}{m} \in \R^{n \times m},
	\\
	\Mx{D}
	&= \sqrenum!{\Vc{d}}{m} \in \R^{n \times m}.
\end{align*}
Using matrix operations we can calculate function values and gradients. For instance, suppose that $k$ is odd, we then have 
\begin{align*}
	\Mx{K}{k} [i, j] &= k! \prn*{\Vc{d}{i}^{\ewpow{k-1}} \ew \m{i}}^{\Tr} \prn*{\Vc{d}{j}^{\ewpow{k-1}} \ew \m{j}} = k! \sum_{\ell=1}^n \Vc{d}{i} [\ell]^{k-1} \m{i} [\ell] \Vc{d}{j} [\ell]^{k-1} \m{j} [\ell], \\
	&= k! \sum_{\ell=1}^n \Mx{D} [\ell, i]^{k-1} \Mx{A} [\ell, i] \Mx{D} [\ell, j]^{k-1} \Mx{A} [\ell, j], \\
	&= k! \prn**{\prn*{\Mx{D}^{\ewpow{k-1}} \ew \Mx{A}}^{\Tr} \prn*{\Mx{D}^{\ewpow{k-1}} \ew \Mx{A}}}[i,j].
\end{align*}
Analogously, we have
\begin{displaymath}
	\Mx{K}{k}=\begin{cases}
		(k{-}1)! \prn*{\Mx{D}^{\ewpow{k}}}^{\Tr}  \Mx{D}^{\ewpow{k}} + 
		\frac{k!}{2} \prn*{\Mx{V}{k} + \Mx{V}{k}'}
		& \text{if $k$ even},\medskip \\
		k! \prn*{\Mx{D}^{\ewpow{ k-1}}\ew \Mx{A}}^{\Tr} \prn*{\Mx{D}^{\ewpow{k-1}}\ew \Mx{A}}
		& \text{if $k$ odd},
	\end{cases}
\end{displaymath}
with $\Mx{V}{k} = \prn*{\Mx{D}^{\ewpow{k}}}^{\Tr} \prn*{\Mx{D}^{\ewpow{k-2}}\ew \Mx{A}^{\ewpow{2}}}$. Then, $\Mx{B}{k}$ is calculated recursively, in terms of $\Mx{K}{k}$, using \cref{eq:bell-recursion}, with $\Mx{B}{0}=\Mx{1}{m\times m}$ and
\begin{displaymath}
	\Mx{B}{k} = \sum_{\ell=0}^{k-1} \binom{k-1}{\ell} \Mx{B}{\ell}\ew \Mx{K}{k-\ell}.
\end{displaymath}
Finally, the function value is calculated using $\fGMM_1(\theta) = \Vc{\lambda}\Mx{B}{d}\Vc{\lambda}'$. Regarding gradients, let $\Mx{W}_{\Mx{A}} = \nabla_{\Mx{A}}\fGMM_1(\theta)$ and define $\Mx{T}{\Mx{A}}^{(k)}$ and $\Mx{U}{\Mx{A}}^{(k)}$ by
\begin{displaymath}
	\Mx{T}{\Mx{A}}^{(k)} = \begin{cases}
		k!\; \Mx{D}^{\ewpow{k}}
		& \text{if $k$ even}, \\
		k!\; \Mx{D}^{\ewpow{k-1}}\ew \Mx{A}
		& \text{if $k$ odd},
	\end{cases}\quad \text{and} \quad
	\Mx{U}{\Mx{A}}^{(k)} = \begin{cases}
		\Mx{D}^{\ewpow{k-2}}\ew \Mx{A}
		& \text{if $k$ even}, \\
		\Mx{D}^{\ewpow{k-1}}
		& \text{if $k$ odd}.
	\end{cases}
\end{displaymath}
Note that for all $i,j,k$, we have $\nabla_{\m{j}}\kappa_k^{(ij)} = \Mx{T}{\Mx{A}}^{(k)}(:, i) \ew \Mx{U}{\Mx{A}}^{(k)}(:, j)$. We then have
\begin{align*}
  \Mx{W}{\Mx{A}}(:, j)
  &= \nabla_{\m{j}}\fGMM_1(\theta)
  = 2 \sum_{k=1}^d \binom{d}{k} \sum_{i=1}^m \lambda_i \lambda_j
  B_{d-k}\prn**{\kappa^{(ij)}_1, \dots, \kappa^{(ij)}_{d-k}} \nabla_{\m{j}}\kappa^{(ij)}_k,
  \\
  &= 2 \sum_{k=1}^d \binom{d}{k} \sum_{i=1}^m
  \tB^{(k)}(i,j) \prn*{ \Mx{T}{\Mx{A}}^{(k)}(:, i) \ew \Mx{U}{\Mx{A}}^{(k)}(:, j)},
  \\
  &= 2 \sum_{k=1}^d \binom{d}{k}
  \prn*{ \Mx{T}{\Mx{A}}^{(k)} \tB^{(k)} \ew \Mx{U}{\Mx{A}}^{(k)}}(:, j),
\end{align*}
where $\tB^{(k)} = \Mx{B}{d-k}\ew \Vc{\lambda}'\Vc{\lambda}$. This formula implies that
\begin{displaymath}
	\Mx{W}_{\Mx{A}}=2 \sum_{k=1}^d \binom{d}{k}  \Mx{T}{\Mx{A}}^{(k)} \tB^{(k)}\ew \Mx{U}{\Mx{A}}^{(k)}.
\end{displaymath}
In a similar fashion, we may define $\prn*{\Mx{T}{\Mx{D} ,r}^{(k)}}_{r=1,2}$ and $\prn*{\Mx{U}{\Mx{D} ,r}^{(k)}}_{r=1,2}$ such that
$$\nabla_{\Vc{d}{j}}\kappa_k^{(ij)} = \Mx{T}{\Mx{D} ,1}^{(k)}(:, i) \ew \Mx{U}{\Mx{D} ,1}^{(k)}(:, j) + \Mx{T}{\Mx{D} ,2}^{(k)}(:, i) \ew \Mx{U}{\Mx{D} ,2}^{(k)}(:, j),$$
which implies
\begin{displaymath}
	\Mx{W}_{\Mx{D}}:= \nabla_{\Mx{D}}\fGMM_1(\theta) = 2 \sum_{k=1}^d \binom{d}{k}  \Mx{T}{\Mx{D} ,1}^{(k)} \tB^{(k)}\ew \Mx{U}{\Mx{D} ,1}^{(k)} + \binom{d}{k}\Mx{T}{\Mx{D} ,2}^{(k)} \tB^{(k)}\ew \Mx{U}{\Mx{D} ,2}^{(k)}.
\end{displaymath}
Finally $\Mx{W}_{\Vc{\lambda}} = 2\Vc{\lambda}\Mx{B}{d}$. We summarize the whole procedure to calculate $\fGMM_1(\theta)$ and its gradients in \cref{alg:fGMM_1}, including the formulas for $\prn*{\Mx{T}{\Mx{D} ,r}^{(k)}}_{r=1,2}$ and $\prn*{\Mx{U}{\Mx{D} ,r}^{(k)}}_{r=1,2}$.
\begin{myalgo}{alg:fGMM_1}%
  {Calculate $\fGMM_1(\theta)$ from \cref{eq:opt-problem-detailed} with diagonal covariances, that is, $\C_j\nobreak=\nobreak\diag(\Vc{d}{j})^2$ for all $j\in[m]$.
    }
  \Require $\theta=\set{\Vc{\lambda}, \Mx{A}, \Mx{D}},$
  $\Vc{\lambda}=\sqrenum!{\lambda}{m},$ 
  $\Mx{A}=\sqrenum!{\m}{m},$ 
  $\Mx{D}=\sqrenum!{\Vc{d}}{m}$ \smallskip
	\Ensure $f=\fGMM_1(\theta)$, 
	$\Mx{W}_{\Mx{A}} = \nabla_{\Mx{A}}\fGMM_1(\theta)$, 
	$\Mx{W}_{\Mx{D}} = \nabla_{\Mx{D}}\fGMM_1(\theta)$, 
	$\Mx{W}_{\Vc{\lambda}} = \nabla_{\Vc{\lambda}}\fGMM_1(\theta)$ \smallskip
	\State $\Mx{B}{0} \gets \Mx{1}{m\times m}$
	\For{$k = 1,\dots,d$}
	\If{$k$ is odd}
	\State $\Mx{K}{k} \gets k! \prn*{\Mx{D}^{\ewpow{ k-1}}\ew \Mx{A}}^{\Tr} \prn*{\Mx{D}^{\ewpow{k-1}}\ew \Mx{A}}$
	\Else
	\State $\Mx{V} \gets \prn*{\Mx{D}^{\ewpow{k}}}^{\Tr} \prn*{\Mx{D}^{\ewpow{k-2}}\ew \Mx{A}^{\ewpow{2}}}$
	\State $\Mx{K}{k} \gets (k{-}1)! \Prn{\Mx{D}^{\ewpow{k}}}^{\Tr}  \Mx{D}^{\ewpow{k}} + 
	\frac{k!}{2} \prn*{\Mx{V} + \Mx{V}'}$
	\EndIf
	\State $\displaystyle \Mx{B}{k} \gets \sum_{r=0}^{k-1} \binom{k-1}{r} \; \Mx{B}{r}\ew \Mx{K}{k-r}$ 
	\EndFor
	\State	$f \gets \Vc{\lambda}\Mx{B}{d}\Vc{\lambda}'$
	\State	$\Mx{W}_{\Vc{\lambda}} \gets 2\; \Vc{\lambda}\Mx{B}{d}$
	\State $\Mx{W}_{\Mx{A}}\gets \Mx{0}$,~~~~$\Mx{W}_{\Mx{D}}\gets \Mx{0}$,
	\For{$k = 1,\dots,d$}
	\State $\tB \gets \Mx{B}{d-k}\ew \Vc{\lambda}'\Vc{\lambda}$
	\If{$k$ is odd} \smallskip
	\State $\Mx{T}{\Mx{D}} \gets k! \; \Mx{D}^{\ewpow{k-1}}\ew \Mx{A}$
	\State $\Mx{U}{\Mx{D}} \gets \bm{1}(k>1)(k-1) \Mx{D}^{\ewpow{k-2}}\ew \Mx{A}$
	\State $\Mx{W}_{\Mx{D}} \gets \Mx{W}_{\Mx{D}}+ 2\binom{d}{k}\Mx{T}{\Mx{D}} \tB\ew \Mx{U}{\Mx{D}}$\medskip
	\State $\Mx{T}{\Mx{A}} \gets k!\; \Mx{D}^{\ewpow{k-1}}\ew \Mx{A}$
	\State $\Mx{U}{\Mx{A}} \gets \Mx{D}^{\ewpow{k-1}}$
	\State $\Mx{W}_{\Mx{A}} \gets \Mx{W}_{\Mx{A}} + 2\binom{d}{k} \Mx{T}{\Mx{A}} \tB\ew \Mx{U}{\Mx{A}}$ 
	\Else \smallskip
	\State $\Mx{T}{\Mx{D},1} \gets k!\; \Mx{D}^{\ewpow{k}} + k\frac{k!}2  \Mx{D}^{\ewpow{k-2}}\ew \Mx{A}^{\ewpow{2}}$
	\State $\Mx{U}{\Mx{D},1} \gets \Mx{D}^{\ewpow{k-1}} $
	\State $\Mx{T}{\Mx{D},2} \gets \frac{k!}{2}\; \Mx{D}^{\ewpow{k}}$
	\State $\Mx{U}{\Mx{D},2} \gets \bm{1}(k>2)(k-2) \Mx{D}^{\ewpow{k-3}}\ew \Mx{A}^{\ewpow{2}}$
	\State $\Mx{W}_{\Mx{D}} \gets \Mx{W}_{\Mx{D}}+ 2\binom{d}{k}\prn*{\Mx{T}{\Mx{D},1} \tB\ew \Mx{U}{\Mx{D},1} + \Mx{T}{\Mx{D},2} \tB\ew \Mx{U}{\Mx{D},2}}$\medskip
	\State $\Mx{T}{\Mx{A}} \gets k!\; \Mx{D}^{\ewpow{k}}$
	\State $\Mx{U}{\Mx{A}} \gets \Mx{D}^{\ewpow{k-2}} \ew \Mx{A}$
	\State $\Mx{W}_{\Mx{A}} \gets \Mx{W}_{\Mx{A}} + 2\binom{d}{k} \Mx{T}{\Mx{A}} \tB\ew \Mx{U}{\Mx{A}}$ 
	\EndIf
	
	\EndFor
\end{myalgo}

\subsubsection[Calculating moment dot-product]{Calculating $\fGMM_2(\theta)$}

Recalling \cref{thm:gmmdotp}, using that $\C_{j} = \diag(\Vc{d}{j})^2$, and thus that $\x{i}'\C{j}\x{i} = \prn*{\x{i}^{\ewpow{2}}}^{\Tr} \Vc{d}{j}^{\ewpow{2}}$, we obtain
\begin{align*}
	\fGMM_2(\theta) &= \sum_{i=1}^p\sum_{j=1}^m \lambda_j \; \alpha_{ij}^{(d)} & 
	\nabla_{\m{j}} \fGMM_2(\theta)
	&= d \lambda_j \sum_{i=1}^p \alpha_{ij}^{(d-1)} \x{i}, \\
	\nabla_{\lambda{j}} \fGMM_2(\theta)  &= \sum_{i=1}^p \; \alpha_{ij}^{(d)} &
	\nabla_{\C{j}} \fGMM_2(\theta)
	&= d(d-1) \lambda_j \sum_{i=1}^p \alpha_{ij}^{(d-2)} \Vc{d}{j}\ew \x{i}^{\ewpow{2}}.	
\end{align*}
where $\alpha_{ij}^{(d)}$ is calculated using \eqref{eq:alpha_recursion_formula}. Let $\Mx{R}{1}, \Mx{R}{2}\in \Real^{p \times m}$, defined entrywise by 
\begin{displaymath}
	\Mx{R}{1}(i, j) = \alpha_{ij}^{(d-1)} \qtext{and} \Mx{R}{2}(i, j) = \alpha_{ij}^{(d-2)}m
\end{displaymath}
and
\begin{gather*}
	\Mx{X}  = \begin{bmatrix} \x{1} & \cdots & \x{p} \end{bmatrix} \in \R^{n \times p},
	\quad
	\Mx{T}=\Mx{X} \Mx{R}{1},\\
	\Mx{U}=(d-1)(\Mx{X}^{\ewpow{2}}\Mx{R}{2})\ew \Mx{D}  \qtext{and}
	 \Vc{z} = \Prn{\Mx{T} \ew \Mx{A} + \Mx{U} \ew \Mx{D}} \Vc{1}.
\end{gather*}
Then we have
\begin{align*}
	\fGMM_2(\theta) &= \Vc{\lambda} \Vc{z}& 
	\nabla_{\Mx{A}} \fGMM_2(\theta)
	&= d \; \Mx{T} \diag(\Vc{\lambda}), \\
	\nabla_{\Vc{\lambda}} \fGMM_2(\theta)  &= \Vc{z}' &
	\nabla_{\Mx{D}} \fGMM_2(\theta)
	&= d \; \Mx{U} \diag(\Vc{\lambda}).	
\end{align*}
The algorithm to calculate these quantities, including the recursion to calculate $\Mx{R}{1}$ and $\Mx{R}{2}$, is summarized in \cref{alg:fGMM_2}.

\stepcounter{footnote}
\begin{myalgo}{alg:fGMM_2}%
  {Calculate $\fGMM_2(\theta)$ from \cref{eq:opt-problem-detailed}
    with diagonal covariances, that is,
    $\C_j\nobreak=\nobreak\diag(\Vc{d}{j})^2$ for all $j\in[m]$.%
}
  \Require 
  $\Mx{X}  = \sqrenum!{\x}{p}, 
  \theta=\set{\Vc{\lambda}, \Mx{A}, \Mx{D}}, 
  \Vc{\lambda}=\sqrenum!{\lambda}{m},
  \Mx{A}=\sqrenum!{\m}{m},
  \Mx{D}=\sqrenum!{\Vc{d}}{m}$
  \Ensure $f=\fGMM_2(\theta)$, 
  $\Mx{W}_{\Mx{A}} = \nabla_{\Mx{A}}\fGMM_2(\theta)$, 
  $\Mx{W}_{\Mx{D}} = \nabla_{\Mx{D}}\fGMM_2(\theta)$, 
  $\Mx{W}_{\Vc{\lambda}} = \nabla_{\Vc{\lambda}}\fGMM_2(\theta)$ \smallskip
  \State $\Mx{V} \gets \Mx{X}'\Mx{A}$
  \State $\Mx{Z} \gets (\Mx{X}')^{\ewpow{2}} \Mx{D}^{\ewpow{2}}$
  \If{$d=1$}
  \State $\Mx{R}{2} \gets \Mx{0}{p\times m}$ (all zeros $p \times m$ matrix)
  \State $\Mx{R}{1} \gets \Mx{1}{p\times m}$ (all ones $p \times m$ matrix)
  \Else
  \State $\Mx{R}{2} \gets \Mx{1}{p\times m}$ 
  \State $\Mx{R}{1} \gets \Mx{V}$
  \EndIf
  \For{$k = 2,\dots,d-1$}
  \State $\Mx{R}{3} \gets \Mx{R}{2}$
  \State $\Mx{R}{2} \gets \Mx{R}{1}$
  \State $\Mx{R}{1} \gets \Mx{R}{2} \ew \Mx{V} + (k-1) \Mx{R}{3} \ew \Mx{Z}$ 
  \EndFor
  \State $\Mx{T}\gets \Mx{X} \Mx{R}{1}$
  \State $\Mx{W}{\Mx{A}} \gets d\; \Mx{T} \diag(\Vc{\lambda})$
  \State $\Mx{U}\gets (d-1)(\Mx{X}^{\ewpow{2}}\Mx{R}{2})\ew \Mx{D}$
  \State $\Mx{W}{\Mx{D}} \gets d\; \Mx{U} \diag(\Vc{\lambda})$
  \State $\Mx{W}{\Vc{\lambda}} \gets \Prn{\Mx{T} \ew \Mx{A} + \Mx{U} \ew \Mx{D}} \Vc{1}$
  \State $f\gets \Mx{W}{\Vc{\lambda}}\Vc{\lambda}'$
\end{myalgo}

\subsection{Implementation of debiasing to GMMs}
\label{sec:impl-debi-gmm}

Consider the optimization problem \eqref{eq:opt-problem-debias}, 
that arises as an application of \cref{thm:GMMdebiasedmoments,prop:debiaseddotp}
to GMMs with known covariance $\C$. In this section, we explain how to calculate $\fdeb(\theta)$ and its derivatives, which is to be used by a first-order optimization method to solve \eqref{eq:opt-problem-debias}. Recalling that $\fdeb(\theta) = \fdeb_1(\theta) - 2 \fdeb_2(\theta)$, we show how each term is calculated in the following sections.

\subsubsection{Computing $\fdeb_1$ and its derivatives}
The derivatives of $\fdeb_1$ are:
\begin{align*}
	\FD{\fdeb_1}{\lambda_j} & = 2\sum_{i=1}^m \lambda_i \ang{\m{i},\m{j}}^d, &
	\FD{\fdeb_1}{\m{j}} & = 2d \sum_{i=1}^m \lambda_i \lambda_j \ang{\m{i},\m{j}}^{d-1} \m{i}.
\end{align*}
Define the vector and matrix quantities,

\begin{gather*}
	\Vc{\lambda}
	= \sqrenum!{\lambda}{m}  \in \R^{1\times m},
	\quad
	\Mx{A}
	= \sqrenum!{\m}{m} \in \R^{n \times m},
	\\
	\Mx{B} = \Mx{A}'\Mx{A} \in \R^{m \times m},
	\quad
	\Mx{C} = \Mx{B}^{\ewpow{d-1}} \in \R^{m \times m}, %
	\qtext{and}
	\Vc{u} = (\Mx{B} \had \Mx{C}) \Vc{\lambda} \in \Real^m, %
\end{gather*}
where $\Mx{B}^{\ewpow{d-1}}$ indicates to raise each element of $\Mx{B}$ to the $(d-1)$ power,
and $\Mx{B} \had \Mx{C}$ indicates Hadamard (elementwise) multiplication.
Then we can express $\fdeb_1$ and its gradients as follows,
with $\Mx{D}{\Vc{\lambda}} = \diag(\Vc{\lambda})$:
\begin{displaymath}
	\fdeb_1(\theta) = \Vc{\lambda}'\Vc{u} \in \R,
	\quad\!\!\!
	\FD{\fdeb_1}{\Vc{\lambda}}  = 2 \Vc{u} \in \R^m,\!\!\!
	\qtext{and}\!\!\! 
	\FD{\fdeb_1}{\Mx{A}} = 2d \, \Mx{A} \Mx{D}{\Vc{\lambda}} \Mx{C} \Mx{D}{\Vc{\lambda}} \in \R^{n \times m}.
\end{displaymath}

\subsubsection{Computing $\fdeb_2$ and its derivatives}
By \cref{prop:debiaseddotp}, the function $\fdeb_2$ and its gradients can be expressed as
\begin{align*}
	\fdeb_2(\theta) = \sum_{j=1}^m \lambda_j \m{j}'\Vc{z}{j},
	\quad
	\FD{\fdeb_2}{\lambda_j} = \m{j}' \Vc{z}{j},
	\qtext{and}
	\FD{\fdeb_2}{\m{j}} = d \lambda_j \Vc{z}{j},
\end{align*}
where
\begin{displaymath}
	\Vc{z}{j} = \frac{1}{p}\sum_{i=1}^p  \beta_{ij}^{(d-1)} \x{i} - (d-1) \beta_{ij}^{(d-2)} \C \m{j}
	\qtext{for all} j \in [m],
\end{displaymath}
and the $\beta$-values are calculated recursively using \eqref{eq:beta_recursion_formula}. Furthermore, let
\begin{gather*}
	\Mx{X}  = \begin{bmatrix} \x{1} & \cdots & \x{p} \end{bmatrix} \in \R^{n \times p},
	\quad
	\Mx{T}=\Mx{X} \Mx{R}{1},\\
	\Mx{U}=(d-1)(\Mx{X}^{\ewpow{2}}\Mx{R}{2})\ew \Mx{D}  \qtext{and}
	\Vc{z} = \Prn{\Mx{T} \ew \Mx{A} + \Mx{U} \ew \Mx{D}} \Vc{1}
\end{gather*}

Assuming we have computed the vectors $\Vc{z}{j}$, 
using the definitions of $\Vc{\lambda}$ and $\Mx{A}$ from computing $\fdeb_1$,
and defining
\begin{gather*}
	\Mx{X}  = \begin{bmatrix} \x{1} & \cdots & \x{p} \end{bmatrix} \in \R^{n \times p},
	\quad
	\Mx{Y} = \C \Mx{A} \in \R^{n \times m},
	\quad
	\Mx{Z}  = \begin{bmatrix} \Vc{z}{1} & \cdots & \Vc{z}{m} \end{bmatrix} \in \R^{n \times m},
	\\
	\qtext{and}
	\Vc{w}  \in \R^m \text{ with } w_j = \m{j}'\Vc{z}{j}". %
\end{gather*}
we can compute $\fdeb_2$ and its derivatives as follows.
\begin{align*}
	\fdeb_2(\theta) &= \Vc{\lambda}'\Vc{w}, &
	\FD{\fdeb_2}{\Vc{\lambda}}(\theta) &= \Vc{w}, &
	\FD{\fdeb_2}{\Mx{A}} &= d\, \Mx{Z} \diag({\Vc{\lambda}}).
\end{align*}

To compute $\Mx{Z}$, we execute a recursive procedure outlined in \cref{alg:fdeb}, to obtain
\begin{displaymath}
	\Mx{R}_1 = \sqr{ \beta_{ij}^{(d-1)} } \in \R^{p \times m}
	\qtext{and}
	\Mx{R}_2 =  \sqr{ \beta_{ij}^{(d-2)} } \in \R^{p \times m}.
\end{displaymath}
With these, we can express $\Mx{Z}$ as
\begin{displaymath}
	\Mx{Z} = \frac1{p}\prn{\Mx{X} \Mx{R}{1} - (d-1) \diag\prn{\Vc{1}' \Mx{R}{2}} \Mx{Y}},
\end{displaymath}
where $\Vc{1} = \begin{bmatrix} 1 & \cdots & 1 \end{bmatrix} \in \Real^p$. We summarize the algorithm to calculate $\fdeb(\theta)$ and its derivatives, which includes calculating $\fdeb_1$ and $\fdeb_2$, in \cref{alg:fdeb}.

\begin{myalgo}{alg:fdeb}{Calculate $\fdeb(\theta)$ from \cref{eq:opt-problem-debias}.}
	\Require $\C$, $\Mx{X}  = \sqrenum!{\x}{p}$,
	$\theta=\set{\Vc{\lambda}, \Mx{A}}$, 
	$\Vc{\lambda}=\sqrenum!{\lambda}{m}$,
	$\Mx{A}=\sqrenum!{\m}{m}$
	\Ensure $f=\fdeb(\theta)$, 
	$\Mx{W}_{\Mx{A}} = \nabla_{\Mx{A}}\fdeb(\theta)$, 
	$\Mx{W}_{\Vc{\lambda}} = \nabla_{\Vc{\lambda}}\fdeb(\theta)$ \smallskip
	\State $\Mx{B} \gets \Mx{A}'\Mx{A}$
	\State $\Mx{C} \gets \Mx{B}^{\ewpow{d-1}}$
	\State $\Vc{u} \gets (\Mx{B} \had \Mx{C}) \Vc{\lambda}$
	\State $\Mx{D}{\Vc{\lambda}} \gets \diag(\Vc{\lambda})$
	\State $f_1 \gets \Vc{\lambda}'\Vc{u} $
	\State $\Mx{W}_{\Vc{\lambda},1}\gets 2\Vc{u}$
	\State $\Mx{W}_{\Mx{A},1}\gets 2d \, \Mx{A} \Mx{D}{\Vc{\lambda}} \Mx{C} \Mx{D}{\Vc{\lambda}}$
	\State $\Mx{V} \gets \Mx{X}'\Mx{A}$
	\State $\Mx{Y} \gets \C \Mx{A}$
	\State $\Mx{U} \gets \diag((\Mx{Y} \ew \Mx{A})\Vc{1})$
	\If{$d=1$}
	\State $\Mx{R}{2} \gets \Mx{0}{p\times m}$ (all zeros $p \times m$ matrix)
	\State $\Mx{R}{1} \gets \Mx{1}{p\times m}$ (all ones $p \times m$ matrix)
	\Else
	\State $\Mx{R}{2} \gets \Mx{1}{p\times m}$ 
	\State $\Mx{R}{1} \gets \Mx{V}$
	\EndIf
	\For{$k = 2,\dots,d-1$}
	\State $\Mx{R}{3} \gets \Mx{R}{2}$
	\State $\Mx{R}{2} \gets \Mx{R}{1}$
	\State $\Mx{R}{1} \gets \Mx{R}{2} \ew \Mx{V} - (k-1) \Mx{R}{3} \Mx{U}$ 
	\EndFor
	\State $\Mx{Z} \gets \frac1{p}\prn{\Mx{X} \Mx{R}{1} - (d-1) \diag\prn{\Vc{1}' \Mx{R}{2}} \Mx{Y}}$
	\State $\Vc{w} \gets (\Mx{Z} \ew \Mx{A})\Vc{1}$
	\State $f_2 \gets \Vc{\lambda}'\Vc{w}$
	\State $\Mx{W}_{\Vc{\lambda},2}\gets \Vc{w}$
	\State $\Mx{W}_{\Mx{A},2}\gets d\, \Mx{Z} \Mx{D}{\Vc{\lambda}}$
	\State $f \gets f_1 - 2f_2$
	\State $\Mx{W}_{\Vc{\lambda}}\gets \Mx{W}_{\Vc{\lambda},1} - 2\Mx{W}_{\Vc{\lambda},2}$
	\State $\Mx{W}_{\Mx{A}}\gets \Mx{W}_{\Mx{A},1} - 2\Mx{W}_{\Mx{A},2}$
\end{myalgo}

\subsection{Augmented system}

For solving the augmented system, we implement two approaches, which we denote by \textsc{Implicit} and \textsc{Post-Processing}. These approaches can be thought of as reparametrizations of the same optimization problem, however we observed empirically that these perform differently. We explain these approaches for $\fGMM(\bar \theta)$; the implementation for $\fdeb(\bar \theta)$ is analogous.

\subsubsection[Implicit]{\textsc{Implicit}}

The idea of the \textsc{Implicit} approach is to solve \eqref{eq:opt-problem-aug} without explicitly forming the augmented variables $(\maug{j}, \Caug{j})_{j=1}^m$ and augmented samples $(\x[\bar]{i})_{i=1}^p$. Comparing $\fGMM_1(\bar\theta) \equiv \nrm*{\M[d][\bar]}^2$ with $\fGMM_1(\theta) \equiv \nrm*{\M}^2$ in \cref{thm:gmm_tensor_norm}, we note that the corresponding cumulants are related by 
\begin{displaymath}
\bar\kappa^{(\ell)}_{ij} = \begin{cases}
\kappa^{(1)}_{ij} + \omega^2& \text{ if }\ell=1,\medskip\\
\kappa^{(\ell)}_{ij}& \text{ otherwise}.
\end{cases}
\end{displaymath}
In a similar fashion, we obtain that
\begin{equation*}%
	\fGMM_2(\bar\theta) = \frac1{p}\sum_{i=1}^p\sum_{j=1}^m \lambda_j \; \bar\alpha_{ij}^{(d)},
\end{equation*}
where
$\bar\alpha_{ij}^{(d)}$ is calculated implicitly using the recursion formula
\begin{equation*}
	\bar\alpha_{ij}^{(\ell)} =
	\bar\alpha_{ij}^{(\ell-1)} (\x{i}' \m{j}" + \omega^2) +(\ell-1) \bar\alpha_{ij}^{(\ell-2)} \x{i}' \C{j}" \x{i}"
\end{equation*}
with the convention that $\bar\alpha_{ij}^{(0)} = 1$ and $\bar\alpha_{ij}^{(1)} = \x{i}' \m{j}" + \omega^2$. Finally, we note that $\omega$ is a hyper-parameter that stays constant throughout the optimization, and that the gradients can be easily adapted to the implicit augmented system.
In this formulation, we must explicitly enforce the constraint $\sum_{j=1}^m \lambda_j=1$.

\subsubsection[Post-Processing]{\textsc{Post-Processing}}

For this approach, we use the augmented variables \linebreak $(\maug{j}, \Caug{j})_{j=1}^m$ and augmented samples $(\x[\bar]{i})_{i=1}^p$, however we do not enforce the constraint $\Vc[\bar]{\mu}{j}(n{+}1) = \omega$ or $\sum_{j=1}^m \lambda_j=1$ throughout the optimization.
These are instead enforced later by post-processing the obtained solution. Without these constraints, the solutions of the optimization problem have again a scaling ambiguity, which we fix by setting $\Vc{\lambda}$ to be constant throughout the optimization. After we obtain a solution to this optimization problem, $\tilde\theta = \set{\prn{\tilde\lambda_j, \Vc[\tilde]{\mu}{j}, \Mx[\tilde]{\Sigma}{j} } }_{j=1}^m$, we obtain the desired solution by rescaling $\tilde\theta$ as follows.
\begin{displaymath}
\tilde\gamma_j = \omega/\Vc[\tilde]{\mu}{j}(n{+}1),\quad \m{j}= \tilde\gamma_j\Vc[\tilde]{\mu}{j}(1{:}n),
\quad \C{j}= \tilde\gamma_j^2\Mx[\tilde]{\Sigma}{j}(1{:}n, 1{:}n)
\qtext{and} \lambda_j = \tilde\gamma_j^{{-}d} \tilde\lambda_j.
\end{displaymath}

Although both augmented system approaches are reparametrizations of the same optimization problem, we observe empirically that \textsc{Post-processing} consistently outperforms \textsc{Implicit}.

\section{Computational Experiments}
\label{sec:comp-exper}

We demonstrate the potential of the proposed approaches in several examples.
We perform our experiments in MATLAB, using
MATLAB's implementation of EM (\texttt{fitgmdist}) from
the Statistics and Machine Learning Toolbox.
We also use the Tensor Toolbox for MATLAB \citep{TensorToolbox,BaKo06}.
The experiments were run on a Windows
laptop with an Intel Core i7-10510U CPU (2.3GHz)
and 16~GB of memory.

\subsection{Demonstrating accuracy of GMM moments}
\label{sec:gmm3_moment_2d_validate}

In \cref{fig:gmm3_moment_2d_validate}, we demonstrate
the accuracy of the GMM moment expression in \cref{prop:GMMbiasedmoments}
for a small example with $n=2$, $m=3$.
\begin{figure}
  \centering
  \subfloat[Probability distribution function (pdf)]{%
    \includegraphics[width=0.45\textwidth]{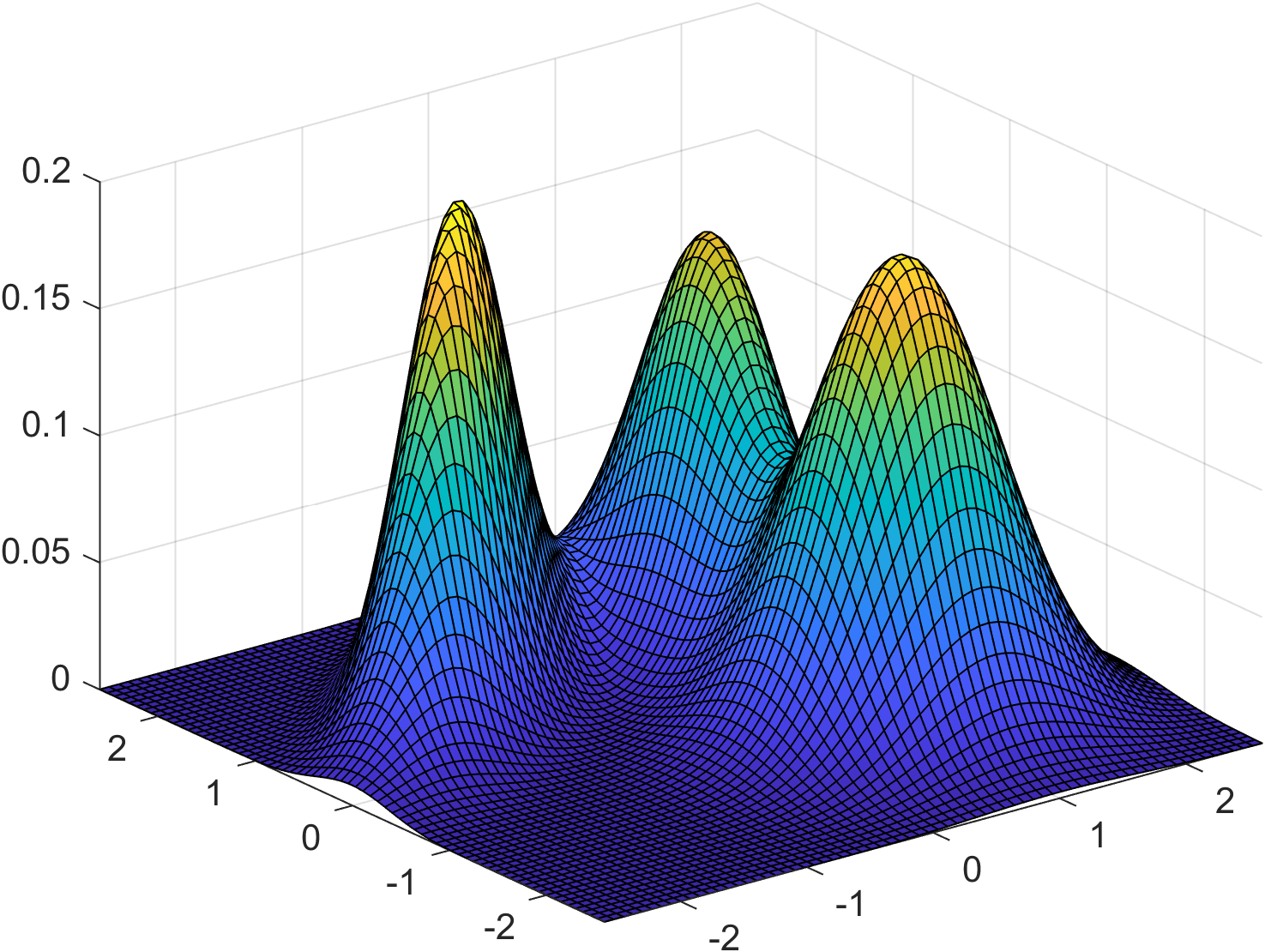}%
  }
  ~~
  \subfloat[Contour lines of pdf]{%
    \includegraphics[width=0.45\textwidth]{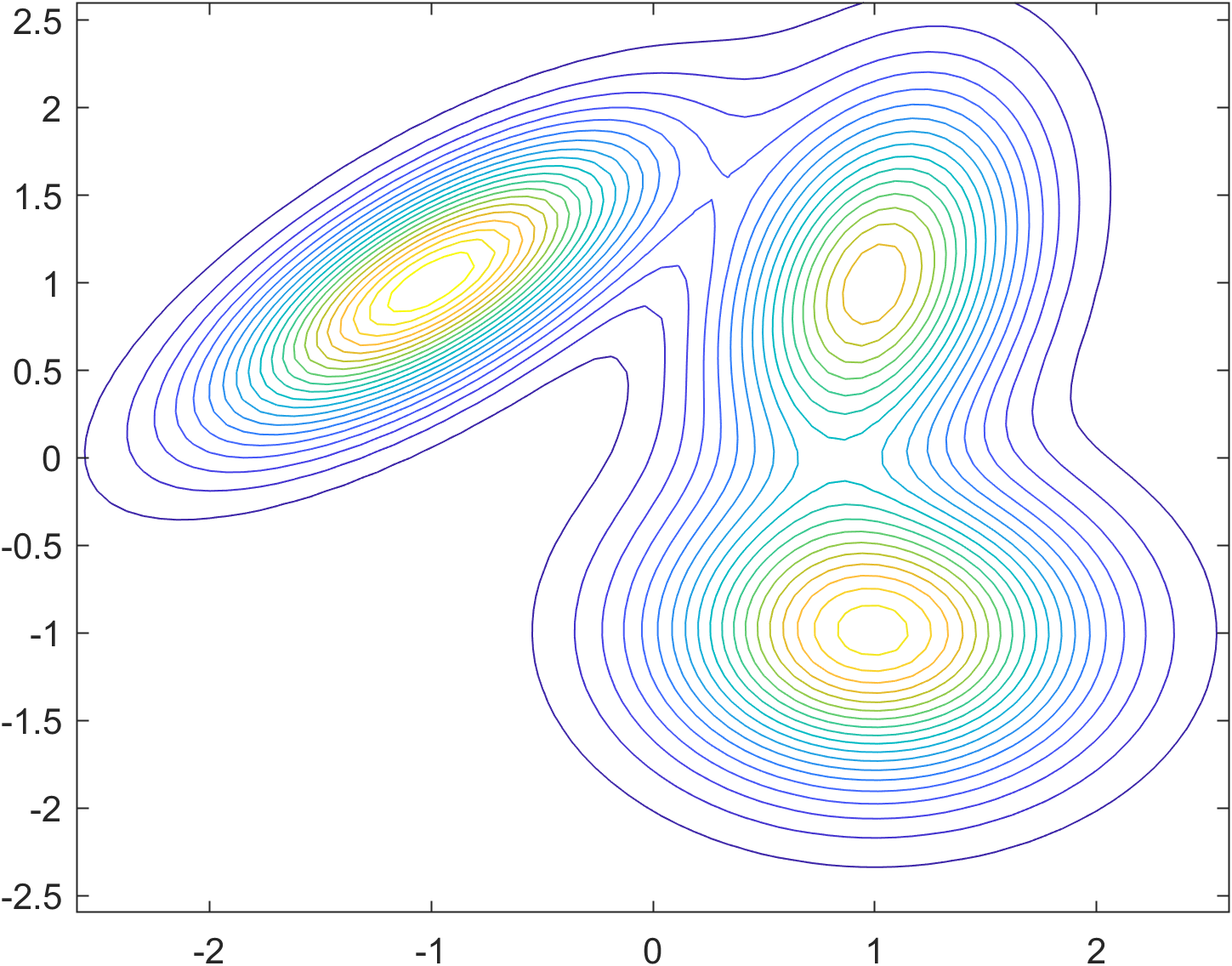}%
  }\\
  \subfloat[Error in theoretical versus empirical moments, as number of samples from the distribution ($p$) increases.]{%
    \begin{tikzpicture}[
      glabel/.style={node font=\footnotesize,align=flush right}
      ]
      \node (image)
      {\includegraphics[scale=0.75]{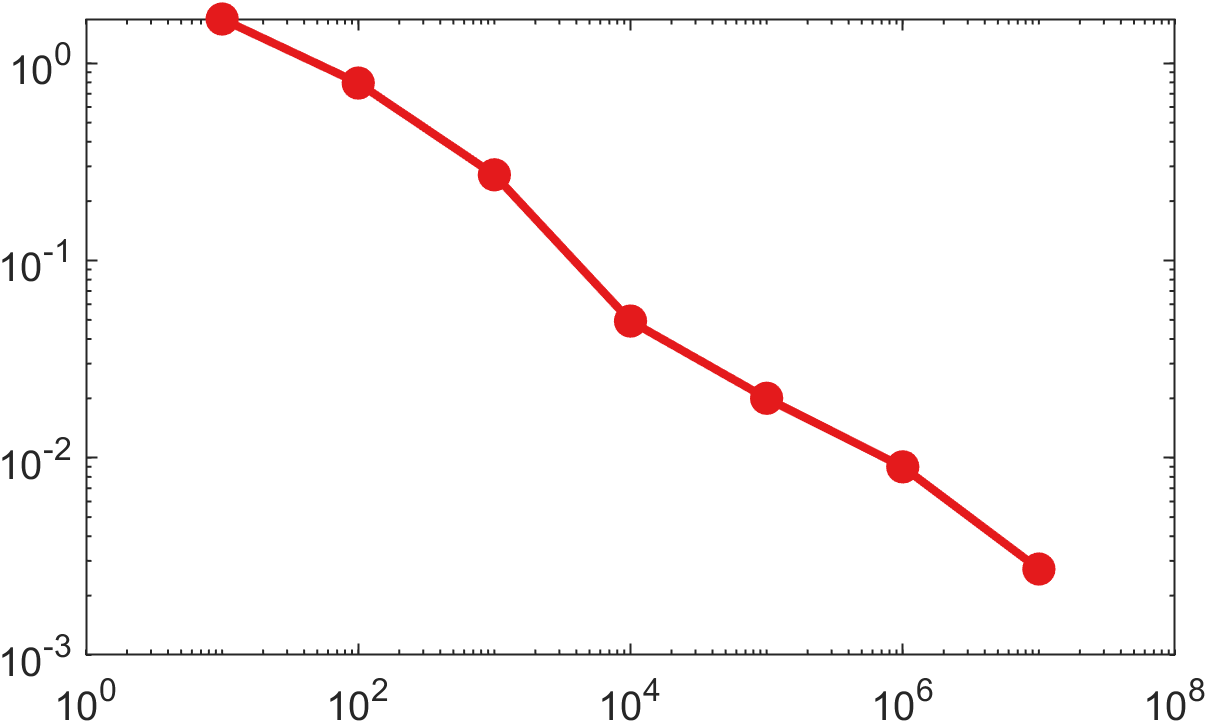}};
      \node[glabel,anchor=north] at (image.south) {number of realizations, $p$};
      \node[glabel,anchor=north east,text width=1.25in] at ($(image.north east) -0.4*(1,1)$)
      {Model versus estimated moments\\[1mm]
        $\nrm**{\M[3] - \Mest[3] }$};
    \end{tikzpicture}%
  }
  \caption{Error in GMM moments for
    example GMM with $n=2$, $m=3$.}
  \label{fig:gmm3_moment_2d_validate}
\end{figure}
Here, we take
\begin{displaymath}
  X = \sum_{j=1}^3 \lambda_j \mathcal{N}(\m{j},\C{j})
  \qtext{with}
\begin{aligned}
  \lambda_1 &= 0.4,
  &
  \m_1&=\begin{bmatrix} 1 \\ -1 \end{bmatrix},
  &
  \C_1&=\begin{bmatrix} 0.4 & 0 \\  0 & 0.3 \end{bmatrix},
  \\
  \lambda_2 &= 0.3,
  &  
  \m_2&=\begin{bmatrix} 1 \\ 1 \end{bmatrix},
  &
  \C_2&=\begin{bmatrix} 0.2 & 0.1 \\  0.1 & 0.5 \end{bmatrix},
  \\
  \lambda_3 &= 0.3,
  &
  \m_3&=\begin{bmatrix} -1 \\ 1 \end{bmatrix},
  &
  \C_3&=\begin{bmatrix} 0.4 & 0.25\\  0.25& 0.3 \end{bmatrix}  .
\end{aligned}
\end{displaymath}
For increasing values of $p$, we draw $p$ random samples
$\set{\miwc[\Vc{x}][p]}$
from the distribution and compute
the norm of the difference between the model moment,
\begin{displaymath}
  \M[3] = \sum_{j=1}^3 \sop[3]{\m_j} + 2 \sym\prn{ \m_j \otimes \C_j},
\end{displaymath}
and the empirical moment,
\begin{displaymath}
  \Mest[3] = \frac{1}{p} \sum_{i=1}^p \sop[3]{\Vc{x}{i}}.
\end{displaymath}
We observe that the empirical moment converges to the model moment at a rate of $1/\sqrt{p}$, which is expected.

\subsection{Demonstrating accuracy of debiased moments}
\label{sec:twodim-debias}

In \cref{fig:twodim_debias_validate}, we demonstrate
the debiased moment expression in \cref{thm:GMMdebiasedmoments}.
\begin{figure}
  \centering
  \subfloat[Probability distribution function (pdf)]{%
    \includegraphics[width=0.45\textwidth]{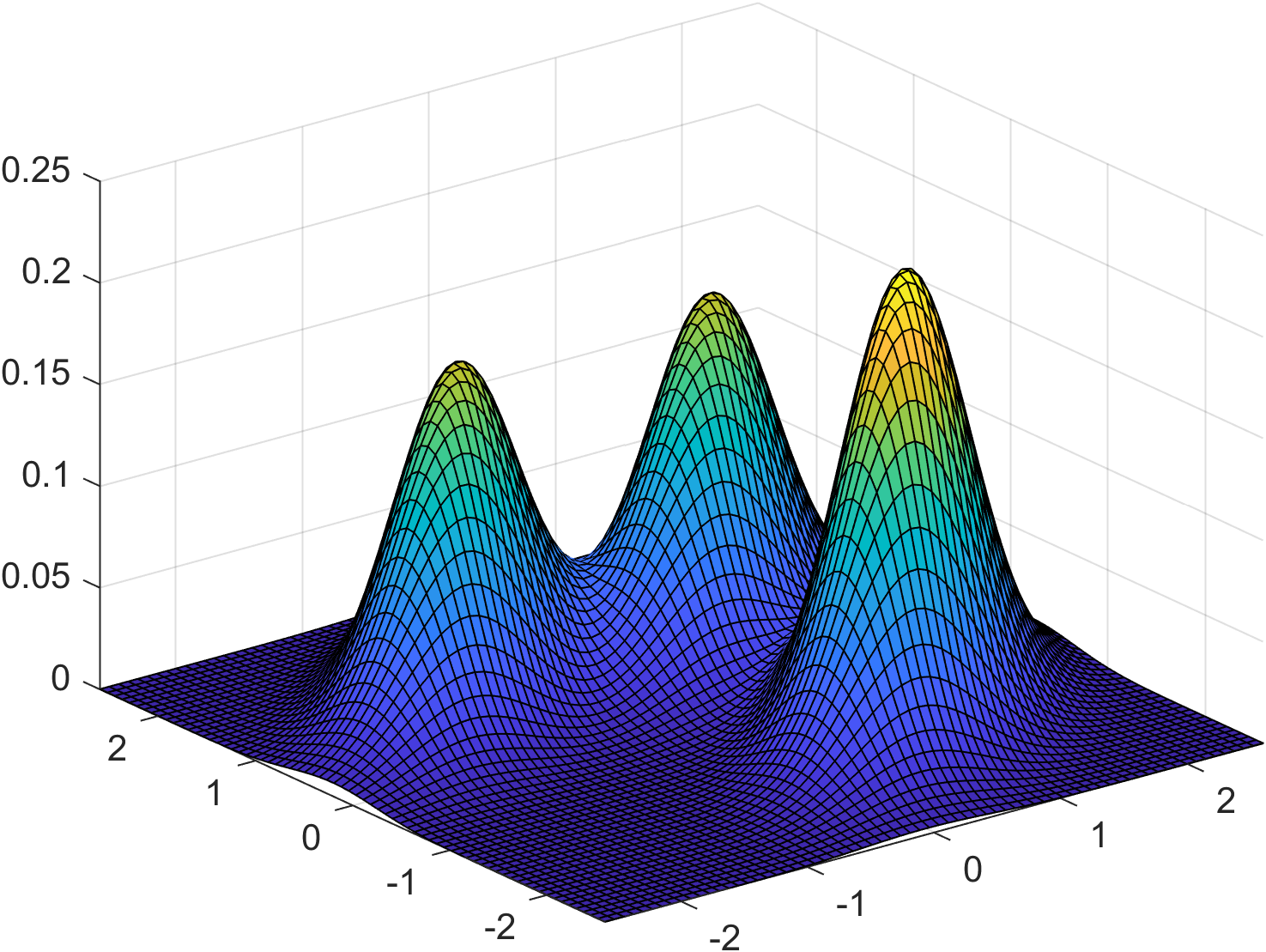}%
  }
  \;\;
  \subfloat[Contour lines of pdf]{%
    \includegraphics[width=0.45\textwidth]{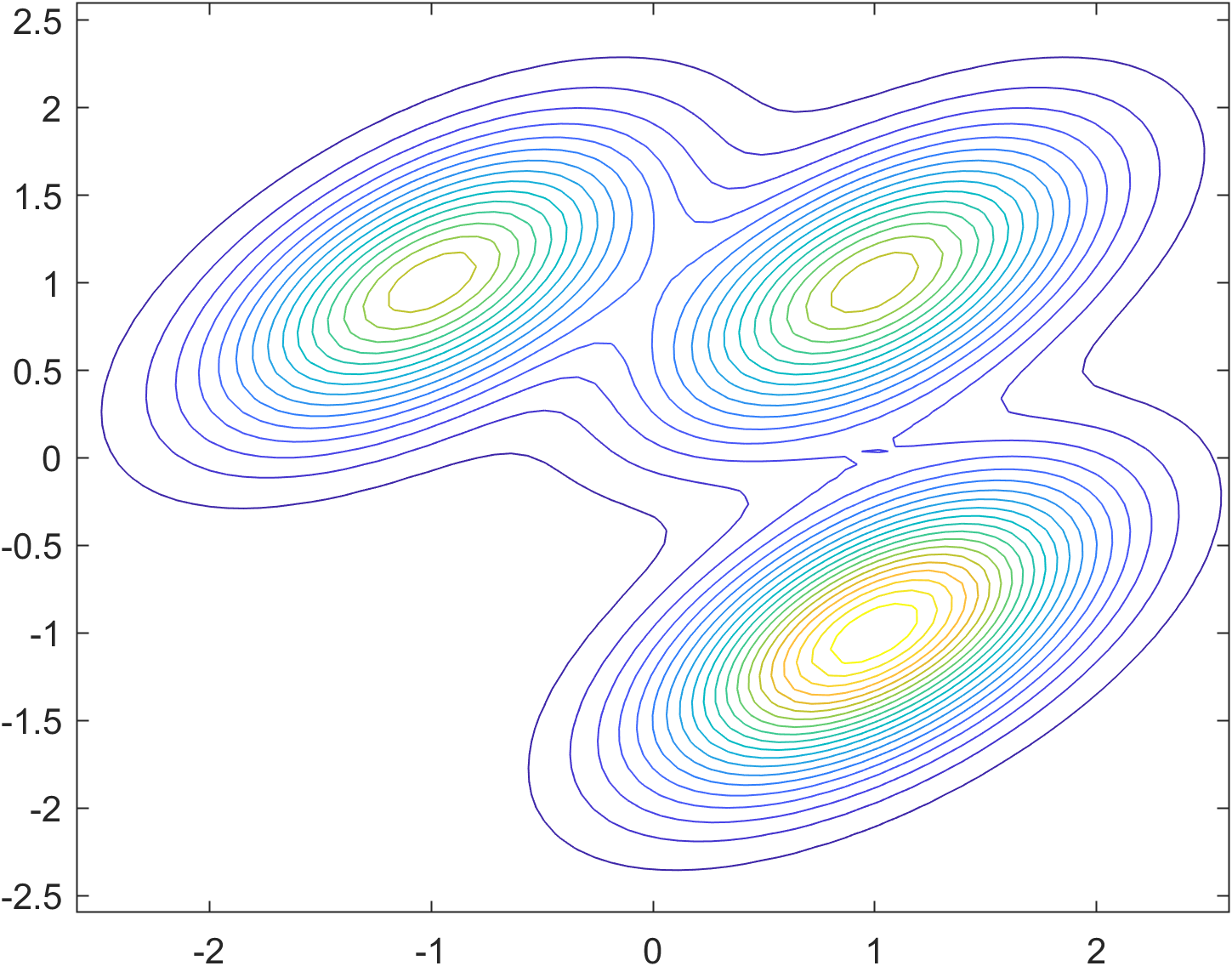}%
  }\\
  \subfloat[Error in theoretical versus debiased empirical moment, as number of realizations ($p$) increases.]{%
    \begin{tikzpicture}[
      glabel/.style={node font=\footnotesize,align=flush right}
      ]
      \node (image)
      {\includegraphics[scale=0.75]{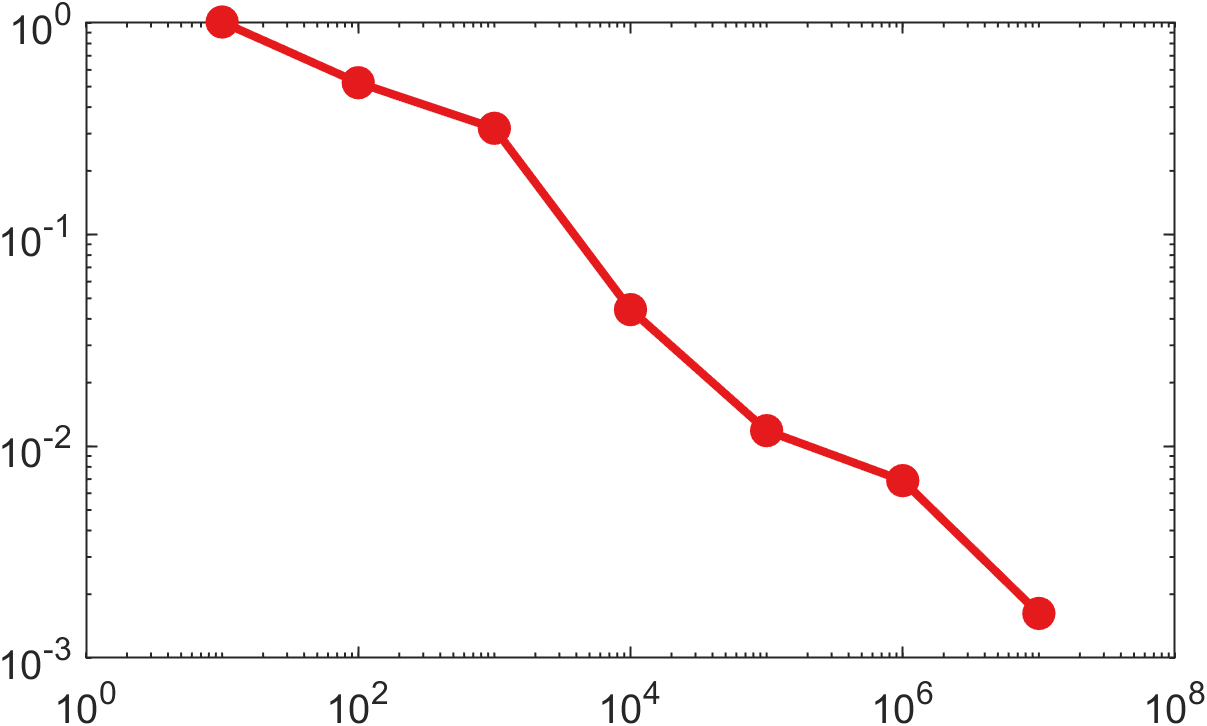}};
      \node[glabel,anchor=north] at (image.south) {number of realizations, $p$};
      \node[glabel,anchor=north east,text width=1.5in] at ($(image.north east) -0.4*(1,1)$)
      {Theoretical versus estimated debiased moments using known covariance\\[1mm]
        $\nrm**{\T[3] - \Test[3] }$};
    \end{tikzpicture}%
  }
  \caption{Error in debiased moments for
    example GMM with $n=2$, $m=3$, and a common covariance matrix.}
  \label{fig:twodim_debias_validate}
\end{figure}
For this example, we have $n=2$ and $m=3$ and a fixed and \emph{known} covariance:
\begin{displaymath}
  X = \sum_{j=1}^3 \lambda_j \mathcal{N}(\m{j},\C)
  \qtext{with}
  \C =
  \begin{bmatrix}
    0.4 & 0.2 \\
    0.2 & 0.3
  \end{bmatrix}
\end{displaymath}
and
\begin{align*}
  \lambda_1 &= 0.4,
  &
  \m_1&=\begin{bmatrix} 1 \\ -1 \end{bmatrix},
  &
  \lambda_2 &= 0.3,
  &  
  \m_2&=\begin{bmatrix} 1 \\ 1 \end{bmatrix},
  &
  \lambda_3 &= 0.3,
  &
  \m_3&=\begin{bmatrix} -1 \\ 1 \end{bmatrix}.
\end{align*}
We can alternatively consider this as
\begin{displaymath}
  X = Y + Z
  \qtext{with}
  \Prob(Y=\m_j) = \lambda_j
  \qtext{and}
  Z \sim \mathcal{N}(\Vc{0},\C).
\end{displaymath}
The model moment tensor $\T[3]=\mathbb{E}(\sop[3]{Y})$ is given by
\begin{displaymath}
  \T[3] = \sum_{j=1}^3 \lambda_j \sop[3]{\m_j}.
\end{displaymath}
For $p$ random samples $\set{\miwc[\Vc{x}][p]}$,
the debiased empirical moment tensor is given by
\begin{displaymath}
  \Test[3] = \frac{1}{p} \sum_{i=1}^{p}
  \prn***{ \sop[3]{\Vc{x}{i}}
  -3 \sym\prn{ \Vc{x}{i} \otimes \C }}.
\end{displaymath}
We show the difference in the norm between $\T[3]$ and $\Test[3]$.  We observe that the empirical moment converges to the model moment at a rate of $1/\sqrt{p}$, which is the expected rate of convergence.

\subsection{Comparison to EM}
\label{sec:comparison-em}

In this section, we compare the proposed method of moments (MoM) with expectation maximization (EM). Our purpose here is to demonstrate that MoM can get a somewhat better
solution than EM in scenarios where EM struggles, i.e., a high number of components,
high noise, and limited samples.
We stress that we are not arguing that MoM has any particular benefit as compared to EM
but rather that it is an intriguing alternative.
Previously, MoM was completely uncompetitive because its cost was exponential in the
order of the moment, i.e., $\mathcal{O}(n^d)$. Now, it's a viable alternative
that may prove useful for some applications.

Our setup is as follows: $n=100$ dimensions, $m=20$ components, and $p=8000$ samples.
We indicate the \emph{true} parameters, i.e., those used to generate the distribution,
with an asterisk.
The problem is randomly generated, with the following conditions:
\begin{displaymath}
  \nrm{\m{j}^*} = 1 \qtext{for all} j \in [m]
  \qtext{and}
  \ang{\m{i}^*,\m{j}^*} = 0.5 \qtext{for all} i,j\in[m].
\end{displaymath}
The covariances $\C{j}^*$ are diagonal with random entries selected uniformly
from the range $[0,2\sigma^2]$ for a given value of $\sigma^2$,
which varies between the experiments.
The proportions are selected randomly such that
\begin{displaymath}
  \min_j \lambda_j^* \geq 0.01
  \qtext{and}
  \sum_{j=1}^m \lambda_j^* = 1.
\end{displaymath}
This setup ensures that the Gaussians are difficult to distinguish,
especially for higher levels of noise, as determined by $\sigma^2$.
We generate $p=8000$ samples from the distribution.

\foreach \grp/\sgm in {1/0.05,2/0.1,3/0.2}
{
\pgfplotstableread[col sep=comma]{"./diagcov_experiments-grp-\grp/results.csv"}{\tdatatable}
\pgfplotstabletranspose[colnames from=colnames]\datatable{\tdatatable}
\newcommand{\logpdftitle}{EM objective (higher is better): log-likelihood}
\newcommand{\momthreetitle}{MoM3 objective (lower is better):
  $\nrm*{ \Tn{M}^{(3)} - \Tn[\hat]{M}^{(3)}}^2$}
\newcommand{\momfourtitle}{MoM4 objective (lower is better): $\nrm*{ \Tn{M}^{(4)} - \Tn[\hat]{M}^{(4)}}^2$}
\newcommand{\runtimetitle}{Runtime (seconds)}
\newcommand{\lambdatitle}{Difference from true weights:
  $\nrm{\Vc{\lambda}-\Vc{\lambda}^*}_1$}
\newcommand{\mutitle}{Difference from true means: $\displaystyle\avg_j
  \frac{\nrm{\Vc{\mu}{j}{-}\Vc{\mu}{j}^*}_2}{\nrm{\Vc{\mu}{j}^*}_2}$}
\newcommand{\sigmatitle}{Difference from true covariances: $\displaystyle\avg_j
  \frac{\nrm{\Mx{\Sigma}{j}{-}\Mx{\Sigma}{j}^*}_F}{\nrm{\Mx{\Sigma}{j}^*}_F}$}
\newcommand{\cosinetitle}{Cosine angular distance from true means: $\displaystyle\avg_j
  \frac{\ang{\Vc{\mu}{j},\Vc{\mu}{j}^*}}{\nrm{\Vc{\mu}{j}}_2\nrm{\Vc{\mu}{j}^*}_2}$}

\pgfplotsset{
  every axis/.append style={
    tick label style={font=\footnotesize},
    xmin=0.5,xmax=3.5,
    xtick=\empty,
    width=1.25in,
    height=1.75in,
  }
}
\tikzset{
  every mark/.append style={mark size=4pt,fill opacity=0.2},
  myplots/.style={only marks},
  em/.style={myplots,fill=Set1-A!50,color=Set1-A},
  mom3/.style={myplots,fill=Set1-B!50,color=Set1-B},
  mom4/.style={myplots,fill=Set1-C!50,color=Set1-C},
}

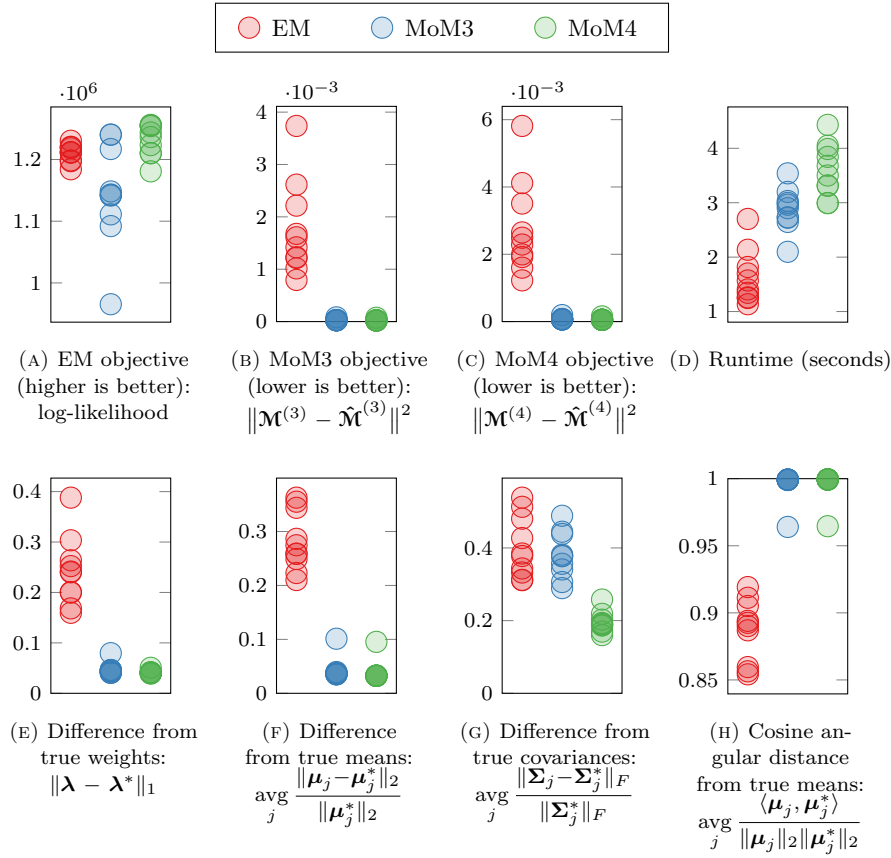
\begin{figure}
  \centering

  \begin{tikzpicture}[fill opacity=0.2,
    every node/.style={right,text=black,fill opacity=1,font=\small}]
    \node (EM) {EM};
    \draw[em] (EM.west) ++(-5pt,0) circle(4pt);
    \node[right=of EM] (MoM3) {MoM3};
    \draw[mom3] (MoM3.west) ++(-5pt,0) circle(4pt);
    \node[right=of MoM3] (MoM4) {MoM4};
    \draw[mom4] (MoM4.west) ++(-5pt,0) circle(4pt);
    \draw
    ($ (current bounding box.south west)+(-0.3,-0.1)$) rectangle 
    ($ (current bounding box.north east)+(0.3,0.1) $);    
  \end{tikzpicture}\\
  
  \foreach \foo/\title/\sep/\style in {%
    logpdf/\logpdftitle//,
    mom3func/\momthreetitle//{ymin=0},
    mom4func/\momfourtitle//{ymin=0},
    runtime/\runtimetitle/{\\}/,
    weight_err/\lambdatitle//{ymin=0},%
    mu_relerr_avg/\mutitle//{ymin=0},%
    sigma_relerr_avg/\sigmatitle//{ymin=0,yticklabel style={/pgf/number format/fixed}},%
    avg_mu_cosine/\cosinetitle//{ymax=1}%
  }
  {\subfloat[\title]{%
      \begin{tikzpicture}
        \path (-0.8,-0.2) rectangle +(3,3);
        \begin{axis}[\style]
          \addplot[em] table[x expr=1,y=em_\foo]{\datatable};    
          \addplot[mom3] table[x expr=2,y=mom3_\foo]{\datatable};    
          \addplot[mom4] table[x expr=3,y=mom4_\foo]{\datatable};
        \end{axis}
      \end{tikzpicture}%
    }\sep
  }
  \caption{Comparison of EM and proposed Method of Moments (MoM). Results for noise level: $\sigma^2=\sgm$.}
  \label{fig:results-\grp}
\end{figure}
}

We do three experiments corresponding to $\sigma^2= \set{0.05,0.1,0.2}$
with results in \cref{fig:results-1,fig:results-2,fig:results-3};
the problem is harder to solve for higher noise, i.e., higher values of $\sigma^2$.
We use the EM implementation \texttt{fitgmdist} in MATLAB
and compare against our implementation
of the method of moments (MoM)
using $\omega=0.5$ for the augmentation, and using
both third (MoM3) and fourth (MoM4) moments.

For each experiment, we run each method ten times with ten different
random starting points.
We report eight different metrics for each run, discussed below.
We plot the result of each run as a partially transparent colored circle so that overlap
can be more easily observed, e.g., a darker circle is indicative of more runs overlapping.
\begin{itemize}
\item \textbf{EM Objective:} Log-likelihood of the final result. This is the quantity maximized by EM. Larger values are better.  In \cref{fig:results-1,fig:results-2},
  both MoM3 and MoM4 achieve higher log-likelihoods than EM, even though they
  are optimizing a different cost function. The alternative
  of MoM is intriguing in this respect. 
\item \textbf{MoM3 \& MoM4 Objectives:} $\nrm{\M[3]-\Mest[3]}^2$ and $\nrm{\M[4]-\Mest[4]}^2$, without augmentation. Lower is better. These are not \emph{exactly} the loss functions used in the optimization because those use the augmented function value, but it gives an idea of
  the convergence. The MoM methods achieve lower values than EM in this metric for all three examples, but it is notable that EM is also achieving small values on these metrics.
  We do observe that MoM seems to converge to the same minimum value in most runs across all three scenarios, perhaps indicating that the optimization landscape is more favorable than
  that of log-likelihood.
\item \textbf{Runtime:} Runtime in seconds, for runs on a dedicated laptop. Faster is better.
  MoM is only a bit slower than EM in \cref{fig:results-1,fig:results-2} and faster
  in \cref{fig:results-3}. These differences are likely attributable to differences
  in stopping conditions, different methods of selecting the initial guess, and different
  code optimizations. In general, these each have the same expense per iteration.
\end{itemize}
The last four metrics  can only
be computed when the generating parameters are known.
Further, these depend on matching the computed and true solutions.
To do this, we define the cost of matching $(i,j)$ to be $\nrm{\m{i}-\m{j}^*}_2$ for every $i,j\in[m]$.
Then we use the MATLAB \texttt{matchpairs} command to find the minimum-cost matching.
The remaining metrics are based on this matching.
\begin{itemize}
\item \textbf{Proportion Error:} $\nrm{\Vc{\lambda}-\Vc{\lambda}^*}_1$. Lower is better.
  This measures the correctness in determining the probability of each component.
  MoM always outperforms EM, with MoM4 generally outperforming MoM3.
\item \textbf{Average Relative Mean Error:} $\avg_j \prn{ \nrm{\m{j}-\m{j}^*}_2/\nrm{\m{j}^*}_2}$. Lower is better.
  The measures the accuracy in identifying the means.
  MoM always outperforms EM, with MoM4 generally outperforming MoM3.
  The overall accuracy degrades for higher values of $\sigma^2$.
\item \textbf{Average Relative Covariance Error:} $\avg_j \prn{ \nrm{\C{j}-\C{j}^*}_F/\nrm{\C{j}^*}_F}$. Lower is better.
  All three methods struggle to obtain high accuracy, and we would likely need more samples to obtain better accuracy with either approach.
\item \textbf{Cosine Angle:} Average cosine of the angle between the true and computed means:
  $\avg_j \prn{ \ang{\m{j},\m{j}^*}/\prn{\nrm{\m{j}}_2\nrm{\m{j}^*}_2}}$. Higher is better, with one being the optimum. This is an alternative to the distance metric in measuring the difference between the true and computed means. MoM clearly outperforms EM on this metric as well.
\end{itemize}

 \textbf{Reproducibility.}
 We provide code for these experiments at \url{https://gitlab.com/tgkolda/gaussian_mixture_experiments}.
The exact values for the true parameters and all
the samples used in these experiments are available as Comma Separated Values (CSV) files, one line
per entry.  We also provide  a log file, and the summary of
the results used to create the images above.

\section{Conclusions}
\label{sec:conclusions}

There are two basic approaches for parameter estimation:
methods that maximize the likelihood, 
and the method of moments.
The method of moments is generally considered impractical
for multivariate problems because the size of the moments
grows exponentially, i.e., the $d$th moment for an $n$-dimensional
random variable is of size $n^d$
and computations with them would be equally expensive.

In this work, we develop new expressions for the
moments of Gaussians and moments of mixtures of Gaussians.
To the best of our understanding, this is an entirely
novel approach, revealing more of the underlying algebraic and combinatorial 
structure, which can be computationally exploited.

From these results, we show that the method of moments is
tractable for Gaussian distributions and GMMs because we need not
explicitly form the moments.
{Instead, we can compute the distance between the empirical and model moments implicitly
and use that to fit the parameters of a Gaussian distribution or GMM.}
In this setting, the per-iteration complexity of the method of moments is the
same as that for expectation maximization, keeping in mind that they
are optimizing different functions.
There are, of course,
reasonable questions about the sensitivity of the method of moments.
Nevertheless, this work adds another ``tool'' to the toolbox for fitting these models,
which may be useful in some scenarios. 

There are still many open questions remaining.
On the computational side, much more investigation is
needed into the robustness and reliability of the method
and its comparison to EM on a range of problems.
We have also deferred the full algorithm for general covariances
to future work since the details are quite lengthy (though the main formulas are present in \cref{thm:gmm_tensor_norm,thm:gmmdotp}).
We may also wish to consider practical issues such as centering and
scaling the data before applying the method.
On the theoretical side, much past work has focused on investigating
the number of samples required to identify a GMM, and it may be
that the formulations for Gaussian distributions presented here open some
new pathways for refinement of those results.

At the intersection of computation and theory,
there is the question of how to choose the
value for the augmentation parameter which implicitly
weighs the combined moments. It would be especially
helpful if there were theory to guide the choice.
Further, in our limited studies, the method of moments seems more robust than EM
to the choice of starting point.
Specifically, we refer to its ability to minimize its objective
function and identify the model parameters. For this reason,
it may also be interesting to investigate if there exists some
fundamental difference in the optimization landscape for the
method of moments in comparison to that of maximizing expectation.

\appendix

\section{Supporting Lemmas and Technical Proofs}
\label{sec:proof_techniques}

\subsection{Completing proof of \texorpdfstring{\cref{prop:Psi}}{Proposition~\ref{prop:Psi}}}
\label{sec:proof-prop-Psi}

We give below in \cref{prop:cum_formula} the calculation of the
cumulants needed in the proof of \cref{prop:Psi}.
Before we get to that result, we first establish another technical lemma.

\begin{lemma}\label{lem:cumulants_1dgaussprod}
  Let $\alpha, \beta$ be independent standard random Gaussian variables, and $d, u, v \in \R$. Then
  \begin{equation}\label{eq:1dcumgenfun}
    \log\Prn{\E\Prn{e^{t (d \alpha \beta + u \alpha + v\beta)}}} =  t^2\frac{u^2 + v^2 + 2t d u v}{2-2t^2 d^2} - \frac12\log\prn{1-t^2 d^2},
  \end{equation}
  and
  \begin{multline}\label{eq:1dcumgenfun_expansion}
    \frac{d^m}{dt^m}\Prn{\log\Prn{\E\Prn{e^{t (d \alpha \beta + u \alpha + v\beta)}}}}_{t=0} \\
    =
    \begin{cases}
      (m-1)!d^m+ m! d^{m-2}\frac{u^2 + v^2}{2} &\text{if } m \text{ is even and } m\ge 2,\\
      m! d^{m-2}u v &\text{if } m \text{ is odd and } m\ge 3,\\
      0& \text{otherwise}.
    \end{cases}
  \end{multline}
\end{lemma}	

\begin{proof}
Conditioning on $\alpha$, and taking the expectation in terms of $\beta$, we obtain
	
\begin{small}
\begin{align*}
\E\Prn{e^{t (d \alpha \beta + u \alpha + v\beta)}\middle|\alpha} &= \frac{1}{\sqrt{2 \pi}}\int_{-\infty}^{\infty} \exp\prn***{-\frac{x^2}{2} + t(d \alpha x + u \alpha + vx)}\,\mathrm{d} x\\
&= \frac{1}{\sqrt{2 \pi}}\int_{-\infty}^{\infty} \exp\prn***{-\frac12(x - t d\alpha - tv)^2}\exp\prn***{\frac{t^2}2 \prn{d \alpha +v}^2 + t u \alpha}\,\mathrm{d}x\\
&=\exp\Prn{\frac{t^2}2 \prn{d \alpha +v}^2 + t u \alpha}.
\end{align*}\end{small}
Then \cref{eq:1dcumgenfun} follows from
\begin{align*}
	\E\Prn{e^{t (d \alpha \beta + u \alpha + v\beta)}} &= \frac{1}{\sqrt{2 \pi}}\int_{-\infty}^{\infty} \exp\Prn{-\frac{x^2}2 + \frac{t^2}2 \prn{d x +v}^2 + t u x} \,\mathrm{d}x\\
	 &= \frac{1}{\sqrt{2 \pi}}\int_{-\infty}^{\infty} \exp\Prn{-\frac{1-t^2 d^2}2\Prn{x - \frac{t^2 d v + tu}{1-t^2 d^2}}^2}\\
	 &\hspace{110pt} \times \exp\Prn{\frac{(t^2 d v + tu)^2}{2-2t^2 d^2} + \frac{t^2}2 v^2}\,\mathrm{d}x\\
	 &=\frac{\exp\Prn{\frac{(t^2 d v + tu)^2}{2-2t^2 d^2} + \frac{t^2}2 v^2 }}{\sqrt{1-t^2 d^2}}=\frac{\exp\Prn{t^2\frac{u^2 + v^2 + 2t d u v}{2-2t^2 d^2} }}{\sqrt{1-t^2 d^2}}.
\end{align*}
Finally, \cref{eq:1dcumgenfun_expansion} follows from the Taylor expansions $\frac{1}{1-t^2 d^2} = 1 - d^2 t^2 + d^4 t^4 - \cdots$ and \begin{equation*}
-\frac12\log(1-t^2 d^2) = \frac12t^2 d^2 + \frac14 t^4 d^4 + \frac16 t^6 d^6 + \cdots \qedhere
\end{equation*}
\end{proof}

\begin{proposition}\label{prop:cum_formula}
Suppose $X \sim \N( \m, \C )$ and $\tilde{X} \sim \N( \tm, \tC)$ are independent random variables, and let
\begin{displaymath}
  \kappa_k =
  \frac{d^k}{d t^k} \prn**{\log\prn**{\E\prn*{\exp\prn*{t \ang{X, \tilde{X}}}}}}_{t=0}
\end{displaymath}
We then have
\begin{equation}\label{eq:kappa_def_proof}
	\kappa_k =
	\begin{cases}
		(k-1)!\trace\prn*{\Mx{Z}^{\frac{k}2}} +
		\frac{k!}{2} \prn*{\m'\tC \Mx{Z}^{\frac{k-2}2} \m + \tm' \Mx{Z}^{\frac{k-2}2} \C \tm }
		&\text{if } k \text{ is even},\\
		k! \; \tm'\Mx{Z}^{\frac{k-1}2}\m
		&\text{if } k \text{ is odd},
	\end{cases}
\end{equation}
with $\Mx{Z} = \C \tC$.
\end{proposition}
\begin{proof}%
  Let $\Csqrt$, $\tCsqrt$ be Cholesky factors (not necessarily lower triangular)
  such that $\C=\Csqrt\Csqrt'$ and $\tC=\tCsqrt\tCsqrt'$,
  and let $\Mx{U},\Mx{S},\Mx{V}$ be the full SVD factors such that $\Csqrt' \tCsqrt = \Mx{U} \Mx{S} \Mx{V}'$.
  Let $\RV, \tRV \in \Real^{n}$ be independent standard normal random vectors.
  Then using \citet[Theorem 3.1]{gut2009probability} we have that
  $$\Csqrt\Mx{U} \RV + \m \sim \N(\m, \Csqrt\Mx{U}\Mx{U}'\Csqrt') =  \N(\m, \C) \qtext{and} \tCsqrt \Mx{V} \tRV + \tm \sim  \N(\tm, \tC).$$
  \begin{align*}
    \E\Prn{\exp\Prn{t \ang{X, \tilde X}}} &=
    \E\Prn{\exp\Prn{t\prn{\Csqrt \Mx{U} \RV +\m}^{\Tr}
        \prn{\tCsqrt \Mx{V} \tRV + \tm}}}\\
    &=  \E\Prn{\exp\Prn{ t\prn{\RV'\Mx{U}'\Csqrt'\tCsqrt \Mx{V} \tRV" + \m'\tCsqrt"\Mx{V}\tRV" + \tm'\Csqrt"\Mx{U} \RV" + \m' \tm"}}}\\
    &=
    \exp\Prn{t \m' \tm"} 
    \E\Prn{\exp\Prn{t\prn{\RV'\Mx{S}\tRV" + \m'\tCsqrt"\Mx{V}\tRV" + \tm'\Csqrt"\Mx{U} \RV"}}}
  \end{align*}
  Letting $\u = \Mx{V}'\tCsqrt'\m$ and $\v = \Mx{U}'\Csqrt'\tm$, we have that
  \begin{displaymath}
    \RV'\Mx{S}\tRV" + \m'\tCsqrt"\Mx{V}\tRV" + \tm'\Csqrt"\Mx{U} \RV" = \sum_{\ell=1}^n s_{\ell\ell} \RV[\ell]\tRV[\ell] +  {v}_{\ell}  \RV[\ell] + {u}_{\ell} \tRV[\ell].          
  \end{displaymath}
  Then, since the random variables $\{\RV[\ell],\tRV[\ell]\}_{\ell \in [n]}$ are independent, applying \cref{lem:cumulants_1dgaussprod} we get
  \begin{align*}
    \kappa_k
    &= \frac{d^k}{d t^k} \Prn{\log \Prn{\E\Prn{\exp\Prn{t \ang{X, \tilde X}}}}}\\
    &= \frac{d^k}{d t^k} \Prn{t \m' \tm"
      + \sum_{\ell=1}^n \log\Prn{\E\Prn{\exp\Prn{t\prn{s_{\ell\ell} \RV[\ell]\tRV[\ell]
              +  {v}_{\ell}  \RV[\ell] + {u}_{\ell} \tRV[\ell]}}}}}\\
    &=\begin{cases}
      \displaystyle\sum_{\ell=1}^n (k-1)! s_{\ell\ell}^k+ k! s_{\ell\ell}^{k-2}\frac{{u}_\ell^2 + {v}_\ell^2}{2} &\text{if } k \text{ is even} \medskip\\
      \displaystyle\sum_{\ell=1}^n k! s_{\ell\ell}^{k-2}{u}_\ell {v}_\ell &\text{if } k \text{ is odd and } k\ge 3 \medskip\\
      \m' \tm"&\text{if } k=1
    \end{cases}
  \end{align*}
  Finally, if $k$ is even, it holds
  
  \begin{displaymath}
    \sum_{\ell=1}^n s_{\ell\ell}^k
    = \trace\Prn{\Mx{S}^k}
    = \trace\Prn{\Mx{U}\Mx{S}^k\Mx{U}'}
    = \trace\Prn{\Prn{\Csqrt' \tCsqrt"\tCsqrt'\Csqrt"}^{\frac{k}{2}}}
    = \trace\Prn{\Mx{Z}^{\frac{k}{2}}},          
  \end{displaymath}
  \begin{align*}
    \sum_{\ell=1}^n s_{\ell\ell}^{k-2} {u}_\ell^2 &= \u' \Mx{S}^{k-2}\u = \m\tCsqrt"\Mx{V} \Mx{S}^{k-2} \Mx{V}'\tCsqrt'\m = \m'\tCsqrt"\Prn{\tCsqrt' \Csqrt"\Csqrt'\tCsqrt"}^{\frac{k-2}{2}} \tCsqrt'\m,\\
    &= \m' \tCsqrt"\tCsqrt'\Prn{ \Csqrt"\Csqrt'\tCsqrt" \tCsqrt'}^{\frac{k-2}{2}}\m = \m' \tC \Mx{Z}^{\frac{k-2}2} \m,
  \end{align*}
  and, analogously, $\sum_{\ell=1}^n s_{\ell\ell}^{k-2} {v}_\ell^2 = \tm' \Mx{Z}^{\frac{k-2}2} \C \tm$.
  On the other hand, if $k\ge 3$ and is odd,
  \begin{align*}
    \sum_{\ell=1}^n  s_{\ell\ell}^{k-2} {u}_\ell {v}_\ell &= \tm'\Csqrt"\Mx{U} \Mx{S}^{k-2} \Mx{V}'\tCsqrt'\m = \tm' \Csqrt" \Prn{\Csqrt' \tCsqrt"\tCsqrt'\Csqrt"}^{\frac{k-3}{2}} \Csqrt'\tCsqrt"\tCsqrt'\m \\
    &= \tm' \Prn{\Csqrt" \Csqrt' \tCsqrt"\tCsqrt'}^{\frac{k-3}{2}}\Csqrt" \Csqrt'\tCsqrt"\tCsqrt'\m = \tm'\Mx{Z}^{\frac{k-1}2}\m.\qedhere
  \end{align*}
\end{proof}

\subsection{Completing proof of
  \texorpdfstring
  {\cref{thm:GMMdebiasedmoments}}
  {Theorem~\ref{thm:GMMdebiasedmoments}}
}
\label{sec:proof-thm:GMMdebiasedmoments}

We give below a technical lemma that is needed in \cref{thm:GMMdebiasedmoments}
for computing the constant terms in the sum.

\begin{lemma} \label{lem:Cdelta0qsum}
Suppose that $d,q$ are non-negative integers such that $d\ge 2q$. Then
\begin{displaymath}
\sum_{k=0}^{q} C_{d,k} (-1)^k C_{d-2k,q-k} = \delta_{0q}  
\end{displaymath}
\end{lemma}
\begin{proof}
We have
\begin{align*}
	\alignindent{50pt} \sum_{k=0}^{q} C_{d,k} (-1)^k C_{d-2k,q-k}\\
	&= \sum_{k=0}^q \binom{d}{2k} \dblfact{k} (-1)^{k}  \binom{d-2k}{2q-2k} \dblfact{q-k}[2q-2k] \\
	&= \sum_{k=0}^q \frac{d!}{(d-2k)!(2k)!}  \frac{(2k)!}{k! 2^k}(-1)^k\frac{(d-2k)!}{(d-2q)!(2q-2k)!}\frac{(2q-2k)!}{(q-k)! 2^{q-k}} \\
	&= \sum_{k=0}^q \frac{d!}{(d-2q)!k!(q-k)!2^q}(-1)^k
	= \binom{d}{2q} \frac{(2q)!}{q! 2^q} \sum_{k=0}^q \binom{q}{k}(-1)^k\\
	&= \binom{d}{2q} \frac{(2q)!}{q! 2^q} \delta_{0q} = \delta_{0q}. \qedhere\\
\end{align*}
\end{proof}

\subsection{Completing proof of
  \texorpdfstring
  {\cref{prop:debiaseddotp}}
  {Theorem~\ref{prop:debiaseddotp}}
}
\label{sec:proof-prop:debiaseddotp}

We have here an extension of \cref{prop:Psi} for symmetric positive semidefinite
but not necessarily positive definite matrices $\C$ and $\tC$,
which is needed in the proof of  \cref{prop:debiaseddotp}.

\begin{lemma}\label{cor:Psi}
	\cref{prop:Psi} also holds if $\C,\tC \in \R^{n\times n}$ are symmetric but not positive definite. That is, we have
	\begin{align*}%
		\PsiFn{\m}{\C}{\tm}{\tC} \alignindent{-30pt}= B_d\prn{\kappa_1, \dots, \kappa_d},\\
		&=\ang****{
			\sum_{k=0}^{\floor{d/2}} \!\! C_{d,k}  \sym\prn**{ \sop[d-2k]{\m} \otimes \sop[k]{\C}},
			\sum_{k=0}^{\floor{d/2}} \!\! C_{d,k}  \sym\prn**{ \sop[d-2k]{\tm} \otimes \sop[k]{\tC}}}.	
	\end{align*}
	with $\kappa_k,\,k\in[d]$, defined as in \cref{prop:Psi}.
\end{lemma}
\begin{proof}
	\Cref{thm:gaussian_moments} and \cref{prop:Psi} imply that this holds for all symmetric positive definite matrices. On the other hand, $\PsiFn{\m}{\C}{\tm}{\tC}$
	is a multivariate polynomial of the entries of $\C$ and $\tC$, and the result follows from our claim that if two polynomials of symmetric matrices are equal for all positive definite matrices, then they must also be equal for all other symmetric matrices.
	
	We prove this claim by contradiction. If the claim is false, then the difference of the two polynomials, which we denote by $p$, vanishes for all positive definite matrices but there exists a symmetric matrix $\Mx{A}$ such that $p(\Mx{A})\neq 0$. Consider the univariate polynomial $g(t) = p(\Mx{I} + (\Mx{A}-\Mx{I})t)$. For all $t\in \Real$ such that $|t|<\frac{1}{\|\Mx{A}-\Mx{I}\|_2}$, we have that $\Mx{I} + (\Mx{A}-\Mx{I})t  \succeq \Mx{0}$, which implies $g(t) = 0$. But since this is an univariate polynomial that vanishes in an open set, it must be $0$ for all reals. This implies that $0 = g(1)= p(\Mx{A})$, contradiction!
\end{proof}

\subsection{Derivatives of inner products of GMMs}
\label{sec:inner_product_derivatives}

We consider the
derivation of gradients for $F_1^{(d)}(\theta)$ from \cref{thm:gmm_tensor_norm}.
For simplicity, we just consider a single pair of moments, representing one term in the summation.

\begin{proposition}\label{prop:Psi_derivatives}
	Define $\PsiFn{\m}{\C}{\tm}{\tC}$ as in \cref{prop:Psi}, then
	\begin{align}  
		\FD{\PsiFn{\m}{\C}{\tm}{\tC}}{\m}
		&= \sum_{k=1}^d \binom{d}{k} B_{d-k} \prn{\kappa_1, \dots, \kappa_{d-k}}
		\FD{\kappa_k}{\m}
		, \label{eq:Psi_derivative_mu}\\
		\FD{\PsiFn{\m}{\C}{\tm}{\tC}}{\C}
		&= \sum_{k=1}^d \binom{d}{k} B_{d-k} \prn{\kappa_1, \dots, \kappa_{d-k}}
		\FD{\kappa_k}{\C} . \label{eq:Psi_derivative_Sigma}
	\end{align}
	where
	\begin{equation}
		\label{eq:kappa_derivative_mu}
		\FD{\kappa_k}{\m} =
		\begin{cases}
			k!\, \tC \Mx{Z}^{\frac{k-2}2} \m  
			& \text{if $k$ even}, \\
			k!\, \prn{\Mx{Z}'}^{\frac{k-1}{2}} \tm
			& \text{if $k$ odd},
		\end{cases}  
	\end{equation}
	\begin{equation}\label{eq:kappa_derivative_Sigma}
		\FD{\kappa_k}{\C} =
		\begin{cases}
			\displaystyle \frac{k!}{2}\tC \Mx{Z}^{\frac{k-2}2}  + 
			\frac{k!}{2} \Prn{\Mx{V} + \Mx[\tilde]{V}}& \text{if $k$ even}, \medskip\\
			\displaystyle k! \sum_{\ell=0}^{\frac{k-3}{2}}\prn{\Mx{Z}'}^\ell
			\tm \m' \tC \Mx{Z}^{\frac{k-3}{2}-\ell} 
			& \text{if $k$ odd},
		\end{cases}  
	\end{equation}
    with $\Mx{Z} = \C \tC$, and
	\begin{displaymath}
		\Mx{V} = \sum_{\ell=0}^{\frac{k-4}{2}}\tC\Mx{Z}^\ell
		\m \m'
		\tC \Mx{Z}^{\frac{k-4}{2}-\ell} \qtext{and} \Mx[\tilde]{V}= \sum_{\ell=0}^{\frac{k-2}{2}} \prn{\Mx{Z}'}^\ell 
		\tm \tm'
		\Mx{Z}^{\frac{k-2}{2}-\ell} 
	\end{displaymath}

\end{proposition}

\begin{proof}
\cref{eq:Psi_derivative_mu,eq:Psi_derivative_Sigma} follow from the formula for the derivative of Bell polynomials \cref{eq:bell-derivative} and the chain rule. Regarding \cref{eq:kappa_derivative_mu}, we use \citet[Eqs. (69) and (81)]{PePe12},
$$\FD{\prn{\a'\x}}{\x} = \a \qtext{and} \FD{\prn{\x'\Mx{B} \x}}{\x} = (\Mx{B} + \Mx{B}')\x,$$
that $\C, \tC$ are symmetric, and for all non-negative integer $r$,
\begin{equation}\label{eq:tCZr_symmetric}
	(\tC\Mx{Z}^{r})^{\Tr}  =\prn{\tC \C}^{r} \tC= \underbrace{\tC \C \cdots \tC \C}_{r \text{ times}} \tC = \tC \underbrace{\C \tC \cdots \C  \tC}_{r \text{ times}} = \tC \Mx{Z}^{r}. 
\end{equation}
We obtain
\begin{displaymath}
	\FD{\kappa_k}{\m} =
	\begin{cases}
	 \FD{\prn*{	\frac{k!}{2}\m'\tC \Mx{Z}^{\frac{k-2}2} \m}}{\m} = k!\tC \Mx{Z}^{\frac{k-2}2}\m
		&\text{if } k \text{ is even},\\
	     \FD{\Prn{k! \tm'\Mx{Z}^{\frac{k-1}2}\m}}{\m} = k! \prn{\Mx{Z}'}^{\frac{k-1}2} \tm
		&\text{if } k \text{ is odd}.
	\end{cases}
\end{displaymath}
As for \cref{eq:kappa_derivative_Sigma}, we use \citet[Eqs. (91) and (121)]{PePe12}
\begin{align*}
	\FD{\prn*{\Vc{a}'\Mx{X}^n\Vc{b}}}{\Mx{X}}
	&=\sum_{i=0}^{n-1} (\Mx{X}^i)^{\Tr} \Vc{a}\Vc{b}'(\Mx{X}^{n-1-i})^{\Tr},
	\\
	\FD{\prn*{\trace(\Mx{X}^k)}}{\Mx{X}}
	&= k\prn{\Mx{X}^{k-1}}^{\Tr}.
\end{align*}
We also use \cref{eq:tCZr_symmetric} and the following formula, which is valid for any differentiable function $f$, and follows from the chain rule:
$$\FD{(f(\C, \Mx{Z}))}{\C} = \FD{\!f}{\Mx{Z}}\,(\C, \Mx{Z})\tC + \FD{\!f}{\C}\,(\C, \Mx{Z}).$$
Using these we obtain
\begin{align*}
	\FD{\trace\prn*{\Mx{Z}^{\frac{k}2}}}{\C} &= \frac{k}{2}\tC \Mx{Z}^{\frac{k-2}2}\\
	\FD{\prn*{\m'\tC \Mx{Z}^{\frac{k-2}2} \m}}{\C} &= \sum_{\ell=0}^{\frac{k-4}{2}}\prn{\Mx{Z}'}^\ell \tC
	\m \m' \prn{\Mx{Z}'}^{\frac{k-4}{2}-\ell} \tC
    = \sum_{\ell=0}^{\frac{k-4}{2}} \tC \Mx{Z}^\ell
	\m \m'
	\tC \Mx{Z}^{\frac{k-4}{2}-\ell} \\
	\FD{\prn*{\tm' \Mx{Z}^{\frac{k-2}2} \C \tm}}{\C} &=
	\sum_{\ell=0}^{\frac{k-4}{2}}\prn{\Mx{Z}'}^\ell \tm
	\tm' \C \prn{\Mx{Z}'}^{\frac{k-4}{2}-\ell} \tC + \prn{\Mx{Z}'}^{\frac{k-2}{2}} \tm \tm' \\
	&= \sum_{\ell=0}^{\frac{k-2}{2}} \prn{\Mx{Z}'}^\ell 
	\tm \tm'
	\Mx{Z}^{\frac{k-2}{2}-\ell} \\
	\FD{\prn*{\tm'\Mx{Z}^{\frac{k-1}2}\m}}{\C} &= \sum_{\ell=0}^{\frac{k-3}{2}}\prn{\Mx{Z}'}^\ell
	\tm \m' \prn{\Mx{Z}'}^{\frac{k-3}{2}-\ell} \tC = \sum_{\ell=0}^{\frac{k-3}{2}}\prn{\Mx{Z}'}^\ell
	\tm \m' \tC \Mx{Z}^{\frac{k-3}{2}-\ell}. \qedhere
\end{align*}
\end{proof}

\section*{Acknowledgments}

The work of J.~Kileel and J.M.~Pereira was partially supported by start-up grants provided to J.~Kileel by the College of Natural Sciences and the Oden Institute of Computational Engineering and Sciences at the University of Texas at Austin. The work of J.M.~Pereira was also partially supported by grants AFOSR MURI FA9550-19-1-0005 and NSF HDR-1934932. The work of T.~G.~Kolda was partially supported by a Distinguished Visiting Professorship in the 
Industrial Engineering and Management Sciences Department of Northwestern University. We would like to thank the Reddit user \verb|u/ForceBru| for feedback on the paper.

\bibliographystyle{abbrvnatmod}

%%\\n

\end{document}